%% file: main.tex
\documentclass[aos]{imsart}

\RequirePackage{amsthm,amsmath,amsfonts,amssymb}
\RequirePackage[numbers]{natbib}
\RequirePackage[colorlinks,citecolor=blue,urlcolor=blue]{hyperref}
\RequirePackage{graphicx}
\usepackage{verbatim}
\startlocaldefs
\theoremstyle{plain}

\newtheorem{theorem}{Theorem}[section]
\newtheorem{lemma}[theorem]{Lemma}

\newtheorem{corollary}[theorem]{Corollary}

\newtheorem{proposition}[theorem]{Proposition}

\theoremstyle{remark}
\newtheorem{definition}{Definition}

\newtheorem{assumption}{Assumption}

\newtheorem{remark}{Remark}

\endlocaldefs

\input{macros.tex}

\usepackage{setspace}
\begin{document}

\begin{frontmatter}
\title{Generalization Ability of Wide Neural Networks on $\mathbb{R}$}

\runtitle{Neural Networks on $\mathbb{R}$}

\begin{aug}

\author[A]{\fnms{Jianfa} \snm{Lai}\ead[label=e1]
{jianfalai@mail.tsinghua.edu.cn}},
\author[A]{\fnms{Manyun} \snm{Xu}\ead[label=e2]
{xumy20@mails.tsinghua.edu.cn}}\footnote{Co-first author},
\author[A]{\fnms{Rui} \snm{Chen}\ead[label=e3]
{chenrui@mail.tsinghua.edu.cn}}
\and
\author[A,B]{\fnms{Qian} \snm{Lin} }\thanks{Corresponding author}

\address[A]{
Center for Statistical Science, Department of Industrial Engineering Tsinghua University\\
\printead{e1,e2,e3}
}

\address[B]{
Beijing Academy of Artificial Intelligence, Beijing, 100084, China, \href{qianlin@tsinghua.edu.cn}{qianlin@tsinghua.edu.cn}
}
\end{aug}

\begin{abstract}
We perform a study on the generalization ability of the wide two-layer ReLU neural network on $\mathbb{R}$.
We first establish some spectral properties of the neural tangent kernel (NTK): $a)$ $K_{d}$, the NTK defined on $\mathbb{R}^{d}$, is positive definite; $b)$ $\lambda_{i}(K_{1})$, the $i$-th largest eigenvalue of $K_{1}$, is proportional to $i^{-2}$.
We then show that: $i)$ when the width $m\rightarrow\infty$, the neural network kernel (NNK) uniformly converges to the NTK;
$ii)$ the minimax rate of regression over the RKHS associated to $K_{1}$ is $n^{-2/3}$;
$iii)$ if one adopts the early stopping strategy in training a wide neural network, the resulting neural network achieves the minimax rate;
$iv)$ if one trains the neural network till it overfits the data, the resulting neural network can not generalize well.
Finally, we provide an explanation to reconcile our theory and the widely observed ``benign overfitting phenomenon''.
\end{abstract}

\begin{keyword}[class=MSC2020]
\kwd[Primary ]{62G08}
\kwd[; secondary ]{68T07}
\kwd{46E22}
\end{keyword}

\begin{keyword}
\kwd{Early stopping}
\kwd{linear Interpolation}
\kwd{neural networks}
\kwd{neural tangent kernel}
\end{keyword}

\end{frontmatter}

    
\section{Introduction}\label{intro}
\input{intro.tex}

\section{Neural tangent kernel and its spectral properties}\label{sec:ntk}
\input{NTK.tex}

\section{Neural network kernel converges to neural tangent kernel uniformly}\label{sec:main_result}

\input{main_result.sava.tody.tex}

\section{Why overfitted neural networks generalize}\label{sec:explanation}
\input{explanation.tex}

\subsection{The effects of signal strength}\label{sec:experments}
\input{Experiment.tex}

\section{Discussion and conclusion}
\input{Discussion.tex}

\input{appendix.tex}

\begin{supplement}
\stitle{Supplement to ``Generalization Ability of Wide Neural Networks on $\mathbb{R}$''} \sdescription{This supplementary file contains the proofs of Theorem \ref{thm:risk:approx}, \ref{thm:nn:early:stopping:d=1} and \ref{thm:bad_gen}.}
\end{supplement}

\bibliographystyle{imsart-number} 
\bibliography{main.bib}


\input{supp.tex}

\end{document}

%% file: macros.tex
\newcommand*{\abs}[1]{\lvert #1 \rvert}

\newcommand*{\bm}{\boldsymbol}

\newcommand*{\dx}{\mathrm{d}}
\newcommand*{\vect}{\operatorname{vec}}
\newcommand*{\unif}{\operatorname{unif}}
\newcommand*{\tran}{^{\mathsf{T}}}
\newcommand*{\NTK}{\mathtt{NTK}}
\newcommand*{\LI}{\mathtt{LI}}
\newcommand*{\Prob}{\mathbf{P}}
\newcommand*{\Expc}{\mathbf{E}}
\newcommand*{\comp}{^{\mathsf{c}}}

\newcommand*{\poprisk}{\mathcal{L}}
\newcommand*{\emprisk}{\hat{\mathcal{L}}_{n}}
\newcommand*{\excrisk}{\mathcal{E}}

%% file: intro.tex
Deep neural networks have been successfully applied in various fields such as image analysis, natural language processing, protein structure prediction, etc.\cite{krizhevsky2012imagenet, devlin2019bert, jumper2021highly}. 
Since the number of parameters appeared in deep neural networks is often ten times or hundred times larger than the sample size of data, the successes of neural network methods have challenged the traditional bias variances trade-off principle, one of the primary doctrines in the classical statistical learning theories \cite{vapnik1995nature}. 
For example, many influential experiments  \cite{belkin2019does, zhang2016understanding,belkin2018understand, nakkiran2021deep,belkin2021fit} suggested that if one trains a neural network till it overfits the data, the resulting network can still generalize well. 
This observation, often referred to as the ``{\it benign overfitting phenomenon}'' \cite{bartlett2020benign, sanyal2020benign,frei2022benign, mallinar2022benign}, actually reshaped the landscape of the studies in neural networks. 
For example, some researchers built giant neural networks in practice which can easily achieve nearly zero training error and possess the state-of-the-art performances \cite{huang2019gpipe, radford2019language, devlin2018bert}. 
Inspired by these experiments and observations, researchers proposed various new theories to explain why overfitted neural networks do generalize well on certain data \cite{belkin2019does, liang2020just, frei2022benign, montanari2022interpolation}.

Several groups of statisticians tried to explain the generalization ability of neural networks from statistical decision theory with various carefully designed nonparametric regression frameworks. 
For example, assuming that the regression function belongs to a carefully designed sub-class of the H\"{o}lder continuous functions,  \cite{bauer2019deep} proved  that there exists a neural network with sigmoid activation function achieving the corresponding minimax rate;
\cite{schmidt-hieber2020nonparametric} further established similar results for ReLU neural networks based on the approximation theory from \cite{yarotsky2017error};
\cite{suzuki2019adaptivity} then extended these results to regression functions in Besov space and its variants. 
Most of these works \cite{kohler2005adaptive, petersen2018optimal, schmidt-hieber2020nonparametric, bauer2019deep, imaizumi2019deep, suzuki2019adaptivity, hayakawa2020minimax, suzuki2021deep,kim2021fast} first proposed  a carefully chosen candidate class of functions/models;   they then showed that the empirical risk minimizer (ERM) of some loss function over the sets of neural networks  can achieve the corresponding minimax rate.  
However, besides  the unrealistic assumptions on the underlying models, these (static) ERMs approaches are hard to apply in practice,  because the corresponding optimization problems are highly non-linear and non-convex. Therefore these static non-parametric explanations are far from a satisfactory theory.

Since training a neural network is a highly non-convex optimization problem, researchers put lots of effort to argue whether the gradient descent (GD) or stochastic gradient descent (SGD) can find the global minimal points.  At first glance, 
it is unlikely that GD/SGD can find the global minimum on a highly non-convex problem. 
Since the number of parameters of an implemented neural network in practice is often ten times or hundred times larger than the sample size, wide neural networks are of the top priority to investigate. 
After assuming that the width of the neural network is large enough, \cite{du2018gradient} first analyzed a wide two-layer neural network with random initialization and showed that under some positiveness conditions on the Gram matrix, the GD can find one of the global minimal points with high probabilities. 
\cite{allen2019convergence} further proved that GD/SGD can find the global minima 
of wide multi-layer neural networks in polynomial time with high probabilities. 
These analyses focus on the so-called ``lazy training regime'' where the width $m$ is sufficiently large such that the weight parameters stay in a small neighborhood of their initialization during the training process.
Though these works showed that GD/SGD can find one of the many global minimum points in the lazy training regime,  they lack the analyses of the generalization ability of the selected neural network. 

In \cite{jacot2018neural}, Jacot et al. proposed a framework to understand the gradient flow appeared in training wide neural networks through the gradient flow of a kernel regression. 
To be more precise, they interpreted the gradient flow of a loss function defined on the set of neural networks as a gradient flow associated to a kernel regression problem where the kernel, often referred to as the {\it neural network kernel} (NNK), is varying during the training process. 
Moreover, by allowing the width $m\to \infty$, they further showed that the NNK stays invariant during the training process. This time-independent kernel, which they called the {\it neural tangent kernel} (NTK), plays an indispensable role in the current research of neural networks. On the one hand, the studies of neural networks in the ``lazy training regime'' can resort to the studies of kernel regression with respect to the NTK. 
\cite{arora2019exact,lee2019wide} concluded that as the width $m\to \infty$, the wide neural network trained by GD converges to the kernel regression predictor with respect to the NTK. \cite{hu2021regularization, suh2022nonparametric} showed that with a proper regularization parameter, the kernel ridge regression with respect to the NTK can reach the minimax-optimal rate.
On the other hand, whether the NTK possesses some remarkable properties inspired a renaissance of the studies in kernel regression from various aspects. 
For example, \cite{rakhlin2019consistency,beaglehole2022kernel, buchholz2022kernel, mallinar2022benign} and \cite{liang2020just} considered the generalization performance of the kernel ridgeless regression in low dimensional and  high dimensional data respectively; 
\cite{jacot2020kernel, bordelon2020spectrum, canatar2021spectral, simon2021neural} reinvestigated the generalization error of kernel ridge regression through the eigenlearning framework.

Though the aforementioned inspirational works shed us some light on understanding  the superior performance of neural networks, they have not formed a comprehensive explanation on why neural networks can generalize, even in the ``lazy training regime''. 
In this paper, we perform a study on the generalization ability of the wide neural network on $\mathbb{R}$.
We first show in Section \ref{sec:main_result} that the NNK converges to the NTK uniformly as the width $m\to \infty$, therefore the gradient flow of the wide two-layer ReLU neural network uniformly converges to the gradient flow of the corresponding NTK regression. 
With these uniform convergences, we then show in Section \ref{subsec:generalization} that:
1. the neural network produced by an early stopping  strategy is minimax rate optimal; 2. the overfitted neural network can not generalize well. It is clear that the ``benign overfitting phenomenon'' violates the latter statement.  
To reconcile this contradiction, we further proposed a hypothesis on the role played by the signal strength in the ``benign overfitting phenomenon'' in Section \ref{sec:explanation}.

\subsection{Contributions} 

In this paper, we focus on training a wide two-layer ReLU neural network in the so-called ``lazy training regime''. That is, the width $m$ of the neural network is sufficiently large so that the parameters of the neural network stay in a small neighbourhood of the initialization.

{\it $\bullet$ Spectral  properties of the NTK.} 
We first show that the NTK is positive definite on $\mathbb{R}^{d}$, filling a long-standing gap in the literature.
We then provide an optimal bound on the minimum eigenvalue of the gram matrix $(K(\bm{x}_{i},\bm{x}_{j}))_{1\leq i,j\leq n}$ for one-dimensional data. 
Finally, we determine the decay rate of the eigenvalues of the NTK defined on $[0,1]$.
To the best of our knowledge, our work is the first result about the spectral properties of the NTK defined on a domain other than sphere \cite{bietti2019inductive, bietti2020deep, cao2019towards}.
Though the eigenvalue decay rate of the NTK is obtained only for a one-dimensional interval in this paper, it sheds light on  obtaining similar results for the NTK defined on $\mathbb{R}^{d}$. We believe this problem would be of great interest to researchers.

{\it $\bullet$ NNK converges to NTK uniformly.} Though many works have claimed that the dynamic of training the wide neural network can be well approximated by that of the NTK regression, all of them only proved this claim pointwisely \cite{arora2019exact, lee2019wide}.
In this paper, we first show that the NNK converges to the NTK uniformly and that the dynamic of training the wide two-layer ReLU neural network converges to that of the NTK regression uniformly. Thus, the generalization performance of the wide neural networks can be approximated well by that of the NTK regression.

{\it $\bullet$ Generalization performance of neural networks on $\mathbb{R}$.}
With the assumption that the regression function $f_{\star}\in\mathcal{H}_{1}$, the RKHS associated to the NTK $K_{1}$ defined on $\mathbb{R}$, 
we prove that training a wide neural network with a properly early stopping strategy can produce a neural network achieving the minimax-optimal rate $n^{-2/3}$, i.e., the early stopped neural network can generalize. On the other hand, we can show that if one trains a wide neural network till it overfits the equally-distanced one-dimensional data, the resulting neural network is essentially a linear interpolation and thus can not generalize. To the best of our knowledge,
it is the first time that we have a concrete understanding on what an overfitted neural network looks like. 

{\it $\bullet$ Implicitly early stopping caused the ``benign overfitting phenomenon''. } Most reported experiments on the ``benign overfitting phenomenon'' in neural networks might ignore a subtle difference between the 100\% training accuracy of labels and the (nearly) zero training loss.
This difference actually leads the training process being stopped earlier than the time needed to overfit the data. 
We call the strategy stopping the training process with near 100\% training accuracy the implicit early stopping rule and find that the occurrence of it depends on the signal strength of the data. 
We further illustrate through several experiments how the signal strength affects the implicitly early stopping rule and the generalization ability of the resulting neural networks.

\subsection{Related works}
Whether the overfitted neural network can generalize is arguably one of the most intriguing questions in explaining the superior performance of the neural network methods in practice.
Inspired by the experiments reported in \cite{zhang2016understanding},
lots of effort  tried to explain that overfitted models/neural networks can generalize well \cite{belkin2019does, liang2020just}. For example, \cite{belkin2019does} exhibited the singular Nadaraya-Watson estimator that interpolates the data can achieve the corresponding minimax optimal rate;
\cite{liang2020just} illustrated that the Kernel ``Ridgeless'' Regression can perfectly fit the high dimensional data and still generalize well.
Though these interpolations possess some generalization ability, we still need more work to explain the ``benign overfitting phenomenon'' for neural networks.
On the other hand, there are few results claiming that kernel interpolations can not generalize well \cite{rakhlin2019consistency, buchholz2022kernel}. 
For example, \cite{rakhlin2019consistency} showed that for fixed dimension, the Laplace kernel interpolation cannot have vanishing error for noisy data as $n\to \infty$, even with bandwidth adaptive to the training set; \cite{buchholz2022kernel} further extended the result to the kernels whose associated reproducing kernel Hilbert space (RKHS) is a Sobolev space $H^{s}$, where $d/2< s< 3d/4$. However, these results can not conclude the inconsistency of the neural network interpolation. 

Besides the aforementioned non-parametric static ERMs approaches, there are few works studying the generalization ability of neural networks through the dynamic of gradient descent or stochastic gradient descent \cite{zhong2017recovery, zhang2019learning, lei2022stability}. 
Most of them assumed that the data live in a sphere since the NTK is an inner product kernel on the sphere and the spectral properties of the NTK are well understood \cite{bietti2019inductive,cao2019towards, geifman2020similarity, bietti2020deep}. 
For example, Hu et al. \cite{hu2021regularization}, one of the most relevant works, considered the generalization performance of a two-layer ReLU neural network defined on a sphere $\mathbb{S}^{d}$ trained by the gradient descent with or without a $L^{2}$ penalized term.
They claimed that: 1) the overfitted neural network does not generalize well; 2) the properly early stopped trained neural network can achieve the optimal rate.
Unfortunately, their first claim relies on an unproved result (the second statement of the Corollary 3 in \cite{raskutti2014early}) essentially; their second claim secretly utilizes another unproved fact: the NNK convergence to the NTK uniformly, one of the major technical contributions in our current work.

\paragraph*{Notation} For every positive integer $n\in\mathbb{N}^{+}$, denote $\{1,\dots,n\}$ by $[n]$. For a real number $x\in\mathbb{R}$, denote by $\lceil x\rceil$ the smallest integer that is greater or equal to $x$ and by $\lfloor x\rfloor$ the greatest integer that is greater or equal to $x$. For $\bm{v}\in\mathbb{R}^{d}$, denote by $\bm{v}_{(j)}$ the $j$-th component of $\bm{v}$ and denote the $\ell_{2}$ norm and supreme norm of $\bm{v}$ by $\|\bm{v}\|_{2}=(\sum_{j\in[d]}\bm{v}_{(j)}^{2})^{1/2}$ and $\|\bm{v}\|_{\infty}=\max_{j\in [d]}|\bm{v}_{(j)}|$ respectively. For a matrix $\bm{A}\in\mathbb{R}^{m\times n}$, denote by $a_{ij}$ the $(i,j)$-th component of $\bm{A}$ and denote the operator norm and the Frobenius norm of $\bm{A}$ by $\|\bm{A}\|_{2}=\sup_{\bm{v}\in\mathbb{R}^{n}}\|\bm{A}\bm{v}\|_{2}/\|\bm{v}\|_{2}$ and $\|\bm{A}\|_{\mathrm{F}}=(\sum_{i\in[m],j\in[n]}a_{ij}^{2})^{1/2}$ respectively. For a set $A$, denote by $|A|$ the number of elements $A$ contains. Let $\mu_{\mathcal{X}}$ be a positive measure on $\mathcal{X} \subseteq \mathbb{R}^d$. We define the space $L_{2}(\mathcal{X},\mu_{\mathcal{X}}) = \{f:\mathcal{X} \to \mathbb{R} : \int_{\mathcal{X}} |f(\bm{x})|^2 \dx \mu_{\mathcal{X}} <\infty \}$. We use the notation $a_{m}=o_{m}(1)$, meaning the sequence $\{a_{m}\}_{m=1}^{\infty}$ converges to zero as $m\to\infty$.

\vspace{3mm}
Let $f_{\star}$ be a continuous function defined on a compact subset $\mathcal{X} \subseteq [-B,B]^{d} \subseteq \mathbb{R}^{d}$ for some $B>0$ and $\mu_{\mathcal{X}}$ be a distribution supported on $\mathcal{X}$. Suppose that we have observed $n$ i.i.d. samples $\{(\bm{x}_{i},y_{i}), i \in [n]\}$ sampling form the model:
\begin{equation}\label{equation:true_model}
    y_i=f_{\star}(\boldsymbol{x}_i)+\varepsilon_{i}, \quad i=1,\dots,n,
\end{equation}
where $\boldsymbol{x}_{i}$'s are sampled from  $\mu_{\mathcal{X}}$ and $\varepsilon_{i} \sim \mathcal{N}(0,\sigma^{2})$ for some fixed $\sigma>0$. 
We are interested in finding $\hat{f}_{n}$  based on these $n$ samples, which can minimize the excess risk, i.e., the difference between $\poprisk(\hat{f}_{n})=\Expc_{(\bm{x},y)}\left[(\hat{f}_{n}(\bm{x})-y)^2\right]$ and $\poprisk(f_{\star})=\Expc_{(\bm{x},y)}\left[(f_{\star}(\bm{x})-y)^2\right]$.
One can easily verify the following formula about the excess risk:
\begin{equation}\label{eq:excrisk}
\excrisk(\hat{f}_{n})=\poprisk(\hat{f}_{n})-\poprisk(f_{\star})=\int_{\mathcal{X}}(\hat{f}_{n}(\bm{x})-f_{\star}(\bm{x}))^{2} \dx \mu_{\mathcal{X}}(\bm{x}).
\end{equation}
It is clear that the excess risk is an equivalent evaluation of the generalization performance of $\hat{f}_{n}$. When $\bm{x}_{i}$ is assumed to be fixed, the excess risk can be taken as measuring the $L^{2}(\mathcal{X},\nu)$ distance between $\hat{f}$ and $f_{\star}$, where $\nu$ is the Lebesgue measure.

%% file: NTK.tex
Given the data $\{(\bm{x}_{i},y_{i})\in \mathbb{R}^{d}\times\mathbb{R}, i\in[n]\}$, we are interested in analyzing the gradient flow of the empirical loss function 
\begin{equation*}
    \emprisk(f_{\bm{\theta}}^{m})=\frac{1}{2n}\sum_{i=1}^{n}\left(y_{i}-f_{\bm{\theta}}^{m}(\bm{x}_{i})\right)^{2}=\frac{1}{2n}\|\bm{y}-f_{\bm{\theta}}^{m}(\bm{X})\|_{2}^{2},
\end{equation*}
where $f^{m}_{\bm{\theta}}$ is a two-layer ReLU neural network with width $m$. More precisely, 
\begin{equation*}
f_{\bm{\theta}}^{m}(\bm{x}) = \sqrt{\frac{2}{m}}\sum_{r=1}^{m} \left(a_{r}\sigma\left(\langle \bm{w}_{r},\bm{x}\rangle+b_{r}\right)\right) + b,
\end{equation*}
where $a_{r}(0),\bm{w}_{r,j}(0),b_{r}(0), b \sim \mathcal{N}(0,1)$ for $r\in [m]$ and $j\in [d]$, $\bm{\theta}=\vect(\{a_{r},\bm{w}_{r},b_{r},b,r\in [m]\})$ and the activation function $\sigma(z)=\max\{z,0\}$.

Since training neural networks is a highly non-linear problem, Jacot et al. \cite{jacot2018neural} proposed to utilize a time-invariant kernel, the neural tangent kernel (NTK), to investigate the training process of neural networks when the width $m\rightarrow\infty$. 
Let us denote the neural tangent kernel of the two-layer neural network by $K_{d}(\bm{x},\bm{x}')$.  
Thanks to the results in \cite{cho2009kernel}  and \cite{jacot2018neural}, we can get the following  explicit expression:
\begin{equation}\label{eq: NTK_formular}
    K_{d}(\bm{x},\bm{x}')=\frac{2}{\pi}\left(\pi-\psi(\bm{x},\bm{x}')\right)(\langle\bm{x},\bm{x}'\rangle +1) +\frac{1}{\pi}\sqrt{\|\bm{x} -\bm{x}'\|_{2}^{2}+\|\bm{x} \|_{2}^{2}\|\bm{x}'\|_{2}^{2}-\langle\bm{x},\bm{x}'\rangle^{2}}+1
\end{equation}
where $\psi(\bm{x},\bm{x}')=\arccos$\resizebox{!}{0.3cm}{$\frac{\langle \bm{x},\bm{x}' \rangle+1}{\sqrt{ (\|\bm{x}\|_{2}^{2} +1)(\|\bm{x}'\|_{2}^{2}+1)}}$}.

\vspace{1mm}
We first show that $K_{d}$ is a positive definite kernel. To avoid the potential confusion between positive definiteness and positive semi-definiteness of a kernel function, we explicitly adopt the following definition of positive definiteness. 
\begin{definition}
       A kernel function $K$ is positive definite (semi-definite) over domain $\mathcal{X}\subseteq \mathbb{R}^{d}$ if for any positive integer $n$ and any $n$ different points $\bm{x}_{1},\dots,\bm{x}_{n}\in \mathcal{X}$, the smallest eigenvalue $\lambda_{\min}$ of the matrix $K(\bm{X},\bm{X})=(K(\bm{x}_{i},\bm{x}_{j}))_{1\leq i,j\leq n}$ is positive (non-negative).
\end{definition}

Positive definiteness of the neural tangent kernel is widely assumed in literature \cite{du2018gradient, du2019gradient, arora2019fine}. To the best of our knowledge, the positive definiteness of the NTK has been only proved when it is defined on sphere $\mathbb{S}^{d-1} \subseteq \mathbb{R}^{d}$ \cite{jacot2018neural}. The following proposition states that the NTK is positive definite on a compact  set $ \mathbb{R}^{d}$. 

\begin{proposition}[Positive definiteness of $K_{d}$]
	\label{PD}
	For any $d\geq 1$, the neural tangent kernel $K_{d}(\bm{x},\bm{x}')$ is  positive definite on any compact subset $\mathcal{X} \subseteq \mathbb{R}^{d}$. 
\end{proposition}

Once we know the positive definiteness of $K_{d}$ on $\mathbb{R}^{d}$,  the spectral properties of $K_{d}$ are of our further interest. For example,  the minimum eigenvalue of $K_{d}(\bm{X},\bm{X})$ is of particular interest in analyzing the dynamic of training a wide neural network,  because most analyses are implicitly or explicitly assumed that it is positive (e.g.,\cite{hu2021regularization, suh2022nonparametric}). 
Furthermore,  the celebrated  Mercer's decomposition theorem asserts (loosely speaking) that there exist non-negative numbers $\lambda_{1} \geq \lambda_{2} \geq \cdots$ and functions $\phi_{1}, \phi_{2}, \cdots \in L_{2}(\mathcal{X}, \mu_{\mathcal{X}})$ such that $\left<\phi_{i},\phi_{j}\right>_{L_{2}(\mathcal{X}, \mu_{\mathcal{X}})}=\delta_{ij}$ and 
 \begin{align}
K_{d}(\boldsymbol{x},\boldsymbol{x}')=\sum_{j=1}^{\infty}\lambda_{j}\phi_{j}(\boldsymbol{x})\phi_{j}(\boldsymbol{x}'),
 \end{align}
 where the series on RHS converges in $L_{2}(\mathcal{X},\mu_{\mathcal{X}})$ (please see Appendix \ref{sec:RKHS}
 for more rigorous statements). The numbers $\{ \lambda_{j}, j \geq 1\}$ and the functions $\{\phi_{j},j \geq 1\}$ are often referred to as the eigenvalues and the eigenfunctions associated to the kernel $K_{d}$ respectively. The decay rate of $\{\lambda_{j},j \geq 1\}$ is of great interest in determining the metric entropy of the RKHS associated to the kernel $K_{d}$.

In this paper, we will focus on the performance of two-layer neural networks on one-dimensional data, i.e., we are more interested in $d=1$. The following theorem summarizes  the spectral properties of $K_{1}$ needed in this paper whose proof is deferred to Appendix \ref{app:ntk properties}.

\begin{theorem}[Spectral properties of $K_{1}$]\label{thm:spectral:d=1:L=1}
	The following properties hold for $K_1$.
	
	$i)$ Let $\bm{X}=\{x_{1},...,x_{n}\} \subseteq [0,\pi]$ and $d_{\min}=\min_{i\neq j}|x_{i}-x_{j}|$. The minimum eigenvalue $\lambda_{\min}$  of $K_{1}(\bm{X},\bm{X})$ satisfies that
	\begin{equation}
		c_1 d_{\min} \leq \lambda_{\min}\leq C_1 d_{\min}
	\end{equation}
	for some absolute constants $c_1$ and $C_1$.
 
	$ii)$     Let $\{\lambda_{j}, j \geq 1\}$ be the eigenvalues associated to the NTK $K_{1}$ defined on $[0,1]$. Then we have 
	\begin{equation}
		 \frac{c_2}{j^{2}}\leq \lambda_j \leq \frac{C_2}{j^{2}},~~ j \geq 1
	\end{equation}
	for some absolute constants $c_{2}$ and $C_{2}$.
\end{theorem}

Theorem \ref{thm:spectral:d=1:L=1} $i)$ shows that the minimum eigenvalue of the gram matrix $K_{1}(\bm{X},\bm{X})$ depends on the minimum distance between samples. In particular, for the equally distanced one-dimensional data with $x_{i}=\frac{i-1}{n-1}$ for $i \in [n]$, the minimum distance $d_{\min}=\frac{1}{n-1}$ and the minimum eigenvalue is $\propto \frac{1}{n}$.
Theorem \ref{thm:spectral:d=1:L=1} $ii)$ states that the eigenvalue decay rate (EDR) is 2, an important quantity in performing the kernel regression. This theorem not only provides us with necessary results for this paper, but it also provides us some guidance to make reasonable conjectures for the spectral properties of $K_{d,L}$, the NTK associated to the $L$-layer neural network defined on $\mathbb{R}^{d}$.

%% file: main_result.sava.tody.tex
When the loss function $\hat{\mathcal{L}}_{n}$ is viewed as a function defined on the parameter space $\Theta$, it induces a gradient flow in $\Theta$ given by
\begin{equation}\label{nn:theta:flow}
\begin{aligned}
      \dot{\bm{\theta}}(t) &=\frac{\dx}{\dx t}\bm{\theta}(t)=-\nabla_{\bm{\theta}}\emprisk(f_{\bm{\theta}(t)}^{m})= - \frac{1}{n}\nabla_{\bm{\theta}} f_{\bm{\theta}(t)}^{m}(\bm{X})\tran (f_{\bm{\theta}(t)}^{m}(\bm{X})-\bm{y})
\end{aligned}
\end{equation}
where we emphasize that $\nabla_{\bm{\theta}} f_{\bm{\theta}(t)}^{m}(\bm{X})$ is an $n\times ((d+2)m+1)$ matrix.
When the loss function $\hat{\mathcal{L}}_{n}$ is viewed as a function defined on $\mathcal{F}^{m}$, the space consisting of all two-layer neural networks $f_{\bm{\theta}}^{m}$, it induces a gradient flow in $\mathcal{F}^{m}$ given by
\begin{equation}\label{nn:f:flow}
\begin{aligned} 
\dot{f}_{\bm{\theta}(t)}^{m}(\bm{x}) &=\frac{\dx }{\dx t}f_{\bm{\theta}(t)}^{m}(\bm{x})=\nabla_{\bm{\theta}} f_{\bm{\theta}(t)}^{m}(\bm{x})\dot{\bm{\theta}}(t)= -\frac{1}{n} K_{\bm{\theta}(t)}^{m}(\bm{x},\bm{X}) (f_{\bm{\theta}(t)}^{m}(\bm{X})-\bm{y}),
\end{aligned}
\end{equation}
where we emphasize that $K_{\bm{\theta}(t)}^{m}(\bm{x},\bm{X}) =\nabla_{\bm{\theta}} f_{\bm{\theta}(t)}^{m}(\bm{x}) \nabla_{\bm{\theta}} f_{\bm{\theta}(t)}^{m}(\bm{X})\tran$ is a $1\times n$ vector. We further introduce a time-varying kernel function
\begin{align*}
K_{\bm{\theta}(t)}^{m}(\bm{x},\bm{x}')=\langle\nabla_{\bm{\theta}}f_{\bm{\theta}(t)}^{m}(\bm{x}),\nabla_{\bm{\theta}}f_{\bm{\theta}(t)}^{m}(\bm{x}')\rangle.
\end{align*}
To avoid potential confusion with the NTK, we refer to this time-varying kernel $K_{\bm{\theta}(t)}^{m}$ as the NNK.

It is clear from the gradient flow equations \eqref{nn:theta:flow} and \eqref{nn:f:flow} that the training process of neural networks is  determined by the random initialization of $\bm{\theta}(0)$.
To avoid unnecessary digression, we adopt a special initialization widely used in literature so that $f_{\bm{\theta}(0)}^{m}(\bm{x})=0$
\cite{hu2019simple, chizat2019lazy}.
More precisely, for a two-layer neural network with width $2m$, we assume that $a_{r}(0)=-a_{r+m}(0),\bm{w}_{r,(j)}(0)=\bm{w}_{r+m,(j)}(0),b_{r}(0)=b_{r+m}(0) \sim \mathcal{N}(0,1)$ for $r\in [m]$, $j\in [d]$ and $b =0$. 

Since it is hard to find an explicit solution of the highly non-linear equations \eqref{nn:theta:flow} and  \eqref{nn:f:flow} , researchers looked for approximated solutions characterizing the asymptotic behavior of the exact solution of these equations (see e.g.,\cite{mei2019mean, karakida2019universal, sirignano2022mean, eldan2021non}).
When the width $m\rightarrow\infty$,  
Jacot et al. \cite{jacot2018neural} observed that the NTK is the limit of NNK.
The time-independent kernel NTK offered us a simplified version of the equations \eqref{nn:theta:flow} and \eqref{nn:f:flow}:
\begin{equation}\label{ntk:f:flow}
    \begin{aligned}
    \dot{f}^{\NTK}_{t}(\bm{x})=\frac{\dx}{\dx t}f^{\NTK}_{t}(\bm{x})=-\frac{1}{n}K_d(\bm{x},\bm{X})(f^{\NTK}_{t}(\bm{X})-\bm{y})
    \end{aligned}
\end{equation}
where  $K_d(\bm{x},\bm{X}) = (K_d(\bm{x},\bm{x}_1),\dots,K_d(\bm{x},\bm{x}_n))\in \mathbb{R}^{1\times n}$.
This equation is defined on the space $\mathcal{H}_{d}$, the RKHS associated to the kernel $K_{d}$. 
The equation \eqref{ntk:f:flow} is called the gradient flow associated to the kernel regression with respect to the kernel $K_{d}$. 
Similar to the special initialization of the  neural network function, we assume that the initial function $f^{\NTK}_{0}(\bm{x})=0$.

Though it is hard to solve the equations \eqref{nn:theta:flow} and \eqref{nn:f:flow}, the equation \eqref{ntk:f:flow}  can be solved explicitly:
\begin{equation}\label{ntk:solution}
    f_t^{\NTK}(\bm{x})=K_d(\bm{x},\bm{X})K_d(\bm{X},\bm{X})^{-1}(\bm{I}-e^{-\frac{1}{n}K_d(\bm{X},\bm{X})t})\bm{y},
\end{equation}
which we refer to as NTK regression function at time $t$ in this paper. Researchers proved that for any given $\delta\in(0,1)$, for every $\bm{x}\in\mathcal{X}$, the neural network $f_{\bm{\theta}(t)}^{m}(\bm{x})$ can be well approximated by the NTK regression function $f_{t}^{\NTK}(\bm{x})$ when $m$ is sufficiently large, i.e., one has the pointwise convergence (see e.g., \cite{lee2019wide, arora2019exact}):
\begin{equation}\label{eq:solution:convergence}
\sup_{t\geq 0}|f^{m}_{\bm{\theta}(t)}(\bm{x})-f_{t}^{\NTK}(\bm{x})|=o_{m}(1)
\end{equation} 
holds with probability at least $1-\delta$.

One of our main technical contributions is that the above convergence is uniform with respect to all $t\geq0$ and all $\bm{x}\in\mathcal{X}$. Thus,  the excess risk $\excrisk(f_{\bm{\theta}(t)}^{m})$ of the wide two-layer ReLU neural network $f_{\bm{\theta}(t)}^{m}$ could be well approximated by the excess risk $\excrisk(f_{t}^{\NTK})$ of the NTK regression function $f_{t}^{\NTK}$.

\begin{theorem}\label{thm:risk:approx}
    Given the training data $\{(\bm{x}_{i},y_{i}),i\in[n]\}$, for any $\epsilon>0$, when the width $m$ of the two-layer ReLU neural network is sufficiently large, we have 
    \begin{equation*}
        \sup_{t\geq 0}|\excrisk(f_{\bm{\theta}(t)}^{m})-\excrisk(f_{t}^{\NTK})|\leq \epsilon 
    \end{equation*}
    holds with probability at least $1-o_{m}(1)$ where the randomness comes from the initialization of the parameters.
\end{theorem}

\begin{proof}
According to the formula of $f_{t}^{\NTK}(\bm{x})$ in \eqref{ntk:f:flow}, we have 
 \begin{equation*}
    \begin{aligned}
     |f_{t}^{\NTK}(\bm{x})|
     \leq&\|K_{d}(\bm{x},\bm{X})\tran\|_{2}\|K_{d}(\bm{X},\bm{X})^{-1}\|_{2}\|\bm{I}-e^{-\frac{1}{n}K_{d}(\bm{X},\bm{X})t}\|_{2}\|\bm{y}\|_{2}\\
     \leq & C\sqrt{n}(\lambda_{\min}(K_{d}(\bm{X},\bm{X})))^{-1}\|\bm{y}\|_{2}
     \end{aligned}
 \end{equation*}
 for some constant $C$ depending only on $B$.
Since $f_{\star}$ is continuous on $\mathcal{X}$, it is bounded. 
Thus, we know that
\begin{equation*}
    \Delta_{\star}:=\sup_{t\geq 0}\sup_{\bm{x}\in\mathcal{X}}|f_{t}^{\NTK}(\bm{x})-f_{\star}(\bm{x})|\leq C\max\{1,\sqrt{n}(\lambda_{\min}(K_{d}(\bm{X},\bm{X})))^{-1}\|\bm{y}\|_{2}\}   
\end{equation*}
for some constant $C$ depending only on $B$ and $f_{\star}$.
Let $\Delta(t,\bm{x})=f_{\bm{\theta}(t)}^{m}(\bm{x})-f_{t}^{\NTK}(\bm{x})$.  
   For any $\epsilon>0$, we know that for sufficiently large $m$,
   \begin{align*}
       \left|\excrisk(f_{\bm{\theta}(t)}^{m})-\excrisk(f_{t}^{\NTK})\right|=&\left|\int_{\mathcal{X}}\Delta(t,\bm{x})^{2}\dx\mu_{\mathcal{X}}(\bm{x})+\int_{\mathcal{X}}\Delta(t,\bm{x})(f_{t}^{\NTK}(\bm{x})-f_{\star}(\bm{x}))\dx \mu_{\mathcal{X}}(\bm{x})\right|\\
       \leq & \int_{\mathcal{X}} |\Delta(t,\bm{x})\left(\Delta(t,\bm{x})+(f_{t}^{\NTK}(\bm{x})-f_{\star}(\bm{x}))\right|\dx \mu_{\mathcal{X}}(\bm{x})\leq \epsilon
   \end{align*}
 where the last line follows from  Proposition \ref{prop:funct:approx} that for sufficiently large $m$, we have 
\begin{equation*}
     \sup_{t\geq 0}\sup_{\bm{x}\in\mathcal{X}}|\Delta(t,\bm{x})|\leq \min\left\{\Delta_{\star},\epsilon/\Delta_{\star}\right\}
\end{equation*}
 with probability at least $1-o_{m}(1)$. 
\end{proof}

Though many works tried to study the generalization performance of the neural network through that of the NTK regression \cite{jacot2018neural, vyas2022limitations, arora2019exact, arora2019fine, hu2021regularization, suh2022nonparametric, montanari2022interpolation}, to the best of our knowledge, most of them took the convergence $\excrisk(f_{\bm{\theta}(t)}^{m})\xrightarrow{m\rightarrow \infty}\excrisk(f_{t}^{\NTK})$ for granted. 
Theorem \ref{thm:risk:approx}
fills this long-standing gap in the literature and ensures the validity of focusing on the generalization properties of NTK regression. The following two propositions are not only of technical interests but also serve as a cornerstone in the future studies of the ``lazy training regime''.

\begin{proposition}\label{prop:funct:approx}
    Given the training data $\{(\bm{x}_{i},y_{i}),i\in[n]\}$, for any $\epsilon>0$, if the width $m$ of the two-layer ReLU neural network is sufficiently large, then  
    \begin{equation*}
        \sup_{t\geq 0}\sup_{\bm{x}\in\mathcal{X}}| f^{m}_{\bm{\theta}(t)}(\bm{x})-f_{t}^{\NTK}(\bm{x})|\leq \epsilon
    \end{equation*}
    holds with probability at least $1-o_{m}(1)$ where the randomness comes from the initialization of the parameters.
\end{proposition}

\begin{proposition}\label{prop:kernel:approx}
     Given the training data $\{(\bm{x}_{i},y_{i}),i\in[n]\}$, for any $\epsilon>0$, if the width $m$ of the two-layer ReLU neural network is sufficiently large, then
    \begin{equation*}
    \sup_{t\geq 0}\sup_{\bm{x},\bm{x}'\in\mathcal{X}}| K_{\bm{\theta}(t)}^{m}(\bm{x},\bm{x}')-K_d(\bm{x},\bm{x}')|\leq\epsilon
    \end{equation*}
    holds with probability at least $1-o_{m}(1)$ where the randomness comes from the initialization of the parameters.
\end{proposition}

The proofs of Proposition \ref{prop:funct:approx} and Proposition \ref{prop:kernel:approx} are deferred to Supplementary Material. The existing results only showed that as $m \to\infty$, the time-varying NNK $K_{\bm{\theta}(t)}^{m}$ and the neural network  $f_{\bm{\theta}(t)}^{m}$ converge pointwise to the time-invariant NTK $K_{d}$ and the NTK regression function $f_{t}^{\NTK}$ (\cite{jacot2018neural, du2018gradient,lee2019wide, arora2019exact}). 
The Proposition \ref{prop:funct:approx}  and \ref{prop:kernel:approx} proved a much stronger uniformly convergence statement.

\section{The generalization performance of wide neural networks on $\mathbb{R}$}\label{subsec:generalization}

In order to get a meaningful discussion about the generalization performance of a neural network, we have to specify a class of functions to which $f_{\star}$ belongs. 
In this paper, we make the following assumption:
\begin{assumption}\label{assump:f_star}
    The regression function $f_{\star}\in \mathcal{H}_{1}$ and $\| f_{\star}\|_{\mathcal{H}_{1}}\leq R$ for some constant $R$, where $\mathcal{H}_{1}$ is the RKHS associated to the kernel $K_{1}$.
\end{assumption}

Proposition \ref{prop:funct:approx} and Theorem \ref{thm:risk:approx} shows that $f^{m}_{\bm{\theta}(t)}$ uniformly converges to $f^{\NTK}_{t}$ and $\excrisk(f_{\bm{\theta}(t)}^{m})$ is well approximate by $\excrisk(f_{t}^{\NTK})$, thus we can focus on studying the generalization ability of the NTK regression function $f^{\NTK}_{t}$. 
It would be easier to stick with the usual assumptions appeared in the kernel regression literature (see e.g., \cite{caponnetto2007optimal, yao2007early, raskutti2014early, blanchard2018optimal, lin2020optimal}). This is exactly the Assumption \ref{assump:f_star}.

\subsection{Wide neural networks with early stopping achieve the minimax rate}

Early stopping, an implicit regularization strategy, is widely applied in training various models such as kernel ridgeless regression, neural networks, etc. 
Lots of solid research has provided theoretical guarantees for early stopping (see e.g. \cite{yao2007early, raskutti2014early, blanchard2018optimal, lin2020optimal}), where the optimal stopping time is depending on the decay rate of eigenvalue associated to the kernel. 
Note that Theorem \ref{thm:spectral:d=1:L=1} gives us the eigenvalue decay rate of $K_{1}$ and Theorem \ref{thm:risk:approx} guarantees the excess risk of the NTK regression function $f_{t}^{\NTK}$ is an accurate alternative of the excess risk of the neural network $f^{m}_{\bm{\theta}(t)}$, thus we have the following Theorem \ref{thm:nn:early:stopping:d=1}.

\begin{theorem}\label{thm:nn:early:stopping:d=1} Suppose Assumption \ref{assump:f_star} holds and we observed $n$ i.i.d. samples $\{(\bm{x}_{i},y_{i}),i\in[n]\}$  from the model \eqref{equation:true_model}. 
    For any given $\delta\in(0,1)$, if one trains a two-layer neural network with width $m$ that is sufficiently large and stops the gradient flow at time $t_{\star}\propto n^{2/3}$, then for sufficiently large $n$, there exists a constant $C$ independent of $\delta$ and $n$, such that   
    \begin{equation}\label{eq:excrisk nn}
\excrisk(f_{\bm{\theta}(t_{\star})}^{m})\leq Cn^{-\frac{2}{3}}\log^{2}\frac{6}{\delta}
    \end{equation}
    holds with probability at least $(1-\delta)(1-o_{m}(1))$ where the randomness comes from the joint distribution of the random samples and the random initialization of parameters in the  neural network $f_{\bm{\theta}(0)}^{m}$.
\end{theorem}

Researchers have established (\cite{blanchard2018optimal}) the following minimax rate of regression over the RKHS $\mathcal{H}_{1}$ associated to $K_{1}$:
    \begin{equation}    \inf_{\hat{f}_{n}}\sup_{f_{\star}\in\mathcal{H}_{1},\| f_{\star}\|_{\mathcal{H}_{1}}\leq R}\Expc\excrisk(\hat{f}_{n})=\Omega(n^{-\frac{2}{3}}).
    \end{equation}
Thus, we have proved that training a wide neural two-layer neural network with the early stopping strategy achieves the optimal rate.

The proof of Theorem \ref{thm:nn:early:stopping:d=1} can be found in Supplementary Material. Theorem \ref{thm:nn:early:stopping:d=1} rigorously shows the fully trained wide two-layer ReLU neural network with early stopping is minimax rate optimal.

\subsection{Overfitted Neural Networks generalize poorly}

In this subsection, we are more interested in the generalization performance of $f_{\bm{\theta}(t)}^{m}(x)$ for sufficiently large $t$ such that $f_{\bm{\theta}(t)}^{m}(x)$ can (nearly) fit the given data. 

To be more concrete, suppose that we observed  $n$ equally-distanced one-dimensional data $\{(x_{i},y_{i}) \mid x_{i}=\frac{i-1}{n-1}, i \in [n]\}$. The following theorem shows that  $f_{\bm{\theta}(t)}^{m}(x)$ almost linearly interpolates these data points when $t$ is sufficiently large, therefore it can not generalize well.
We remind that a linear interpolation of the equally-distanced  data is given by:
\begin{align}
    f_{\LI}(x)=y_{i}+\frac{y_{i+1}-y_{i}}{x_{i+1}-x_{i}}(x-x_{i}), \mbox{ when } x\in [x_{i},x_{i+1}].
\end{align}

\begin{theorem}[Overfitted networks generalize poorly]\label{thm:bad_gen}
    Suppose that we have observed $n$ data $\{(x_{i},y_{i}), i\in [n]\}$ from the model \eqref{equation:true_model} where $x_i=\frac{i-1}{n-1}$, $i\in [n]$. When the width $m$ is sufficiently large, the following statements hold.

        $i)$  There exist some absolute constants $C_{1}$, $C_{2}$ and $C_{3}$ such that for any $t>C_{1}n^{2}\log n$, we have
        \begin{align}
 \sup_{x\in [0,1]}|f_{\bm{\theta}(t)}^{m}(x)-f_{\LI}(x)|\leq C_3 \sqrt{\log n} /(n-1)^{2}
\end{align}
        holds with probability at least $1-\frac{C_{2}}{n}$.

        \vspace{1mm}
        $ii)$ There exist some positive constant $C_{4}$ depending only on $\sigma$ and absolute constant $C_{5}$ 
        such that for any $t>C_{1}n^{2}\log n$, we have 
        $\excrisk(f_{\bm{\theta}(t)}^{m}) \geq C_4$
        holds with probability at least $1-\frac{C_{5}}{n}$.
\end{theorem}

Theorem \ref{thm:bad_gen} $i)$ shows that the overfitted neural network is nearly a linear interpolation (e.g., shown in Figure \ref{fig:LI}(a)). To the best of our knowledge, this is the first result explicitly showing how the overfitted neural network interpolates the data. 
We have to emphasize that not every kernel interpolation  (kernel ridgeless regression) is nearly linear interpolation. For example, it is clear that the radial basis function(RBF) kernel interpolation interpolates the data nonlinearly, shown in Figure \ref{fig:LI}(b).
Figure \ref{fig:LI}(c) shows that the maximum gap between the overfitted neural network and linear interpolation is exactly $O(\frac{1}{n^2})$, which is in line with the Theorem \ref{thm:bad_gen} $i)$. 
 
Theorem \ref{thm:bad_gen} $ii)$ shows that the generalization error of overfitted neural networks has a constant lower bound at least for the equally-distanced data. It strongly suggests that overfitted neural networks can not generalize well, which contradicts the ``benign overfitting phenomenon''. So in the next section, we will provide an explanation to reconcile our theoretical result and the ``benign overfitting phenomenon''.

\begin{figure}[htbp]
    \begin{minipage}[t]{0.33\linewidth}
      \centering
      \includegraphics[width=\textwidth]{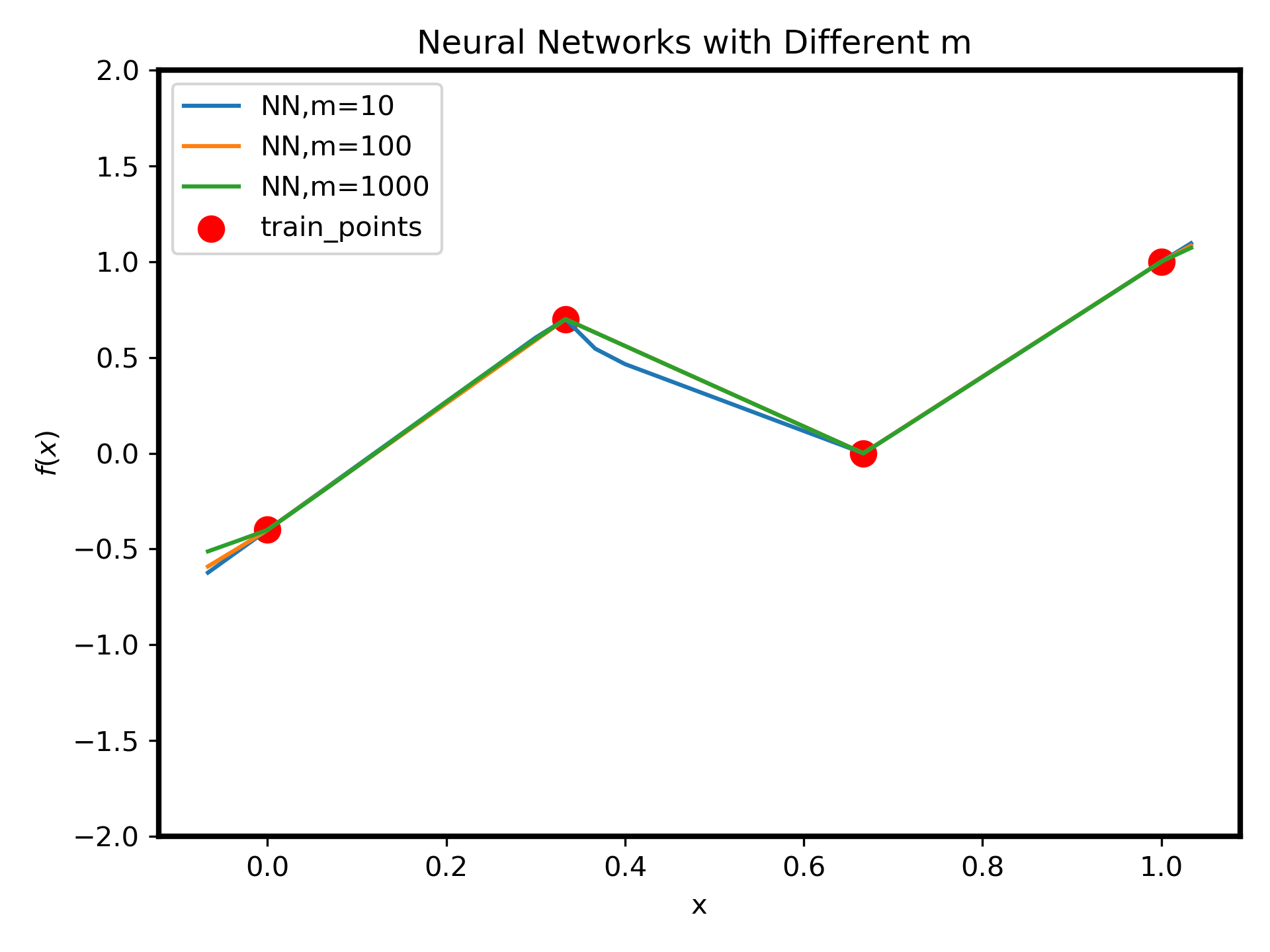}
      \centerline{(a)}
  \end{minipage}%
  \begin{minipage}[t]{0.33\linewidth}
      \centering
      \includegraphics[width=\textwidth]{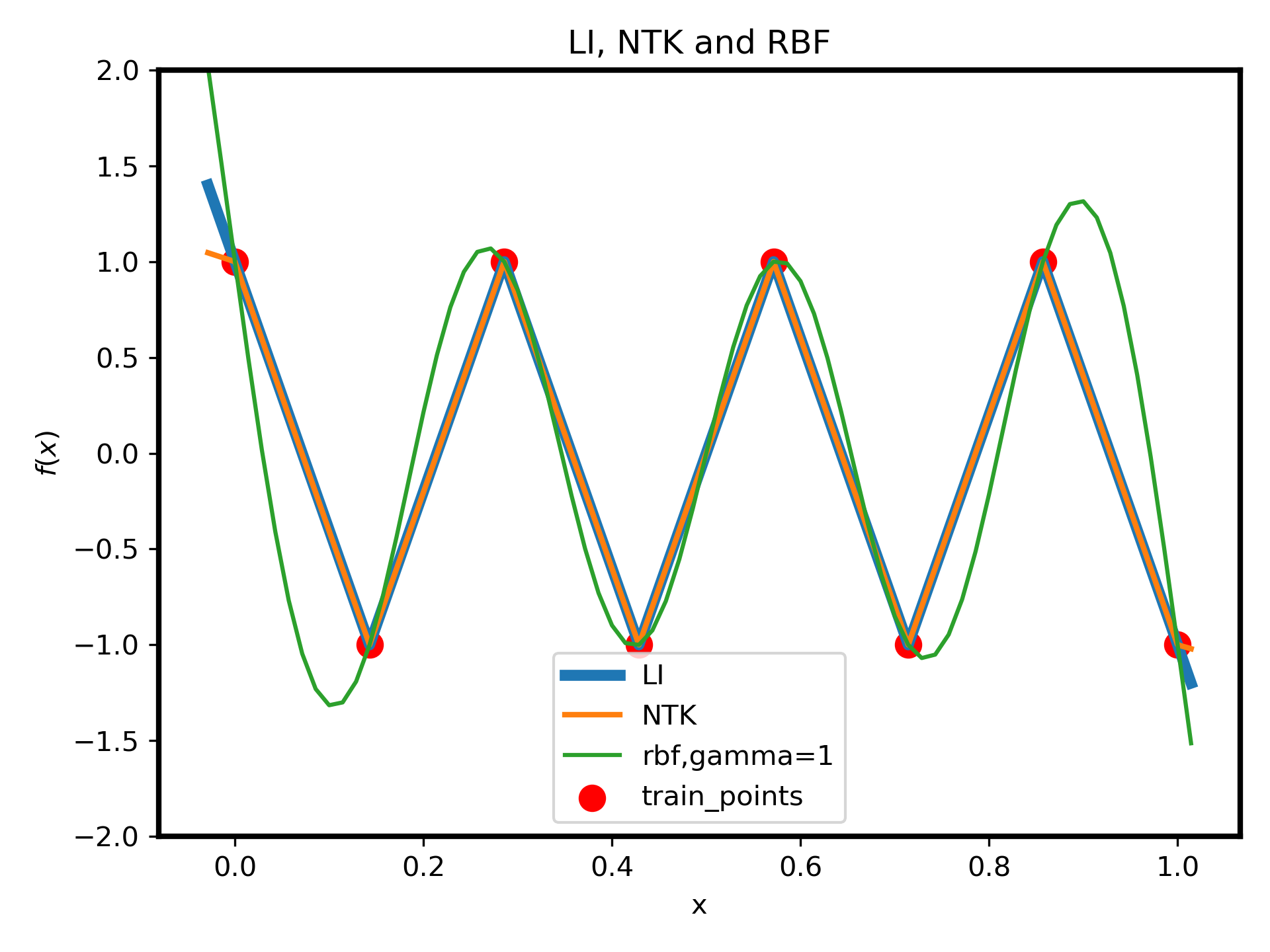}
      \centerline{(b)}
  \end{minipage}%
  \begin{minipage}[t]{0.33\linewidth}
      \centering
      \includegraphics[width=\textwidth]{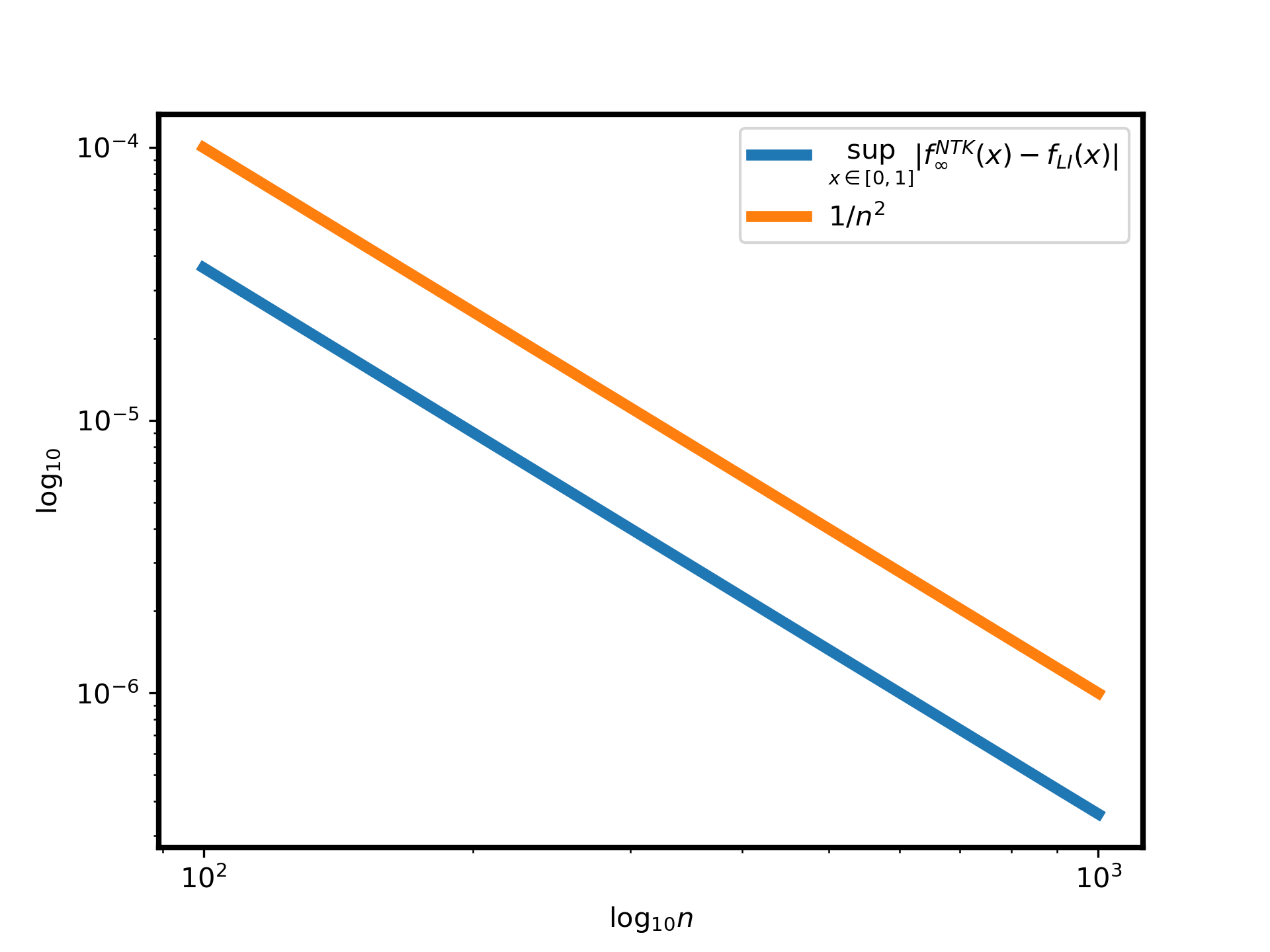}
      \centerline{(c)}
  \end{minipage}
   \caption{(a): The interpolation ways of two-layer neural networks with different widths; (b): The interpolation ways of linear interpolation, $f^{\NTK}_{\infty}(x)$ and RBF kernel regression with the bandwidth $\gamma=1$; (c): The maximum gap between $f^{\NTK}_{\infty}$ and linear interpolation.
   The input of the training data is the equally-distanced one-dimensional data $\{x_{i}=\frac{i-1}{n-1}, i \in [n]\}$ with randomly selected labels.
   The sample size $n=100,200,\dots,1000$}\label{fig:LI}
\end{figure}

%% file: explanation.tex
In Section \ref{sec:main_result},  we have shown that the generalization ability of a wide neural network depends on when the training process is stopped.  $i)$ If the training process stops at a properly chosen time, the generalization ability of the resulting neural network can achieve the minimax rate; $ii)$ If the training process stops when the loss is near zero (or overfitting the data), the resulting neural network can not generalize well. The latter statement, however, contradicts to the reported ``benign overfitting phenomenon'' where overfitted neural networks do generalize well in certain situations. 
To resolve this annoying contradiction, 
we scrutinize the reported observations again and propose a hypothesis on the role of signal strength played in these observations.
We believe this explanation reconciles the conflict between our theory and the widely observed ``benign overfitting phenomenon''.

\subsection{Three stopping rules}
We first emphasize that a subtle difference between the classification problem and the regression problem might be ignored in the reported experiments. To be more concrete, 
we  have three choices of stopping times in the classification problem:
$i)$ the stopping time $t_{\text{opt}}$  where the training process stopped at the time suggested by our theory ; 
$ii)$ the stopping time $t_{\text{loss}}$ where the training process stopped till the value of the loss function nears zero;  
$iii)$ the stopping time $t_{\text{label}}$ where the training process stopped till the label error rate nears zero. 
 Most of the reported experiments in ``benign overfitting phenomenon'' utilize the stopping time $t_{\text{label}}$ and claim that the resulting neural network can overfit the data and generalize well \cite{zhang2016understanding}. 
 
Our theoretical results suggested that the neural network at the stopping time $t_{\text{opt}}$ has the best generalization ability and the neural network at the stopping time $t_{\text{loss}}$ can not generalize. 
Thus, there might be a significant difference between the stopping time $t_{\text{label}}$ and $t_{\text{loss}}$. 
This difference can clearly be seen from a toy example consisting of 4 data points $\{(0,0),(\frac{1}{3},1),(\frac{2}{3},0), (1,1) \}$.  We fit the data with a two-layer neural network with width $m=1000$ with respect to the square loss ( regression ) and cross-entropy loss (classification ) separately. 
The results are reported in figure \ref{fig: interpolation vs acc}. 
It is clear that for both loss functions, the stopping time $t_{\text{label}}$ is much earlier than $t_{\text{loss}}$. 
The fact that the stopping time $t_{\text{label}}$ may be far earlier than $t_{\text{loss}}$  partially explained why the training stopped at time $t_{\text{label}}$ produces a neural network with some generalization ability; if the stopping time $t_{\text{label}}$ is close to $t_{\text{opt}}$, then the training process stopped at time $t_{\text{label}}$ produces a neural network with the optimal generalization ability.

\begin{figure}

    \centering
    \begin{minipage}[b]{0.8\textwidth}
    \includegraphics[width=1\textwidth]{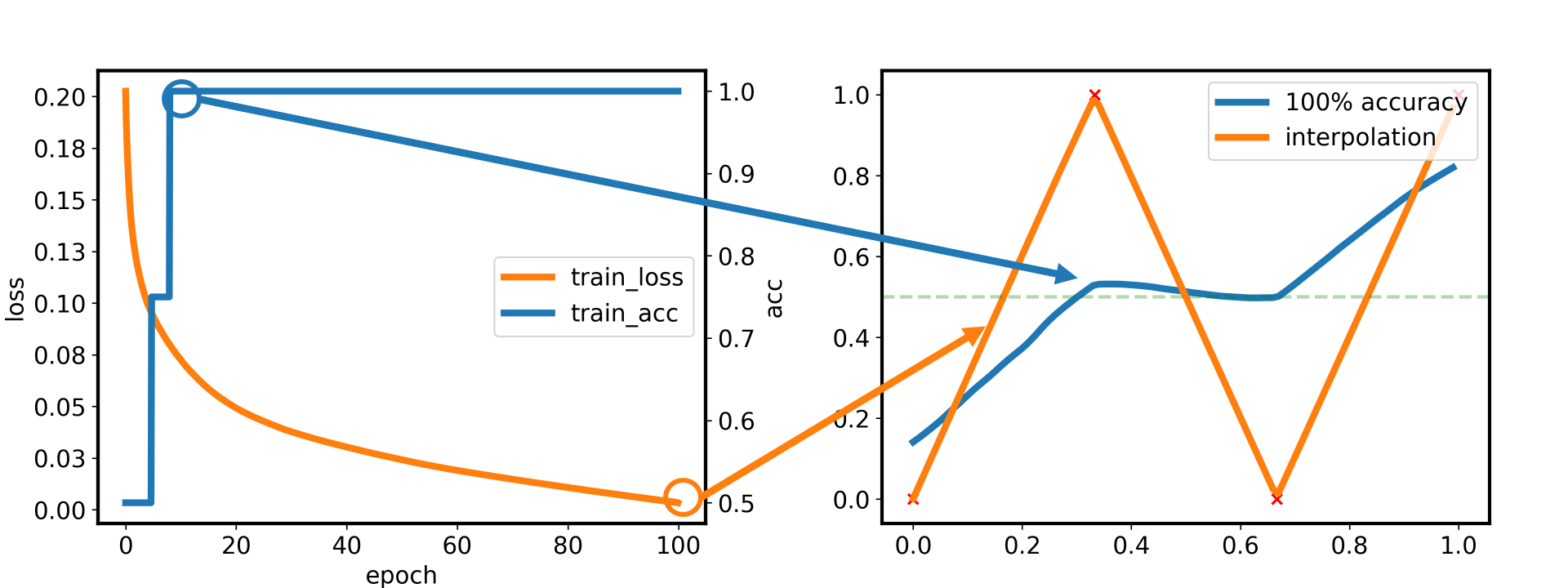}
    \end{minipage}
    \begin{minipage}[b]{0.8\textwidth}
    \includegraphics[width=1\textwidth]{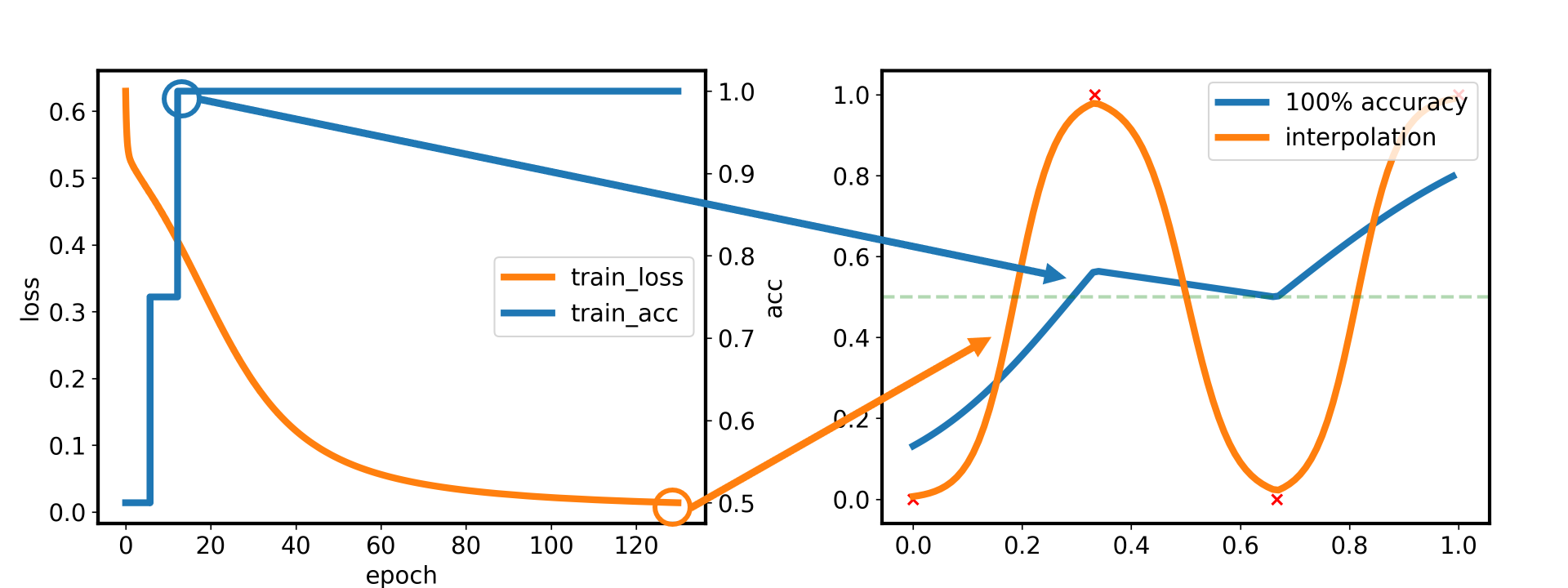}
    \end{minipage}
    
     \caption{Overfitting vs. 100 \% label accuracy for regression and classification problems: The upper figures are from the regression setup and the lower ones are from the classification setup. The left figures present the loss and the accuracy of 4 training data points. The right figures show the function in the epoch with 100 \% label accuracy and in the interpolation regime.} \label{fig: interpolation vs acc}
    \end{figure}

%% file: Experiment.tex
We have seen the three different stopping rules and the subtle difference between $t_{\mathrm{label}}$ and $t_{\mathrm{loss}}$.
What makes the stopping time $t_{\mathrm{label}}$ far from $t_{\mathrm{loss}
}$ or not?  We hypothesize that: $i)$ if the signal strength is strong, then $t_{\mathrm{label}}$ nears $t_{\mathrm{opt}}$; $ii)$ if the signal strength is weak, then $t_{\mathrm{label}}$ nears $t_{\mathrm{loss}}$. 
We justified this hypothesis through various experiments.

\vspace{3mm}
\noindent{\it  $\bullet$ Synthetic Data:} 
Suppose that $\bm{x}_i , 1\leq i \leq 100$ are i.i.d. sampled from $\unif((0,1)^3)$ and 
\begin{equation*}
    y_{i}=f_{\star}(x_{i})=\lfloor2 \bm{x}_{i,(1)}\rfloor+2\lfloor2 \bm{x}_{i,(2)}\rfloor+4\lfloor2 \bm{x}_{i,(3)}\rfloor\in \{0,1,\cdots,7\}, 1\leq i \leq 100.
\end{equation*} 
For a given $p\in [0,1]$, we corrupt every label $y_{i}$ of the data with probability $p$ by a uniform random integer from $\{0,1,\cdots,7\}$.

For corrupted data with $p\in \{0,0.3,0.6\}$, 
we train a two-layer neural network (width $m=10000$) with the squared loss and collect the testing accuracy and loss based on 1000 testing data points. The results are reported in Figure \ref{fig: MLP_diff_noise}(a).
We also execute the same experiment with the cross-entropy loss and report the results in Figure \ref{fig: MLP_diff_noise}(b). 

\begin{figure}[htbp]
  \begin{minipage}[t]{0.5\linewidth}
      \centering
      \includegraphics[width=\textwidth]{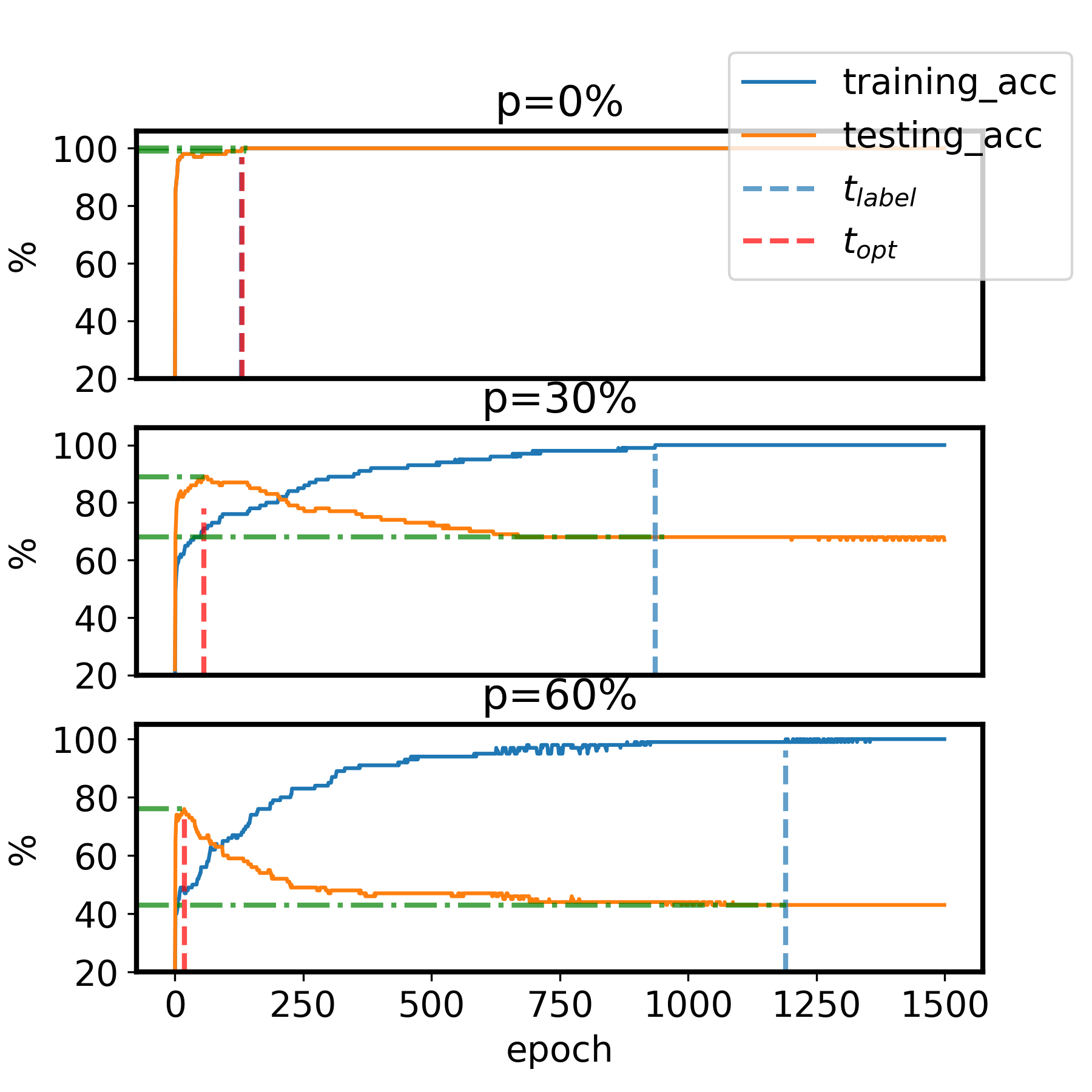}
      \centerline{(a) the results of the squared loss}
  \end{minipage}%
  \begin{minipage}[t]{0.5\linewidth}
      \centering
      \includegraphics[width=\textwidth]{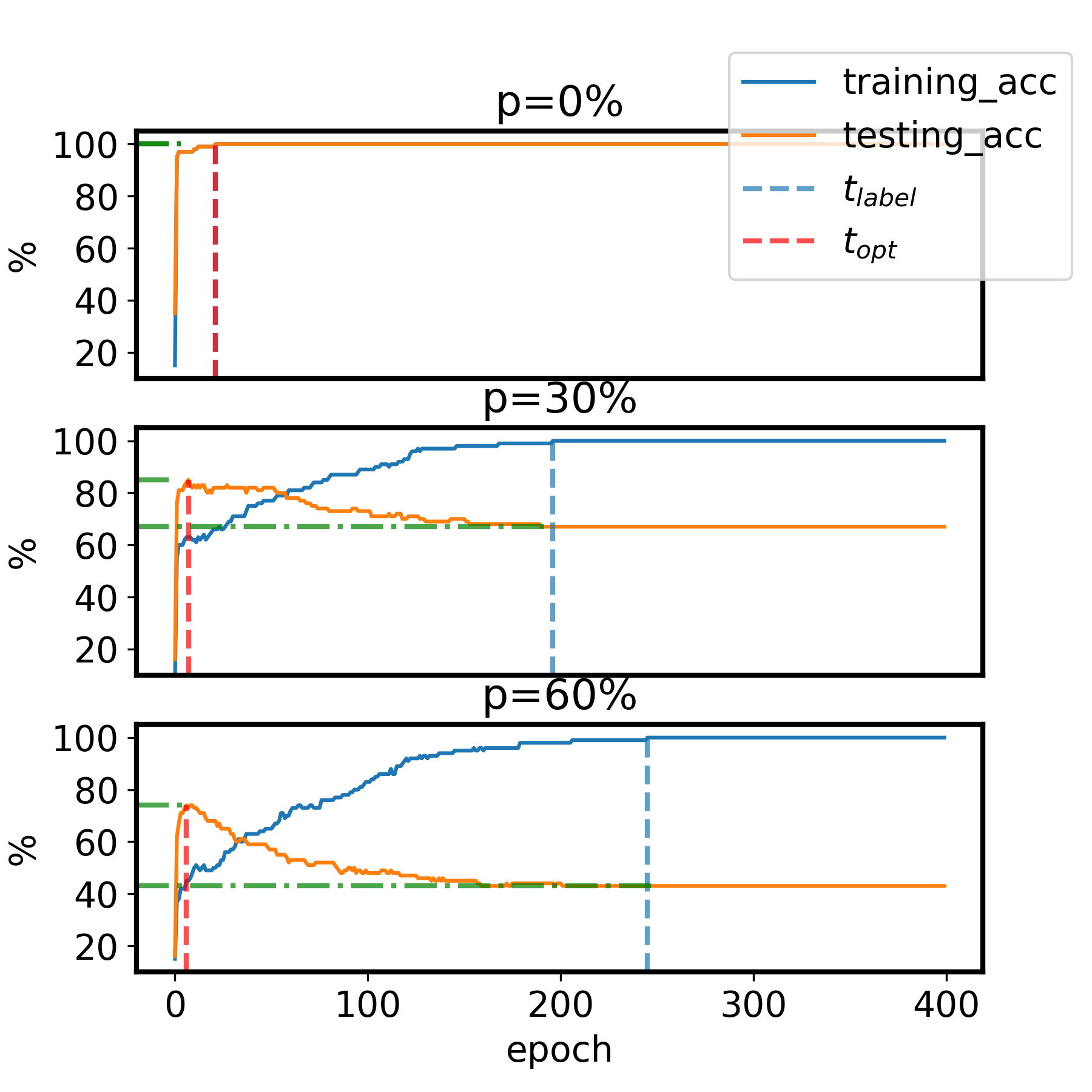}
      \centerline{(b) the results of the cross-entropy loss}
  \end{minipage}
  \caption{Synthetic Data: the gap between $t_{\mathrm{label}}$ and $t_{\mathrm{opt}}$ is increasing when the label corruption ratio $p$ is increasing. When $p=0$, the gap between $t_{\mathrm{label}}$ and $t_{\mathrm{opt}}$ and the gap between the corresponding testing accuracies are extremely small, i.e., we observed the ``benign overfitting''.}
  \label{fig: MLP_diff_noise}
\end{figure}

\vspace{3mm}
{\it \noindent $\bullet$ Real Data:} Inspired by the numerical studies (the classification setup) in \cite{zhang2016understanding}, we perform the experiments on CIFAR-10 with AlexNet.
Again, we corrupt the data with $p=\{0,0.3,0.6\}$ and apply the SGD to training Alex with the momentum parameter of 0.9, the initial learning rate of 0.01 and the decay factor 0.95 per training epoch. The results are reported in Figure \ref{fig: Alexnet_cifar10_diff_noise_ce}. 

  \begin{figure}[htbp]
        \centering
        \includegraphics[width=0.6\textwidth]{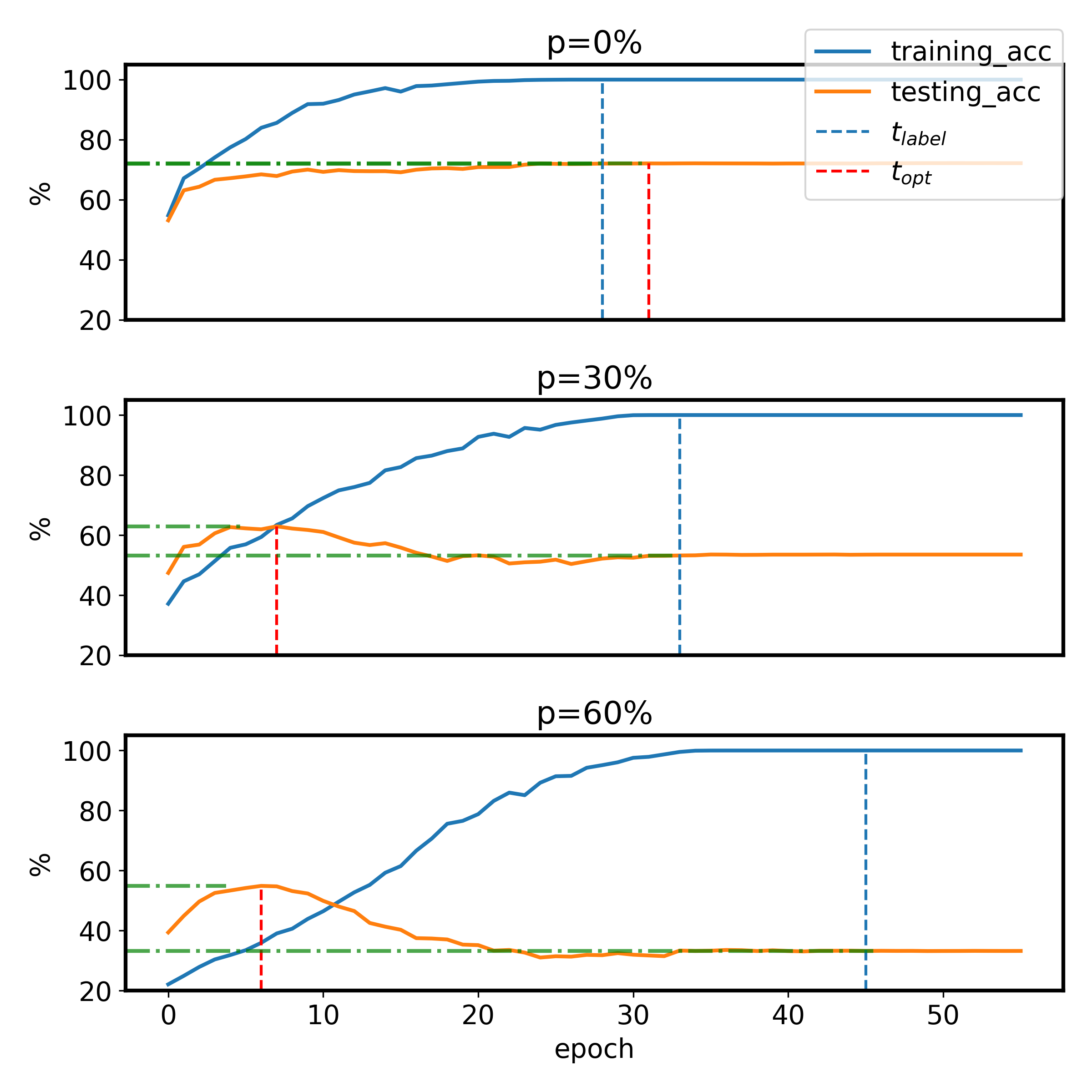}
        \centerline{(b) Generalization gap with different label corruption ratios $p$}
    \caption{AlexNet on CIFAR-10 (cross-entropy loss): the time gap between $t_{label}$ and $t_{opt}$ is increasing when the label corruption ratio is increasing.}
    \label{fig: Alexnet_cifar10_diff_noise_ce}
  \end{figure}
The above experiments support our hypothesis and reconcile the conflict between the ``benign overfitting phenomenon'' and our theory: if the signal strength is strong, ``benign overfitting'' holds and our theoretical results still work; if the signal strength is weak, ``benign overfitting'' can not hold anymore and our theoretical results explain the reason for the failure of ``benign overfitting''.


%% file: Discussion.tex
In this paper, we first showed the positive definiteness of the NTK $K_d$ defined on $\mathbb{R}^d$, filling a long-standing gap in the literature. 
We then proved that the NNK uniformly converges to the NTK, which implies that the excess risk of wide neural networks is well approximated by that of the NTK regression function. 
Thus, for two-layer neural networks and one-dimensional data, we could prove that: $i)$ if one stops the training process of wide neural networks at a proper time, the excess risk of the resulting neural network achieves the minimax optimality; $ii)$ an overfitted neural network can not generalize well. 
Finally, we proposed an explanation to reconcile the contradiction between our theoretical result and the widely observed ``benign overfitting phenomenon''.

Though the current work only dealt with the two-layer neural network and one-dimensional data, it is clear that the strategy works for more general neural networks and more complicated data. 
To be more precise, one may try to first show that neural network kernels of CNN, ResNet, etc. uniformly converge to the corresponding neural tangent kernels, then study the spectral properties of NTK such as positive definiteness and eigenvalue decay rate. Thus, we can expect that training a wide deep neural network with the early stopping strategy can produce a neural network with optimal generalization ability.

\begin{acks}[Acknowledgments]
The authors would like to thank the anonymous referees, the Associate Editor and the Editor for their constructive comments that improved the quality of this paper. The corresponding author was supported in part by the National Natural Science Foundation of China (Grant 11971257), Beijing Natural Science Foundation (Grant Z190001), National Key R\&D Program of China (2020AAA0105200), and Beijing Academy of Artificial Intelligence.
\end{acks}

%% file: appendix.tex
\vspace{4mm}

\begin{appendix}
\section{Reproducing Kernel Hilbert Space} \label{sec:RKHS}
In this section, we recollect some essential concepts and theorems in the reproducing kernel Hilbert space (RKHS). For simplicity, we assume that $\mathcal{H}$ is a separable Hilbert space.
\begin{definition}[RKHS and reproducing kernel]\label{def:RKHS}
Let $\mathcal{H}$  be a Hilbert space of functions defined on a non-empty $\mathcal{X}$. It is an RKHS if for all $x\in \mathcal{X}$, there exists a positive constant $M_{x}$, such that
\begin{align}
    |f(x)|\leq M_{x}\|f\|_{\mathcal{H}}, \quad \forall f\in \mathcal{H}.
\end{align}
By Riesz representation theory, for any $x$, there is an element $K(\cdot,x) \in \mathcal{H}$ such that
\begin{align}
    f(x)=\langle f, K(\cdot,x) \rangle_{\mathcal{H}}.
\end{align}
The function $K:\mathcal{X}\times\mathcal{X} \rightarrow \mathbb{R}$ such that 
\begin{align}
    K(x,y) = \langle K(\cdot,x), K(\cdot,y)\rangle_{\mathcal{H}},
\end{align}
is referred to as the reproducing kernel associated with $\mathcal{H}$. It is clear that $K$ is a positive semi-definite kernel on $\mathcal{X}$.
\end{definition}

\begin{lemma}
    Suppose that $\{e_{j},j \geq 1\}$ is an orthonormal basis of $\mathcal{H}$.  Then
    \begin{align}
        K(x,y)=\sum_{j=1}^{\infty}e_{j}(x)e_{j}(y)
    \end{align}
    where the sum on RHS converges in $\mathcal{H}$.
    \proof
    Since $\mathcal{H}$ is an RKHS, the Plancherel theorem shows that
    \begin{align}
        K(x,y) = \sum_{j=1}^{\infty}\langle K(\cdot,x),e_{j}\rangle e_{j}(y)
    \end{align}
    where the sum on the RHS converges in $\mathcal{H}$.
    \qed
\end{lemma}

Suppose that there is a topological structure and a Borel measure (or its completion) $\mu_{\mathcal{X}}$ on $\mathcal{X}$ with $\operatorname{supp}(\mu_{\mathcal{X}})=\mathcal{X}$ such that $\mathcal{X}$ is compact and $K$ is continuous. One then can easily verify that the natural embedding inclusion operator $I_K:\mathcal{H} \to L^2(\mathcal{X},\mu_{\mathcal{X}})$ is a compact operator and the adjoint operator $I_K^*: L^2(\mathcal{X},\mu_{\mathcal{X}}) \to \mathcal{H}$ of  $I_{K}$ is given by: 
\begin{align*}
    I_K^* f(x) =\int_{\mathcal{X}} K(x,x')f(x') \dx \mu_{\mathcal{X}}(x').
\end{align*}
Thus we can define an integral operator
\begin{align}
        T_K = I_K \circ I_K^*: L^2(\mathcal{X},\mu_{\mathcal{X}}) \to L^2(\mathcal{X},\mu_{\mathcal{X}})
\end{align}
which is a positive semi-definite, self-adjoint, compact operator. The spectral theorem of the positive semi-definite, self-adjoint, compact operator shows that  
there exists a set of non-negative numbers $\lambda_{1} \geq \lambda_{2} \geq \cdots$ and an orthonormal basis \ $\{\phi_{j},j \geq 1\}$ of $L^{2} (\mathcal{X},\mu_{\mathcal{X}})$ such that 
\begin{align}
    T_{K}f=\sum_{j=1}^{\infty}\lambda_{j}\left<f,\phi_{j}\right>_{L^{2}}\phi_{j},\quad \forall f\in L^{2}(\mathcal{X},\mu_{\mathcal{X}}).
\end{align}
where the sum on the RHS converges in $L^{2}(\mathcal{X},\mu_{\mathcal{X}})$.
In addition, if the operator $I_{K}$ is injective, then $\{\sqrt{\lambda_{j}}I_{K}^{*}\phi_{j},j \geq 1\}$ is an orthonormal basis of $\mathcal{H}$. Thus, we have
\begin{align}\label{Mercer's decomposition}
    K(x,x')=\sum_{j=1}^{\infty}\lambda_{j}I_{K}^{*}\phi_{j}(x)I_{K}^{*}\phi_{j}(x')
\end{align}
where the sum on the RHS converges in $\mathcal{H}$. Note that for any $f,g\in \mathcal{H}$ and for any $x\in \mathcal{X}$, we have
\begin{align}
    |f(x)-g(x)|=|\left<f-g,K_{x}\right>|\leq M_{x}\|f-g\|_{\mathcal{H}}.
\end{align}
Thus the equation \eqref{Mercer's decomposition} holds pointwise.
This is the celebrated Mercer's decomposition theorem.

The numbers $\{ \lambda_{j},j \geq 1\}$  and the functions $\{I_{K}^{*}\phi_{j},j \geq 1\} \subseteq \mathcal{H}$ are often referred to as the eigenvalues and the eigenfunctions associated to the kernel $K$ respectively. With these eigenvalues and the eigenfunctions, $\mathcal{H}$ can be formulated as 
\begin{equation}
    \mathcal{H} = \left\{\sum_{j=1}^{\infty} c_j I_{K}^{*}\phi_j ~\middle|~ \sum_{j=1}^{\infty}c_j^2/\lambda_j <\infty \right\}.
\end{equation}

\section{Proof of Section \ref{sec:ntk}}\label{app:ntk properties}
\begin{lemma}\label{lem:strict:positive}
 Let $K$ be an inner product kernel on $\mathbb{S}^{d}$, i.e., $K(\bm{x},\bm{x}')=f(\langle \bm{x},\bm{x}'\rangle)$ for some function 
 $f(t):[-1,1]\to \mathbb{R}$ such that $f(t)=\sum^{\infty}_{k=0}a_{k}t^{k}$ where $a_{k}\geq 0$ for any $k\geq 0$. If there are infinity number $k$ such that $a_{k}>0$, then $K$ is positive definite on $\mathbb{S}^{d}_{+}:=\{\bm{x}=(x_{1},\ldots,x_{d+1})\in \mathbb{S}^{d}\mid x_{d+1}>0\}$.
 \end{lemma}
\begin{proof}
 For any different $n$ points $\bm{x}_{1}$, $\ldots$, $\bm{x}_{n}\in \mathbb{S}^{d}_{+}$, the Gram matrix $(K(\bm{x}_{i},\bm{x}_{j}))_{1\leq i,j\leq n}$ has an explicit formula:
 \[
 (K(\bm{x}_{i},\bm{x}_{j}))_{1\leq i,j\leq n}=\sum_{k\geq 0}a_{k}M_{k},
 \]
 where $M_{k}=\left(\langle \bm{x}_{i},\bm{x}_{j}\rangle^{k}\right)_{1\leq i,j\leq n}$.
It is obvious that each $M_{k}$ is positive semi-definite. Since the diagonal elements of $M_{k}$ equal to 1
and $\langle \bm{x}_{i},\bm{x}_{j}\rangle<1$ for  $\bm{x}_{i}\neq \bm{x}_{j}\in \mathbb{S}^{d}_{+}$. By Gershgorin circle theorem, we know that for sufficiently large $k$, $M_{k}$ is positive definite. Since there are infinitely many positive $a_{k}$'s, we know that $(K(\bm{x}_{i},\bm{x}_{j}))_{1\leq i,j\leq n}$ is positive definite.
\end{proof}

\begin{proof}[Proof of Proposition \ref{PD}]
 
Let $\bm{z}\tran=(\bm{x}\tran,1)$. By the definition of the NTK, we have
\begin{align*}
        K(\bm{x},\bm{x}') &= (1+\left<\bm{x},\bm{x}'\right>) \kappa_0(\bm{z},\bm{z}') + \| \bm{z}\|_2 \| \bm{z}' \|_2 \kappa_1(\bm{z},\bm{z}')+ 1\\
    &=\kappa_{0}(\bm{z},\bm{z}') + \left<\bm{x},\bm{x}'\right> \kappa_0(\bm{z},\bm{z}') + \| \bm{z}\|_2 \| \bm{z}' \|_2 \kappa_1(\bm{z},\bm{z}')+ 1 
\end{align*}
where $\kappa_{n}(\bm{z},\bm{z}^\prime):=2\Expc_{\omega\sim \mathcal{N}(0,\bm{I})}[\sigma'(\langle\omega, \bm{z}\rangle)\sigma'(\langle\omega, \bm{z}^\prime\rangle)(\langle\omega, \bm{z}\rangle)^{n}(\langle\omega, \bm{z}^\prime\rangle)^{n}]$ is the arc-cosine kernels of degree $n$ \cite{cho2009kernel} and $\sigma(x):=\max\{x,0\}$. 

The arc-cosine kernels $\kappa_0$ and $\kappa_{1}$ of degree $0$ and $1$, have the explicitly form (see e.g., \cite{cho2009kernel}):
\begin{align*}
    \kappa_0(\bm{z},\bm{z}') & = \frac{1}{\pi}(\pi-\psi(\bm{z},\bm{z}'))\\
    \kappa_1(\bm{z},\bm{z}') &= \frac{1}{\pi}\left ( \langle \frac{\bm{z}}{\lVert \bm{z} \rVert_2},\frac{\bm{z}'}{\lVert \bm{z}' \rVert_2} \rangle (\pi-\psi(\bm{z},\bm{z}')) + \sin (\psi(\bm{z},\bm{z}')) \right),
\end{align*}
where $\psi(\bm{z},\bm{z}') = \arccos  \left(\langle \frac{\bm{z}}{\lVert \bm{z} \rVert_2},\frac{\bm{z}'}{\lVert \bm{z}' \rVert_2} \rangle \right)$. We can see that $\kappa_0(\bm{z},\bm{z}')$ and $\kappa_1(\bm{z},\bm{z}')$ can be considered inner product kernels on $\mathbb{S}^{d}_{+}$, i.e., $\kappa_0(\bm{z},\bm{z}')=f_0(\langle\hat{\bm{z}},\hat{\bm{z}}'\rangle)$ and $\kappa_1(\bm{z},\bm{z}')=f_1(\langle\hat{\bm{z}},\hat{\bm{z}}'\rangle)$ for some functions $f_0$ and $f_1$ satisfying the conditions of Lemma \ref{lem:strict:positive} and $\hat{\bm{z}}=\frac{\bm{z}}{\lVert \bm{z} \rVert_2}$ and $\hat{\bm{z}'}=\frac{\bm{z}'}{\lVert \bm{z}' \rVert_2}$ defined on $\mathbb{S}^{d}_{+}$. By Lemma \ref{lem:strict:positive}, $\kappa_{0}(\bm{z},\bm{z}')$ and $\kappa_{1}(\bm{z},\bm{z}')$ are positive definite, meaning that $K(\bm{x},\bm{x}')$ is also positive definite.
\end{proof}


In the following section, we consider the spectral properties of the NTK over data of dimensional one assuming that $\mu_{\mathcal{X}}$ is the uniform distribution on $[0,1]$. Through Equation \eqref{eq: NTK_formular} with $d=1$, the NTK $K(x,x^{\prime})$ can be presented as followed:
\begin{align}\label{NTK:d=1:explicit}
    K(x,x^{\prime}) = \frac{2}{\pi}(\pi-\psi(x,x^{\prime}))(1+xx^{\prime}) + \frac{1}{\pi}|x-x^{\prime}| +1.
\end{align}
where $\psi(x,x^{\prime}) = \arccos \frac{1+xx^{\prime}}{\sqrt{(1+x^2)(1+(x^{\prime})^2)}}$. Define $\Pi_0$ and $\Pi_{1}$ as
\begin{align}
    \Pi_0(x,x^{\prime}) &= \frac{1}{\pi}(\pi-\psi(x,x^{\prime}))\\
    \Pi_1(x,x^{\prime}) &= \frac{1}{\pi}\left ( (1+xx^{\prime})(\pi-\psi(x,x^{\prime})) + |x-x^{\prime}| \right).
\end{align}
The following lemma shows the positive definiteness of $\Pi_0$ and $\Pi_{1}$:
\begin{lemma}\label{lem: Pi_positive_definite}
    $\Pi_0(x,x')$ and $\Pi_{1}(x,x')$ are positive definite on $[0,1]$.
\end{lemma}
\begin{proof}[Proof of Lemma \ref{lem: Pi_positive_definite}]
    Suppose that $\bm{X}=\{x_{1},\dots,x_{n}\} \subseteq [0,1]$ and $\bm{Z}=\{\bm{z}_{1},\cdots,\bm{z}_{n}\}$ where $\bm{z}_i=(x_i,1)$. Denote $D_{z} = \operatorname{diag} \{\|\bm{z}_i\|_2 \}_{i\in[n]}$ where $\|\bm{z}_i\|_2\geq 1$. $\Pi_0(\bm{X},\bm{X}) = \kappa_0(\bm{Z},\bm{Z})$ and $\Pi_1(\bm{X},\bm{X}) = D_{z}\kappa_1(\bm{Z},\bm{Z}) D_{z}$. Since $\kappa_0(\bm{Z},\bm{Z})$, $\kappa_1(\bm{Z},\bm{Z})$ are positive definite and $D_{z}$ is invertible, $\Pi_0(\bm{X},\bm{X})$ and $\Pi_1(\bm{X},\bm{X})$ are  positive definite.
\end{proof}

\begin{proof}[Proof of Theorem \ref{thm:spectral:d=1:L=1} $i)$]
Suppose that  $\bm{X}=\{x_{1},...,x_{n}\}\subseteq  [0,\pi]$. Let $G_{\alpha}(x,x')=\alpha-\frac{|x-x'|}{\pi}$ where $\alpha\geq 1$. Since $\bm{X}$ is one-dimensional data, we have the following lemmas.
\begin{lemma}\label{lem:min:eigenvalue} Let $d_{\min}=\min\{|x_{i}-x_{j}|\}$. We then have
\begin{align}
      \frac{d_{\min}}{2 \pi } \leq  \lambda_{\min}(G_{\alpha}(\bm{X},\bm{X}))\leq \frac{2d_{\min}}{ \pi }
\end{align}

\end{lemma}
\begin{lemma}\label{lem:G_leq_K_leq_7G}  
Suppose that $A$ and $B$ are two symmetric matrices. We use that notation $A \geq B$ if $A-B$ is a positive semi-definite matrix. Then we have
\begin{equation}
    G_{1}(\bm{X},\bm{X}) \leq K(\bm{X},\bm{X})\leq 7G_{9/7}(\bm{X},\bm{X}). 
\end{equation}
 
\end{lemma}

Its clear that Lemma \ref{lem:G_leq_K_leq_7G} implies that
\begin{align}
    \lambda_{\min}(G_{1}(\bm{X},\bm{X})) \leq \lambda_{\min}(K(\bm{X},\bm{X})) \leq 7 \lambda_{\min} (G_{9/7}(\bm{X},\bm{X}))
\end{align}
and Lemma \ref{lem:min:eigenvalue} implies Theorem \ref{thm:spectral:d=1:L=1} $i)$.

\end{proof}

\begin{proof}[Proof of Lemma \ref{lem:min:eigenvalue}]
    
Note that $G_{\alpha}^{-1}(\bm{X},\bm{X}) =$

\begin{equation}
    \resizebox{0.9\hsize}{!}{$
    \frac{\pi}{2} \begin{pmatrix}
		 \frac{1}{x_2-x_1} + \frac{1}{2\alpha\pi-x_n+x_1} & -\frac{1}{x_2-x_1}& 0            & \dots & 0           & \frac{1}{2\alpha\pi-x_n+x_1}        \\
		 -\frac{1}{x_2-x_1}    & \frac{1}{x_2-x_1} + \frac{1}{x_3-x_2}           & -\frac{1}{x_3-x_2} & \dots & 0           & 0          \\
		 0                & \ddots      & \ddots       & \ddots&             & \vdots       \\
		 \vdots           &             &              &       &             &             \\
		 0                &             &              &       &\frac{1}{x_{n-1}-x_{n-2}} + \frac{1}{x_{n}-x_{n-1}}              & -\frac{1}{x_{n}-x_{n-1}}  \\
		  \frac{1}{2\alpha\pi-x_n+x_1}              & 0           &  \dots       &  0    &-\frac{1}{x_{n}-x_{n-1}}  &\frac{1}{x_{n}-x_{n-1}}  + \frac{1}{2\alpha\pi-x_n+x_1}
		\end{pmatrix}.
    $}
\end{equation}

By Gershgorin circle theorem, every eigenvalue of $G_{\alpha}^{-1} = G_{\alpha}^{-1}(\bm{X},\bm{X})$ lies in one of the Gershgorin discs $D_i = \left\{\lambda~\big\vert~|\lambda-(G_{\alpha}^{-1})_{i,i}| \leq \sum_{j\neq i} |(G_{\alpha}^{-1})_{i,j}|\right\}$. In particular, we have  
\begin{equation}
    \lambda_{\max}(G_{\alpha}^{-1}) \leq \max_{i\in [n]} \max_{\lambda} D_i \leq \max_{i\in [n]} \left\{\sum_{j\in [n]} |G^{-1}_{i,j}|\right\}\leq \frac{2\pi}{d_{\min}},
\end{equation}
which means $\lambda_{\min}(G_{\alpha})\geq \frac{d_{\min}}{2\pi}$.

One the other hand, assume that $x_{k+1}-x_k=d_{\min}$ for some $k$. Since $\lambda_{\max}(G_{\alpha}^{-1})\geq u\tran G_{\alpha}^{-1}u$ for $\forall u$ with $\|u\|_{2}=1$, let $u$ be the vector that only has 1 in the $k$-th entry and the rest are zero. Thus, we have $\lambda_{\max}(G_{\alpha}^{-1})\geq \frac{\pi}{2d_{\min}}$, which means $\lambda_{\min}(G_{\alpha})\leq \frac{2d_{\min}}{\pi}$. To sum up, we have
\begin{equation}\label{eq:min_eigen_G}
    \frac{d_{\min}}{2\pi} \leq \lambda_{\min}(G_{\alpha})\leq \frac{2d_{\min}}{\pi}.
\end{equation}
\end{proof}

\begin{corollary}\label{cor:distance:positive}
The kernel function $G_{\alpha}$ is positive definite on $[0,\pi]$.
\end{corollary}

\begin{proof}[Proof of lemma \ref{lem:G_leq_K_leq_7G}] We can easily verify the following equation from  the equation \eqref{NTK:d=1:explicit} and the definition of $G_{\alpha}$: 
\begin{equation}\label{eqn:K:decompostion}
\begin{aligned}
       K(x,x^{\prime})&= G_{1}(x,x^{\prime}) + 2\Pi_1(x,x^{\prime})\\
    &=2\Pi_0(x,x^{\prime})(1+xx') - G_{1}(x,x') +2.
\end{aligned}
\end{equation}
It is clear that $K(\bm{X},\bm{X})\geq G_{1}(\bm{X},\bm{X})$ from the first line in \eqref{eqn:K:decompostion}. 

On the other hand, let $z(x)=2x-\psi(0,x)$. We can easily verify that
\begin{align*}
    z(0)=0; \quad  z'(x) = 2-\frac{\psi(0,x)}{\partial x} = 2-\frac{1}{1+x^2}>0, \forall x\in [0,1]; \quad z(1)=2-\frac{\pi}{4} \leq \pi.
\end{align*}
Thus $z(x) \in[0,\pi]$. Let $\bm{Z}=\{z_1,...,z_n\}$, where $z_i=2x_i-\psi(0,x_i)$. Each entry of $G(\bm{Z},\bm{Z})$ is given by
\begin{equation*}
    G(z_i,z_j)=1-\frac{|2x_i-2x_j-\psi(0,x_i)+\psi(0,x_j)|}{\pi}\\
    =2(1-\frac{|x_i-x_j|}{\pi})-(1-\frac{\psi(x_i,x_j)}{\pi}).
\end{equation*}
Thus $G_{1}(\bm{Z},\bm{Z})=2G_{1}(\bm{X},\bm{X})-\Pi_{0}(\bm{X},\bm{X})$.
Since $G_{1}(\bm{Z},\bm{Z})$ is positive definite by Corollary \ref{cor:distance:positive}, we know $\Pi_0(\bm{X},\bm{X}) < 2G_{1}(\bm{X},\bm{X})$. 

\vspace{3mm}

Let $D_{X}=\operatorname{diag}\{ (x_i)_{i \in [n]}\}$ and $1_{n}=[1,\dots,1]\tran$.
Note that $D_{X}\Pi_0(\bm{X},\bm{X}) D_{X} \leq  \Pi_0(\bm{X},\bm{X})$ since  $x_{i}\in[0,1]$. 
 Thus, from the second line in \eqref{eqn:K:decompostion} we have
\begin{align*}
	K(\bm{X},\bm{X}) &=  2\Pi_0(\bm{X},\bm{X}) + 2D_{X} \Pi_0(\bm{X},\bm{X}) D_{X} - G_{1}(\bm{X},\bm{X}) + 2 1_{n} 1_{n}\tran\\
    &\leq 4\Pi_0(\bm{X},\bm{X}) - G_{1}(\bm{X},\bm{X}) + 21_{n} 1_{n}\tran  \leq 7 G_{1}(\bm{X},\bm{X}) + 2 1_{n} 1_{n}\tran = 7G_{\frac{9}{7}}(\bm{X},\bm{X}).
\end{align*}
\end{proof}

\begin{proof}[Proof of Theorem \ref{thm:spectral:d=1:L=1} $ii)$]
Theorem \ref{thm:spectral:d=1:L=1} $ii)$ is a direct corollary of the following lemmas.

\begin{lemma}\label{lemma:K_bound_by_G}
    Let $\{\lambda^{(\alpha)}_{j}, j \geq 1\}$ and $\{\lambda^K_{j}, j \geq 1\}$ be the eigenvalues associated to the kernel $G_{\alpha}$ and $K$ on $[0,1]$, respectively. Then we have 
\begin{equation}\label{equation:G_K_G_alpha}
	   \lambda^{(1)}_{j} \leq \lambda^K_{j} \leq 7\lambda^{(9/7)}_{j} ,\quad j \geq 1.
\end{equation}
\end{lemma}

\begin{lemma}\label{lemma:G_decay_rate} Suppose that $\alpha=\{1,\frac{9}{7}\}$. There exist constants $c$ and $C$ such that
\begin{equation}
     \frac{c}{j^{2}}\leq \lambda^{(\alpha)}_j \leq  \frac{C}{j^{2}}, j \geq 1.
\end{equation}
\end{lemma}
\end{proof}

\begin{proof}[Proof of lemma \ref{lemma:K_bound_by_G}]
It is a direct corollary of the following lemma.
\begin{lemma}[Corollary of Theorem 3.1 of \cite{koltchinskii2000random}]\label{lemma:sample_eigen_to_total_eigen}
    Let $A$ be a kernel function on $\mathcal{X}\times \mathcal{X}$ with $\int_{\mathcal{X}} \int_{\mathcal{X}} A(x,x')^2 \dx x \dx x^{\prime}< \infty$. For $\bm{X}=\{x_1,\dots, x_n\}\subseteq \mathcal{X}$, let $\hat{\lambda}_j$ be the eigenvalue of $A(\bm{X},\bm{X})$ and $\lambda_j$ be the eigenvalue of the kernel $A$. Then, for any fixed $j$, we have
    \begin{equation}
        \vert \hat{\lambda}_j -\lambda_j\vert \to 0, \mbox{as $n\to \infty$}.
    \end{equation}
\end{lemma}

In fact, by Lemma \ref{lem:G_leq_K_leq_7G}, we have 
$
    G_{1}(\bm{X},\bm{X}) \leq K(\bm{X},\bm{X})\leq 7G_{\frac{9}{7}}(\bm{X},\bm{X}).
$
Thus, we have 
$   \hat{\lambda}^{(1)}_{j} \leq \hat{\lambda}^K_{j} \leq 7\hat{\lambda}^{(9/7)}_{j}, j=1,2,...,n.$
Then Lemma \ref{lemma:sample_eigen_to_total_eigen} provides us that for any fixed $j \geq 1$,  
\begin{equation}
	   \lambda^{(1)}_{j} \leq \lambda^K_{j} \leq 7\lambda^{(9/7)}_{j}.
\end{equation}
\end{proof}

\begin{proof}[Proof of lemma \ref{lemma:G_decay_rate}]
    
Since the $G_{\alpha}$ ($\alpha=1, 9/7$) is a positive definite kernel on $[0,1]$, we know that $\lambda_{j}\geq 0$ for any $j=1,2,\cdots$.
Let $\lambda \neq 0$ be an eigenvalue of $G_{\alpha}$, i.e., there is an eigenfunction $f(x)$ such that  

\begin{equation}\label{equation:origin_function}
\begin{split}
     \lambda f(x) &= \left(T_{G_{\alpha}}f\right)(x)= \int_{0}^{1} \left(\alpha-\frac{|x-x'|}{\pi}\right) f(x') \dx x'\\ 
    &=\int_{0}^{1}\alpha f(x')\dx x'-\frac{1}{\pi}\left(\int_{0}^{x}(x-x')f(x')\dx x'+\int_{x}^{1}(x'-x)f(x')\dx x'\right).
\end{split}
\end{equation}

After taking the first and second derivatives on both sides with respect to $x$, we get
\begin{equation}
    \lambda f'(x)=-\frac{1}{\pi}\left(\int_{0}^{x}f(x')\dx x' - \int_{x}^{1}f(x')\dx x'\right),
\end{equation}
and
\begin{equation}\label{equation:second_derivative}
\lambda f''(x)=-\frac{2}{\pi} f(x).
\end{equation}

It is well known that the solutions of \eqref{equation:second_derivative} are of the following forms:
\begin{equation}\label{equation:eigenfunction}
    f(x)=A\cos (\omega x)+B\sin(\omega x).
\end{equation}
Inserting equation \eqref{equation:eigenfunction} back to equation  \eqref{equation:second_derivative}, we know that
\begin{equation}\label{equation:oemga_lambda}
    \omega^2=\frac{2}{\pi\lambda}>0.
\end{equation}
Inserting Equation \eqref{equation:eigenfunction} and \eqref{equation:oemga_lambda} in Equation \eqref{equation:origin_function}, we have
\begin{equation}
\begin{split}
    \frac{2}{\pi \omega^2} f(x) &= -\frac{1}{\pi}x\omega(B+B\cos(\omega)-A\sin(\omega)) +\frac{2}{\pi \omega^2} f(x)\\
    &+ \frac{1}{\pi}A(\alpha\pi\omega\sin(\omega)-1-\omega\sin(\omega) - \cos(\omega)) \\
    &+ \frac{1}{\pi}B(\alpha\pi\omega(1-\cos(\omega)) + \omega\cos(\omega)-\sin(\omega)),
\end{split}
\end{equation}
which holds for all $x$ if and only if
\begin{equation*}
 \begin{cases}
 -A\sin(\omega)+B(1+\cos(\omega))=0,\\
   A(\alpha\pi\omega\sin(\omega)-1-\omega\sin(\omega) - \cos(\omega))+B(\alpha\pi\omega(1-\cos(\omega)) + \omega\cos(\omega)-\sin(\omega))=0.\\
	\end{cases}
\end{equation*}
A necessary and sufficient condition for this system to be degenerate (i.e., it has a nontrivial solution $A$ and $B$) is 
\begin{equation}
 \det \begin{pmatrix}
	-\sin(\omega) & 1+\cos(\omega)\\
	 \alpha\pi\omega\sin(\omega)-1-\omega\sin(\omega) - \cos(\omega)& \alpha\pi\omega(1-\cos(\omega)) + \omega\cos(\omega)-\sin(\omega)
	\end{pmatrix}=0,
\end{equation}
i.e.,
\begin{equation}\label{equation:omega}
	2+2\cos(\omega)+ \omega \sin(\omega)(1-2\alpha\pi)=0.
\end{equation}
In fact, denote the left-hand side of the equation \eqref{equation:omega} by $h(\omega)$, i.e., $h(\omega)=2+2\cos(\omega)+\omega \sin(\omega)(1-2\alpha\pi)$. Since  $h(\omega)=h(-\omega)$, we only need to prove the assertion \eqref{eqn:solution:omega} for $\omega > 0$. By Lemma \ref{lemma:h_equal_zero} and Equation \eqref{equation:oemga_lambda}, we have

\begin{equation}
     \begin{cases}
     \lambda_j \in [\frac{8}{\pi^3},\frac{72}{\pi^3}], & \mbox{$j=1$};\\
    \lambda_j= \frac{2}{\pi^3}(j-1)^{-2}, & \mbox{$j$ is even}; \\
    \lambda_j \in [\frac{2}{\pi^3}(j-\frac{1}{2})^{-2},\frac{2}{\pi^3}(j-1)^{-2}], &  \mbox{$j>1$ and $j$ is odd}.
    \end{cases}
\end{equation}
To sum up,
\begin{equation}
   \frac{c}{j^{2}}\leq \lambda_j \leq  \frac{C}{j^{2}}, j \geq 1
\end{equation}
for some absolute constants $c$ and $C$.
\end{proof}

\begin{proof}[Proof of lemma \ref{lemma:sample_eigen_to_total_eigen}]
 By Theorem 3.1 of \cite{koltchinskii2000random}, we have
\begin{equation}
    \sqrt{\sum_{j=1}^{\infty} (\hat{\lambda}_j-\lambda_j)^2}\to 0, \mbox{as $n\to \infty$}, 
\end{equation}
where $\hat{\lambda}_{j} =0$ for $j\geq n$. This implies that for any fixed $j \geq 1$,
\begin{equation}
    \vert \hat{\lambda}_j - \lambda_j \vert \leq \sqrt{\sum_{l=1}^{\infty} (\hat{\lambda}_l-\lambda_l)^2} \to 0, \mbox{as $n\to \infty$}.
\end{equation}
\end{proof}

\begin{lemma}\label{lemma:h_equal_zero}
    Let $h(\omega)=2+2\cos(\omega) + \omega \sin(\omega)(1-2\alpha \pi)$, where $\alpha \in \{1, \frac{9}{7}\}$ Then the solutions of 
	\begin{equation}\label{equation:h_equal_zero}
		h(\omega)=0, \quad \omega>0
	\end{equation}
	 are given by

\begin{equation}
    \begin{cases}
    \omega_{j} \in [\frac{1}{6}\pi, \frac{1}{2}\pi], & \mbox{$j=1$};\\
    \omega_{j} = (j-1)\pi, & \mbox{$j$ is even};\\
    \omega_{j} \in ((j-1)\pi, (j-\frac{1}{2})\pi), & \mbox{$j>1$ and $j$ is odd}.
    \end{cases}
\end{equation}

\end{lemma}

\begin{proof}
	When $\omega \in (0,\pi)$, we can easily verify the following facts
\begin{itemize}
    \item[ ~~~~~(1).] $h(\frac{\pi}{6})=2+\sqrt{3} - (2\alpha\pi-1)\frac{\pi}{12}>0$;
    
    \item[ ~~~~~(2).] If  $\omega \in ( 0, \frac{1}{2}\pi]$, we have 
    \begin{equation}
    h'(\omega) = -2\sin(\omega) + (1-2\alpha \pi)\sin(\omega) + (1-2\alpha \pi)\omega \cos(\omega)<0.
    \end{equation}

    \item[ ~~~~~(3).] If $\omega \in [\frac{1}{2}\pi,\pi)$, then $$h(\omega)<2+2\cos(\omega) - 2\sin(\omega) = 2+2\sqrt{2}\cos(\omega+\frac{\pi}{4}) \leq 0. $$  
\end{itemize} 
Since $h(\omega)$ is a continuous function on $(0,\pi)$, the above facts imply that $h(\omega)=0$ has a unique solution $[\frac{1}{6}\pi,\frac{1}{2}\pi]$ which is denoted by $\omega_{1}$.\\

When $\omega \geq \pi$, it is clear that for any even number $j$, $j\pi$ is a solution of the equation \eqref{equation:h_equal_zero} with multiplicity one. We will show that there is another unique solution of the equation \eqref{equation:h_equal_zero} in the interval $(j\pi,(j+2)\pi)$ where $j$ is an even integer. Thus, the solutions of the equation can be indexed by $\mathbb{N}_{+}$ in the following way
\begin{align}\label{eqn:solution:omega}
    \omega_{2k}=(2k-1)\pi \mbox{ and } \omega_{2k+1}\in ((2k-1)\pi, (2k+1)\pi), \quad k=1,2,\cdots.
\end{align}

When  $\omega\in ((2k-1)\pi,(2k+1)\pi), k\in \mathbb{N}_{+}$, we can easily verify the following facts.
\begin{itemize}
    \item[ ~~~~~(1).] If $\omega \in ((2k-1)\pi,2k\pi))$, then  $h(\omega)>\omega \sin(\omega)(1-2\alpha \pi)>0$;
    
    \item[ ~~~~~(2).] If  $\omega \in [ 2k\pi, (2k+\frac{1}{2})\pi]$, we have 
    \begin{equation}
    h'(\omega) = -2\sin(\omega) + (1-2\alpha \pi)\sin(\omega) + (1-2\alpha \pi)\omega \cos(\omega) <0.
    \end{equation}

    \item[ ~~~~~(3).] If $\omega \in ((2k+\frac{1}{2})\pi,(2k+1)\pi)$, then $$h(\omega)<2+2\cos(\omega) - 2\sin(\omega) = 2+2\sqrt{2}\cos(\omega+\frac{\pi}{4}) < 0. $$  
\end{itemize}
Since $h(\omega)$ is a continuous function on $((2k-1)\pi,(2k+1)\pi)$, the above facts imply that $h(\omega)=0$ has a unique solution $\in [2k\pi, (2k+\frac{1}{2})\pi]$, which is denoted by $\omega_{2k+1}$. Thus, we have

\begin{equation}
    \begin{cases}
    \omega_{j} \in [\frac{1}{6}\pi, \frac{1}{2}\pi], & \mbox{$j=1$};\\
    \omega_{j} = (j-1)\pi, & \mbox{$j$ is even};\\
    \omega_{j} \in ((j-1)\pi, (j-\frac{1}{2})\pi), & \mbox{$j>1$ and $j$ is odd}.
    \end{cases}
\end{equation}
\end{proof}

\end{appendix}

%% file: supp.tex
\newpage

\supplement

\title{Supplement to ``Generalization Ability of Wide Neural Networks on $\mathbb{R}$''}

\section{Proof of Section \ref{sec:main_result}}\label{app:ntk:approx}
We first prove Proposition \ref{prop:kernel:approx} and then prove Proposition \ref{prop:funct:approx}. For brevity, denote the pre-activation value and the activation pattern for the $r$-th neuron of the hidden layer of the neural network with parameters $\bm{\theta}$ by $h_{\bm{\theta},r}(\bm{x})=\langle \bm{w}_{r},\bm{x}\rangle+b_{r}$ and $\bm{1}_{\bm{\theta},r}(\bm{x})=\bm{1}_{\{h_{\bm{\theta},r}(\bm{x})\geq 0\}}$ respectively. For simplicity, we consider the neural network to have $2m$ neurons.

\subsection{Proof of Proposition \ref{prop:kernel:approx}}\label{app:kernel:approx}
We defer the proof to the end of 
Section \ref{app:kernel:approx}. To start with,
Lemma \ref{lem: event B}, Lemma \ref{lem: event R} and Lemma \ref{lem: event C} are the building blocks to prove Proposition \ref{prop:kernel:approx}, since in these lemmas we will show the events we need to condition on hold with probability converging to one as $m\to\infty$.

Lemma \ref{lem: event B} controls the scale of the parameters of the neural network at initialization.

\begin{lemma}\label{lem: event B}
    Define the event 
    \begin{equation*}
        \mathcal{B}=\left\{\omega\mid|a_{r}(0)|,|\bm{w}_{r,(j)}(0)|,|b_{r}(0)|\leq R_{B}, r\in[2m], j\in[d]\right\}, \text{~where~} R_{B}=\sqrt{3\log m}.
    \end{equation*}
    Conditioning on the event $\mathcal{B}$, we have $|h_{\bm{\theta}(0),r}(\bm{x})|\leq (dB+1)R_{B}$ for all $r\in [2m]$ and $\bm{x}\in\mathcal{X}$. The event $\mathcal{B}$ holds with high probability, i.e., $\Prob_{\bm{\theta}(0)}(\mathcal{B})\geq 1-P_{\mathcal{B}}(m)$, where $P_{\mathcal{B}}(m)=\frac{2(d+2)}{\sqrt{2\pi}}m^{-1/2}$.
\end{lemma}

\begin{proof}
    Under our special initialization setting where $a_{r}(0)=-a_{r+m}(0)$, $\bm{w}_{r,(j)}(0)=\bm{w}_{r+m,(j)}(0)$, $b_{r}(0)=b_{r+m}(0)\sim \mathcal{N}(0,1)$ for $r\in[m]$, the total number of the elements in $\mathcal{B}$ that need to be controlled is $(d+2)m$.
    For $Z\sim\mathcal{N}(0,1)$, classical Gaussian tail bound gives 
    \begin{equation*}
        \Prob(|Z|\geq R_{B})\leq\frac{2e^{-R_{B}^{2}/2}}{\sqrt{2\pi}R_{B}}\leq\frac{2e^{-R_{B}^{2}/2}}{\sqrt{2\pi}}=\frac{2}{\sqrt{2\pi}}m^{-3/2}.
    \end{equation*}
    Then $\Prob_{\bm{\theta}(0)}(\mathcal{B})\geq 1-\frac{2(d+2)}{\sqrt{2\pi}}m^{-1/2}$ by the union bound. Conditioning on $\mathcal{B}$, we have
    \begin{equation*}
        |h_{\bm{\theta}(0),r}(\bm{x})|=| \langle \bm{w}_{r}(0),\bm{x}\rangle+b_{r}(0)|\leq\| \bm{w}_{r}(0)\|_{2}\| \bm{x}\|_{2}+|b_{r}(0)|\leq (dB+1)R_{B}.
    \end{equation*}
\end{proof}

Our main contribution is the uniform convergence of kernel which relies on the analysis of the continuity of $K_{\bm{\theta}(0)}^{m}$ and $K_{d}$ and a method that is similar to the epsilon-net argument. On each dimension, we place $\lfloor m^{\beta}\rfloor$ points with distance 
\begin{equation*}
    \epsilon=2B/\lfloor m^{\beta}\rfloor
\end{equation*}
in $[-B,B]$ for some $\beta\in(0,1]$. Denote the collection $\mathcal{N}_{\epsilon}$ so that $|\mathcal{N}_{\epsilon}|=\lfloor m^{\beta}\rfloor^{d}$. The idea is to use $\mathcal{N}_{\epsilon}$ to discretize the domain $\mathcal{X}$ and then use classical concentration inequality on points in $\mathcal{N}_{\epsilon}$, which makes the probability of the complement of the events decaying exponentially fast with $m$. Then with the continuity of $K_{\bm{\theta}(0)}^{m}$ and $K_{d}$, the events hold over $\mathcal{X}$ with high probability.

Lemma \ref{lem: event R} shows the pre-activation values of most neurons are large, which hints that the activation pattern for these neurons is likely to stay unchanged during training since a large pre-activation value requires the parameters to travel a long way from the initialization to change the sign. This is crucial to prove that the training wide neural networks fall into the lazy regime where the parameters stay close to the initialization during training.

\begin{lemma}\label{lem: event R}
    Define the events 
    \begin{equation*}
        \mathcal{R}(\mathcal{N}_{\epsilon})=\left\{\omega ~\middle|~|h_{\bm{\theta}(0),r}(\bm{z})|\leq 2(dB+1)R\mbox{~holds for at most~}2\lfloor m^{\gamma}\rfloor\mbox{~of~}r\in [2m],\forall\bm{z}\in\mathcal{N}_{\epsilon}\right\}
    \end{equation*}
    and
    \begin{equation*}
        \mathcal{R}=\left\{\omega ~\middle|~|h_{\bm{\theta}(0),r}(\bm{x})|\leq (dB+1)R\mbox{~holds for at most~}2\lfloor m^{\gamma}\rfloor\mbox{~of~} r\in [2m],\forall\bm{x}\in [-B,B]^{d}\right\},
    \end{equation*}
    where $R=\frac{\sqrt{2\pi}}{4(dB+1)}m^{-\alpha}$ for some $\alpha\in(0,\beta)$ and $\gamma>\max\{1-\alpha,\delta\}$ with $\delta>1/2$. If $m$ is sufficiently large, then $\mathcal{R}\supseteq\mathcal{R}(\mathcal{N}_{\epsilon})$ and the event $\mathcal{R}$ holds with high probability, i.e., $\Prob_{\bm{\theta}(0)}\left(\mathcal{R}\right)\geq \Prob_{\bm{\theta}(0)}\left(\mathcal{R}(\mathcal{N}_{\epsilon})\right)\geq 1-P_{\mathcal{R}}(m)$, where $P_{\mathcal{R}}(m)=m^{d\beta}e^{-2m^{2\delta-1}}$.
\end{lemma}

\begin{proof}
    Due to our special initialization setting, we only need to consider $r\in[m]$ since $|h_{\bm{\theta}(0),r+m}(\bm{z})|=|h_{\bm{\theta}(0),r}(\bm{z})|$ for $r\in[m]$. For every $\bm{z}\in\mathcal{N}_{\epsilon}$, let $T_{r}=\bm{1}_{\{|h_{\bm{\theta}(0),r}(\bm{z})|\leq 2(dB+1)R\}}$ with mean 
    \begin{equation*}
        p=\Expc_{\bm{\theta}(0)}T_{r}=\Prob_{\bm{\theta}(0)}(|h_{\bm{\theta}(0),r}(\bm{z})|\leq 2(dB+1)R)\leq\frac{2}{\sqrt{2\pi}}\frac{2(dB+1)R}{\sqrt{\|\bm{z}\|_{2}^{2}+1}}\leq m^{-\alpha},
    \end{equation*}
    where the second inequality holds due to $h_{\bm{\theta}(0),r}(\bm{z})\sim\mathcal{N}(0,\lVert\bm{z}\rVert_{2}^{2}+1)$ and the density function of a standard Gaussian is upper bounded by $1/\sqrt{2\pi}$.
By Hoeffeding's inequality (see Theorem 2.8 in \cite{boucheron2013concentration}), for all $\delta>0$, we have $\Prob_{\bm{\theta}(0)}\left( \sum_{r\in[m]}T_{r} \geq  mp+m^{\delta}\right)\leq e^{-2m^{2\delta-1}}$. 
Now we have 
\begin{equation*}
    \begin{aligned}
    & \Prob_{\bm{\theta}(0)}\left( |h_{\bm{\theta}(0),r}(\bm{z})\leq 2(dB+1)R\mbox{~holds for at most~}\lfloor m^{\gamma}\rfloor\mbox{~of~} r \in[m]\right)\\
    = & \Prob_{\bm{\theta}(0)}\left(\sum_{r\in[m]}T_{r}\leq \lfloor m^{\gamma}\rfloor\right)=1-\Prob_{\bm{\theta}(0)}\left(\sum_{r\in[m]}T_{r}>\lfloor m^{\gamma}\rfloor\right)\\
    \geq& 1-\Prob_{\bm{\theta}(0)}\left(\sum_{r\in[m]}T_{r}\geq mp+m^{\delta}\right)\geq 1-e^{-2m^{2\delta-1}},
    \end{aligned}
\end{equation*}
where the first inequality holds when $m$ is large enough such that $mp+m^{\delta}\leq m^{1-\alpha}+m^{\delta}\leq \lfloor m^{\gamma}\rfloor$. Hence we have $\Prob_{\bm{\theta}(0)}(\mathcal{R}(\mathcal{N}_{\epsilon}))\geq 1-\lvert \mathcal{N}_{\epsilon}\rvert e^{-2m^{2\delta-1}}$ simply by the union bound. For every $\bm{x}$, we choose $\bm{z}\in\mathcal{N}_{\epsilon}$ such that $\|\bm{x}-\bm{z}\|_{2}\leq \sqrt{d}\epsilon$, so 
\begin{equation*}
    \begin{aligned}
    \abs{h_{\bm{\theta}(0),r}(\bm{z})}&=\abs{h_{\bm{\theta}(0),r}(\bm{x})+\langle\bm{w}_{r}(0),\bm{z}-\bm{x}\rangle}\\
    &\leq \abs{h_{\bm{\theta}(0),r}(\bm{x})}+\|\bm{w}_{r}(0)\|_{2}\|\bm{z}-\bm{x}\|_{2}\leq \abs{h_{\bm{\theta}(0),r}(\bm{x})}+dR_{B}\epsilon.
    \end{aligned}
\end{equation*}
Thus $|h_{\bm{\theta}(0),r}(\bm{x})|\geq|h_{\bm{\theta}(0),r}(\bm{z})|-dR_{B}\epsilon>(dB+1)R$, where the last inequality holds when $m$ is large enough such that $dR_{B}\epsilon<(dB+1)R$. 
\end{proof}

It is intuitive that the point-wise convergence of $K_{\bm{\theta}(0)}^{m}-K_{d}$ holds simply by the law of large numbers. The result from Lemma \ref{lem: event C} shows this convergence is uniform for points in the collection $\mathcal{N}_{\epsilon}$.

\begin{lemma}\label{lem: event C}
    Define the event 
    \begin{equation*}
        \mathcal{C}=\left\{\omega ~\middle|~ \sup_{\bm{z},\bm{z}'\in\mathcal{N}_{\epsilon}}| K_{\bm{\theta}(0)}^{m}(\bm{z},\bm{z}')-K_d(\bm{z},\bm{z}')|\leq C_{1}\sqrt{\frac{\log m}{m}}\right\},
    \end{equation*}
    where $C_{1}>0$ is a constant depending on $d,B,\beta$. If $m$ is sufficiently large, then the event $\mathcal{C}$ holds with high probability, i.e., $\Prob_{\bm{\theta}(0)}\left(\mathcal{C}\right)\geq 1-P_{\mathcal{C}}(m)$, where $P_{\mathcal{C}}(m)=4m^{-d\beta}$.
\end{lemma}

Before we give the proof of Lemma \ref{lem: event C}, we need to dive into details of the kernel of the neural network from here to analyze further, so we introduce more notations. Given the parameters $\bm{\theta}$ of the neural network, let $H_{\bm{\theta},r}(\bm{x},\bm{x}')=\left(\langle\bm{x},\bm{x}'\rangle+1 \right)a_{r}^{2}\bm{1}_{\bm{\theta},r}(\bm{x})\bm{1}_{\bm{\theta},r}(\bm{x}')$, $G_{\bm{\theta},r}(\bm{x},\bm{x}')=\sigma(h_{\bm{\theta},r}(\bm{x}))\sigma(h_{\bm{\theta},r}(\bm{x}'))$ be the contribution to the kernel from the $r$-th neuron at the first and second layer respectively. Then we decompose
\begin{equation*}
  K_{\bm{\theta}}^{m}(\bm{x},\bm{x}')=\langle\nabla_{\bm{\theta}}f_{\bm{\theta}}^{m}(\bm{x}),\nabla_{\bm{\theta}}f_{\bm{\theta}}^{m}(\bm{x}')\rangle=1+H_{\bm{\theta}}^{m}(\bm{x},\bm{x}')+G_{\bm{\theta}}^{m}(\bm{x},\bm{x}'),
\end{equation*} 
where 
\begin{equation*}
    H_{\bm{\theta}}^{m}(\bm{x},\bm{x}')=\frac{1}{m}\sum_{r\in[2m]}H_{\bm{\theta},r}(\bm{x},\bm{x}'),G_{\bm{\theta}}^{m}(\bm{x},\bm{x}')=\frac{1}{m}\sum_{r\in[2m]}G_{\bm{\theta},r}(\bm{x},\bm{x}').
\end{equation*} 
A similar decomposition for the NTK is 
\begin{equation*}
    K_{d}(\bm{x},\bm{x}')=1+H(\bm{x},\bm{x}')+G(\bm{x},\bm{x}'),
\end{equation*}
where $H(\bm{x},\bm{x}')=\Expc_{\bm{\theta}(0)}H_{\bm{\theta}(0)}^{m}(\bm{x},\bm{x}')$ and $G(\bm{x},\bm{x}')=\Expc_{\bm{\theta}(0)}G_{\bm{\theta}(0)}^{m}(\bm{x},\bm{x}')$. Thanks to the decomposition, we can simply analyze each part of the kernel and then use the triangle inequality to apply to the whole kernel. 

\begin{proof}[Proof of Lemma \ref{lem: event C}]
    Notice that $H_{\bm{\theta}(0),r}(\bm{z},\bm{z}')$ and $G_{\bm{\theta}(0),r}(\bm{z},\bm{z}')$ are both sub-exponential and their sub-exponential norm is bounded by a constant $c'$ depending on 
    $d,B$. Then by Bernstein's inequality(see Theorem 2.8.1 in \cite{vershynin2018highdimensional}), for every $c>0$, 
\begin{equation*}
    \begin{aligned}
    \Prob_{\bm{\theta}(0)}\left(|H_{\bm{\theta}(0)}^{m}(\bm{z},\bm{z}')-H(\bm{z},\bm{z}')|\geq c\sqrt{\frac{\log m}{m}}\right)\leq& 2e^{-c_{0}\min\{\frac{c^{2}}{c'^{2}}\log m,\frac{c}{c'}\sqrt{m\log m}\}}\\
    =&2m^{-c_{0}c^{2}/c'^{2}},
    \end{aligned}
\end{equation*}
where $c_{0}$ is an absolute constant and the equality holds when $m$ is large enough such that $\frac{c^{2}}{c'^{2}}\log m\leq \frac{c}{c'}\sqrt{m\log m}$.
Likewise, we have the same inequality for $G_{\bm{\theta}(0)}^{m}$, so that
\begin{equation*}
    \begin{aligned}
    \Prob_{\bm{\theta}(0)}\left(\sup_{\bm{z},\bm{z}'\in\mathcal{N}_{\epsilon}}\lvert K_{\bm{\theta}(0)}^{m}(\bm{z},\bm{z}')-K_{d}(\bm{z},\bm{z}')\rvert\leq 2c\sqrt{\frac{\log m}{m}}\right)\geq&1-4\binom{|\mathcal{N}_{\epsilon}|}{2} m^{-c_{0}c^{2}/c'^{2}}\\
    \geq&1-4m^{-(c_{0}c^{2}/c'^{2}-2d\beta)}\\
    =&1-4m^{-d\beta}
    \end{aligned}
\end{equation*}
simply by the triangle inequality and the union bound, where we set $c=\sqrt{3c'^{2}d\beta/c_{0}}$ in the last equality. 
\end{proof}

For initialization that lies in the intersection of the events, i.e., $\mathcal{B}\cap\mathcal{R}\cap{\mathcal{C}}$, Lemma \ref{lem: initial kernel close to fixed} and Lemma \ref{lem: kernel with lazy params close to initial kernel} shows how the width $m$ control the convergence of kernel at the initialization and during training. The proof of Lemma \ref{lem: initial kernel close to fixed} and Lemma \ref{lem: kernel with lazy params close to initial kernel} could be found in Section \ref{app:subsec:init:kernel:close:fixed} and Section \ref{app:subsec:lazy:kernel:close:init} respectively.

\begin{lemma}\label{lem: initial kernel close to fixed}
    Conditioning on the event $\mathcal{B}\cap\mathcal{R}\cap\mathcal{C}$, if we set $\gamma>1-\beta/4$ and $m$ is sufficiently large, then
    \begin{equation*}
        \sup_{\bm{x},\bm{x}'\in\mathcal{X}}|K_{\bm{\theta}(0)}^{m}(\bm{x},\bm{x}')-K_{d}(\bm{x},\bm{x}')|\leq C_{2}m^{-(1-\gamma)}\log m,
    \end{equation*}
    where $C_{2}>0$ is a constant depending on $d,B$. 
\end{lemma}

\begin{lemma}\label{lem: kernel with lazy params close to initial kernel}
     Conditioning on the event $\mathcal{B}\cap\mathcal{R}\cap\mathcal{C}$, if we set $\alpha<1/2$ and $m$ is sufficiently large, then
    \begin{equation*}
        \sup_{t\geq 0}\sup_{\bm{x},\bm{x}'\in\mathcal{X}}| K_{\bm{\theta}(t)}^{m}(\bm{x},\bm{x}')-K_{\bm{\theta}(0)}^{m}(\bm{x},\bm{x}')|\leq C_{3}m^{-(1-\gamma)}\log m,
    \end{equation*}      
    where $C_{3}>0$ is constant depending on $d,B$.
\end{lemma}

\begin{proof}[Proof of Proposition \ref{prop:kernel:approx}]
    Consider the initialization $\omega\in\mathcal{B}\cap\mathcal{R}\cap\mathcal{C}$. Then for all $\bm{x},\bm{x}'\in\mathcal{X}$, we have
    \begin{equation*}
    \begin{aligned}
    |K_{\bm{\theta}(t)}^{m}(\bm{x},\bm{x}')-K_{d}(\bm{x},\bm{x}')|&\leq|K_{\bm{\theta}(t)}^{m}(\bm{x},\bm{x}')-K_{\bm{\theta}(0)}^{m}(\bm{x},\bm{x}')|+|K_{\bm{\theta}(0)}^{m}(\bm{x},\bm{x}')-K_{d}(\bm{x},\bm{x}')|\\
    &\leq (C_{2}+C_{3})m^{-(1-\gamma)}\log m,
    \end{aligned}
\end{equation*}
    where the last inequality follows from Lemma \ref{lem: initial kernel close to fixed} and Lemma \ref{lem: kernel with lazy params close to initial kernel}. With Lemma \ref{lem: event B}, Lemma \ref{lem: event R} and Lemma \ref{lem: event C}, we show that
    \begin{equation*}
        \begin{aligned}
            &\Prob_{\bm{\theta}(0)}\left(\sup_{\bm{x},\bm{x}'\in\mathcal{X}}|K_{\bm{\theta}(t)}^{m}(\bm{x},\bm{x}')-K_{d}(\bm{x},\bm{x}')|\leq (C_{2}+C_{3})m^{-(1-\gamma)}\log m\right)\\
            \geq&1-P_{\mathcal{B}}(m)-P_{\mathcal{R}}(m)-P_{\mathcal{C}}(m).
        \end{aligned}
    \end{equation*}
\end{proof}

\subsubsection{Proof of Lemma \ref{lem: initial kernel close to fixed}}\label{app:subsec:init:kernel:close:fixed}
Conditioning on $\mathcal{B}\cap\mathcal{R}\cap\mathcal{C}$, for all $\bm{x},\bm{x}'\in\mathcal{X}$, decomposition of $\abs{K_{\bm{\theta}(0)}^{m}(\bm{x},\bm{x}')-K_{d}(\bm{x},\bm{x}')}$ by the triangle inequality gives
\begin{equation*}
    \begin{aligned}
    |K_{\bm{\theta}(0)}^{m}(\bm{x},\bm{x}')-K_{d}(\bm{x},\bm{x}')|\leq&|K_{\bm{\theta}(0)}^{m}(\bm{x},\bm{x}')-K_{\bm{\theta}(0)}^{m}(\bm{z},\bm{z}')|+|K_{d}(\bm{z},\bm{z}')-K_{d}(\bm{x},\bm{x}')| \\
    &+|K_{\bm{\theta}(0)}^{m}(\bm{z},\bm{z}')-K_{d}(\bm{z},\bm{z}')| \\
    \leq & 2C_{4}m^{-(1-\gamma)}\log m+2C_{5}(\sqrt{d}\epsilon)^{1/4}+C_{1}\sqrt{\frac{\log m}{m}} \\
    \end{aligned}
\end{equation*}
by Lemma \ref{lem: continuity of K_0}, Lemma \ref{lem: continuity of K} and Lemma \ref{lem: event C} when $m$ is sufficiently large.

\paragraph{The continuity of $K_{\bm{\theta}(0)}^{m}$}
Using the triangle inequality again yields
\begin{equation*}
    |K_{\bm{\theta}(0)}^{m}(\bm{x},\bm{x}')-K_{\bm{\theta}(0)}^{m}(\bm{z},\bm{z}')|\leq |K_{\bm{\theta}(0)}^{m}(\bm{x},\bm{x}')-K_{\bm{\theta}(0)}^{m}(\bm{x},\bm{z}')|+|K_{\bm{\theta}(0)}^{m}(\bm{x},\bm{z}')-K_{\bm{\theta}(0)}^{m}(\bm{z},\bm{z}')|.
\end{equation*}
We here illustrate how to control the first term, as the control of the second term follows from the symmetry of $K_{\bm{\theta}(0)}^{m}(\cdot,\cdot)$.

\begin{lemma}\label{lem: init activation pattern for close points}
    For all $\bm{x}\in [-B,B]^{d}$ and $\bm{z}\in\mathcal{N}_{\epsilon}$ such that $\|\bm{x}-\bm{z}\|_{2}\leq\sqrt{d}\epsilon$, conditioning on $\mathcal{B}\cap\mathcal{R}$, if $m$ is sufficiently large, then $| I(\bm{x},\bm{z})|\geq2(m-\lfloor m^{\gamma}\rfloor)$, where $I(\bm{x},\bm{z})=\{r\mid\bm{1}_{\bm{\theta}(0),r}(\bm{x})=\bm{1}_{\bm{\theta}(0),r}(\bm{z})\}$ is the index set of neurons on which the activation pattern for $\bm{x}$ and $\bm{z}$ is the same at $\bm{\theta}(0)$.
\end{lemma}

\begin{proof}
    Notice that 
    \begin{equation*}
        |h_{\bm{\theta}(0),r}(\bm{x})-h_{\bm{\theta}(0),r}(\bm{z})|=|\langle \bm{w}_{r}(0),\bm{x}-\bm{z}\rangle|\leq\|\bm{w}_{r}(0)\|_{2}\|\bm{x}-\bm{z}\|_{2}\leq dR_{B}\epsilon.
    \end{equation*}
    For $r\in I_{\text{in}}(\bm{x})=\{r\mid |h_{\bm{\theta}(0)}(\bm{x})|>(dB+1)R\}$, if $m$ is large enough such that $dR_{B}\epsilon\leq(dB+1)R$, we have
    \begin{equation*}
        |h_{\bm{\theta}(0),r}(\bm{x})-h_{\bm{\theta}(0),r}(\bm{z})|<|h_{\bm{\theta}(0),r}(\bm{x})|,
    \end{equation*}
    which implies $\bm{1}_{\bm{\theta}(0),r}(\bm{x})=\bm{1}_{\bm{\theta}(0),r}(\bm{z})$, thus $|I(\bm{x},\bm{z})|\geq|I_{\text{in}}(\bm{x})|\geq 2(m-\lfloor m^{\gamma}\rfloor)$.
\end{proof}

\begin{lemma}\label{lem: continuity of K_0}
    For all $\bm{x},\bm{x}'\in [-B,B]^{d}$ and $\bm{z}'\in\mathcal{N}_{\epsilon}$ such that $\|\bm{x}'-\bm{z}'\|_{2}\leq\sqrt{d}\epsilon$, conditioning on the event $\mathcal{B}\cap\mathcal{R}$, if we set $\gamma > 1 - \beta$ and $m$ is sufficiently large,
    \begin{equation*}
        |K_{\bm{\theta}(0)}^{m}(\bm{x},\bm{x}')-K_{\bm{\theta}(0)}^{m}(\bm{x},\bm{z}')|\leq C_{4}m^{-(1-\gamma)}\log m,
    \end{equation*}
    where $C_{4}$ is a constant depending on $d,B$.
\end{lemma}

\begin{proof}
    For simplicity, let $I=I(\bm{x}',\bm{z}')$. Then 
    \begin{equation*}
        \begin{aligned}
        &|H_{\bm{\theta}(0)}^{m}(\bm{x},\bm{x}')-H_{\bm{\theta}(0)}^{m}(\bm{x},\bm{z}')|\\
        \leq&\left|\langle\bm{x},\bm{x}'-\bm{z}'\rangle\frac{1}{m}\sum_{r\in[2m]}a_{r}^{2}(0)\bm{1}_{\bm{\theta}(0),r}(\bm{x})\bm{1}_{\bm{\theta}(0),r}(\bm{x}')\right|\\
        &+\left|(\langle\bm{x},\bm{z}'\rangle+1)\frac{1}{m}\sum_{r\in[2m]}a_{r}^{2}(0)\bm{1}_{\bm{\theta}(0),r}(\bm{x})(\bm{1}_{\bm{\theta}(0),r}(\bm{x}')-\bm{1}_{\bm{\theta}(0),r}(\bm{z}'))\right|\\
        \leq &dB\epsilon\cdot 2R_{B}^{2}+\frac{(dB^{2}+1)R_{B}^{2}}{m}\sum_{r\in I}\sum_{r\in I\comp}\abs{\bm{1}_{\bm{\theta}(0),r}(\bm{x}')-\bm{1}_{\bm{\theta}(0),r}(\bm{z}')}\\
        \leq &2dBR_{B}^{2}\epsilon+(dB^{2}+1)R_{B}^{2}\frac{2\lfloor m^{\gamma}\rfloor}{m}\\
        \leq & 4(dB^{2}+1)R_{B}^{2}m^{-(1-\gamma)},
        \end{aligned}
    \end{equation*}
    where 
    the first inequality holds by plugging in $\langle\bm{x},\bm{z}'\rangle\frac{1}{m}\sum_{r\in[2m]}a_{r}^{2}(0)\bm{1}_{\bm{\theta}(0),r}(\bm{x})\bm{1}_{\bm{\theta}(0),r}(\bm{x}')$ and using the triangle inequality, 
    the third inequality follows from Lemma \ref{lem: init activation pattern for close points}, and the last inequality holds if $m$ is large enough such that $2dBR_{B}^{2}\epsilon\leq2(dB^{2}+1)R_{B}^{2}\frac{\lfloor m^{\gamma}\rfloor}{m}$. Similarly,
    \begin{equation*}
        \begin{aligned}
            &|G_{\bm{\theta}(0)}^{m}(\bm{x},\bm{x}')-G_{\bm{\theta}(0)}^{m}(\bm{x},\bm{z}')|\\
            \leq&\left|\langle\bm{x},\bm{x}'-\bm{z}'\rangle\frac{1}{m}\sum_{r\in[2m]}\sigma(h_{\bm{\theta}(0),r}(\bm{x}))\sigma(h_{\bm{\theta}(0),r}(\bm{x}'))\right|\\
            &+\left|\langle\bm{x},\bm{z}'\rangle\frac{1}{m}\sum_{r\in[2m]}\sigma(h_{\bm{\theta}(0),r}(\bm{x}))(\sigma(h_{\bm{\theta}(0),r}(\bm{x}'))-\sigma(h_{\bm{\theta}(0),r}(\bm{z}')))\right|\\
            \leq&dB\epsilon\cdot 2(dB+1)^{2}R_{B}^{2}\\
            &+\frac{1}{m}\sum_{r\in I}\sum_{r\in I\comp}|h_{\bm{\theta}(0),r}(\bm{x})||\sigma(h_{\bm{\theta}(0),r}(\bm{x}'))-\sigma(h_{\bm{\theta}(0),r}(\bm{z}'))|\\
            \leq & 2dB(dB+1)^{2}R_{B}^{2}\epsilon\\
            &+\frac{1}{m}\sum_{r\in I}|h_{\bm{\theta}(0),r}(\bm{x})||h_{\bm{\theta}(0),r}(\bm{x}')-h_{\bm{\theta}(0),r}(\bm{z}')|\\
            &+\frac{1}{m}\sum_{r\in I\comp}|h_{\bm{\theta}(0),r}(\bm{x})|\max\{|h_{\bm{\theta}(0),r}(\bm{x}')|,|h_{\bm{\theta}(0),r}(\bm{z}')|\}\\
            \leq&2dB(dB+1)^{2}R_{B}^{2}\epsilon+2(dB+1)R_{B}\cdot dR_{B}\epsilon+(dB+1)^{2}R_{B}^{2}\frac{2\lfloor m^{\gamma}\rfloor}{m}\\
            \leq&6(dB+1)^{2}R_{B}^{2}m^{-(1-\gamma)},
        \end{aligned}
    \end{equation*}
    where 
    the first inequality holds by plugging in $\langle\bm{x},\bm{z}'\rangle\frac{1}{m}\sum_{r\in[2m]}\sigma(h_{\bm{\theta}(0),r}(\bm{x}))\sigma(h_{\bm{\theta}(0),r}(\bm{x}'))$ and using the triangle inequality, the third and the last but second inequality follows from Lemma \ref{lem: init activation pattern for close points}, the last inequality holds if $m$ is large enough such that 
    \begin{equation*}
        \max\{2dB(dB+1)^{2}R_{B}^{2}\epsilon,2d(dB+1)R_{B}^{2}\epsilon\}\leq 2(dB+1)^{2}R_{B}^{2}\frac{\lfloor m^{\gamma}\rfloor}{m}.
    \end{equation*}
\end{proof}

\paragraph{The continuity of $K_{d}$}
The triangle inequality shows
\begin{equation*}
    \abs{K_{d}(\bm{z},\bm{z}')-K_{d}(\bm{x},\bm{x}')}\leq \abs{K_{d}(\bm{z},\bm{z}')-K_{d}(\bm{z},\bm{x}')}+\abs{K_{d}(\bm{z},\bm{x}')-K_{d}(\bm{x},\bm{x}')}.
\end{equation*}
Similarly, we only need to show the control of the first term, since it is the same for the second term by the symmetry of $K_{d}(\cdot,\cdot)$. 

\begin{lemma}\label{lem: continuity of K}
    For every $\bm{x}',\bm{z},\bm{z}'\in [-B,B]^{d}$ and $\epsilon_{0}>0$, if $\|\bm{x}'-\bm{z}'\|_{2}\leq\epsilon_{0}$, then
    \begin{equation*}\label{eq: NTK epsilon close}
        |K_{d}(\bm{z},\bm{z}')-K_{d}(\bm{z},\bm{x}')|\leq C_{5}\max\{\epsilon_{0},\epsilon_{0}^{1/4}\},
    \end{equation*}
    where $C_{5}$ is a constant depending on $d,B$.
\end{lemma}

\begin{proof}
    Recall the expression of NTK $K_{d}$, we have
    \begin{equation*}
        \begin{aligned}
            &|K_{d}(\bm{z},\bm{z}')-K_{d}(\bm{z},\bm{x}')|\\
            \leq&\underbrace{\frac{2}{\pi}|(\pi-\psi(\bm{z},\bm{z}'))(\langle\bm{z},\bm{z}'\rangle+1)-(\pi-\psi(\bm{z},\bm{x}'))(\langle\bm{z},\bm{x}'\rangle+1)|}_{\text{\uppercase\expandafter{\romannumeral1}}}\\
            &+\underbrace{\frac{1}{\pi}\left|\sqrt{\|\bm{z}-\bm{z}'\|_{2}^{2}-\|\bm{z}\|_{2}^{2}\|\bm{z}'\|_{2}^{2}-\langle\bm{z},\bm{z}'\rangle^{2}}-\sqrt{\|\bm{z}-\bm{x}'\|_{2}^{2}-\|\bm{z}\|_{2}^{2}\|\bm{x}'\|_{2}^{2}-\langle\bm{z},\bm{x}'\rangle^{2}}\right|}_{\text{\uppercase\expandafter{\romannumeral2}}}.
        \end{aligned}
    \end{equation*}
    For the first term $\text{\uppercase\expandafter{\romannumeral1}}$, plugging in $(\pi-\psi(\bm{z},\bm{z}'))(\langle\bm{z},\bm{x}'\rangle+1)$ and using the triangle inequality yields
    \begin{equation*}
        \begin{aligned}
            \text{\uppercase\expandafter{\romannumeral1}}\leq & \frac{2}{\pi}\left((\pi-\psi(\bm{z},\bm{z}'))|\langle\bm{z},\bm{z}'-\bm{x}'\rangle|+|\langle \bm{z},\bm{x}'\rangle +1||\psi(\bm{z},\bm{z}')-\psi(\bm{z},\bm{x}')|\right)\\
            \leq & \frac{2}{\pi}\left( 2\pi\cdot \sqrt{d}B\epsilon_{0}+(dB^{2}+1)|\psi(\bm{z},\bm{z}')-\psi(\bm{z},\bm{x}')|\right)\\
            \leq&\frac{2}{\pi}\left(2\pi\sqrt{d}B\epsilon_{0}+(dB^{2}+1)C_{6}\max\{\sqrt{\epsilon_{0}},\epsilon_{0}^{1/4}\}\right),
        \end{aligned}
    \end{equation*}
    where the last inequality holds due to Lemma \ref{lem: continuity of psi} where $C_{6}$ is a constant depending on $d,B$. For the second term $\text{\uppercase\expandafter{\romannumeral2}}$,
    \begin{equation*}
        \begin{aligned}  \text{\uppercase\expandafter{\romannumeral2}}&\leq\sqrt{|(\|\bm{z}-\bm{z}'\|_{2}^{2}-\|\bm{z}-\bm{x}'\|_{2}^{2})+(\|\bm{z}\|_{2}^{2}\|\bm{z}'\|_{2}^{2}-\|\bm{z}\|_{2}^{2}\|\bm{x}'\|_{2}^{2})-(\langle\bm{z},\bm{z}'\rangle^{2}-\langle\bm{z},\bm{x}'\rangle^{2})|}\\
        &=\sqrt{|(2\langle\bm{z},\bm{x}'-\bm{z}'\rangle+\|\bm{z}'\|_{2}^{2}-\|\bm{x}'\|_{2}^{2})+\|\bm{z}\|_{2}^{2}(\|\bm{z}'\|_{2}^{2}-\|\bm{x}'\|_{2}^{2})-\langle\bm{z},\bm{z}'+\bm{x}'\rangle\langle\bm{z},\bm{z}'-\bm{x}'\rangle|}\\
        &\leq\sqrt{4\sqrt{d}B\epsilon_{0}+2(\sqrt{d}B)^{3}\epsilon_{0}+2(\sqrt{d}B)^{3}\epsilon_{0}}=2\sqrt{(\sqrt{d}B+d^{3/2}B^{3})}\sqrt{\epsilon_{0}},
        \end{aligned}
    \end{equation*}
    where the first inequality holds since $|\sqrt{x}-\sqrt{x'}|\leq\sqrt{|x-x'|}$ for all $x,x'>0$ and the last inequality holds by the Cauchy-Schwartz inequality and the fact that $\bm{x}',\bm{z},\bm{z}'\in [-B,B]^{d}$ and $\|\bm{x}'-\bm{z}'\|_{2}\leq\epsilon_{0}$.
\end{proof}

\begin{lemma}\label{lem: continuity of psi}
    For every $\bm{x}',\bm{z},\bm{z}'\in [-B,B]^{d}$ and $\epsilon_{0} > 0$, if $\|\bm{x}'-\bm{z}'\|_{2}\leq\epsilon_{0}$, then
    \begin{equation}\label{eq: psi epsilon close}
            |\psi(\bm{z},\bm{z}')-\psi(\bm{z},\bm{x}')|\leq C_{6}\max\{\sqrt{\epsilon_{0}},\epsilon_{0}^{1/4}\},
    \end{equation}
    where $C_{6}$ is a constant depending on $d,B$.
\end{lemma}

\begin{proof}
    Let $\Delta=|\cos(\psi(\bm{z},\bm{z}'))-\cos(\psi(\bm{z},\bm{x}'))|$. Then plug in $\frac{\langle\bm{z},\bm{x}'\rangle+1}{\sqrt{(\|\bm{z}\|_{2}^{2}+1)(\|\bm{z}'\|_{2}^{2}+1)}}$ and the triangle inequality concludes
    \begin{equation*}
        \begin{aligned}
        \Delta\leq&\left|\frac{\langle\bm{z},\bm{z}'\rangle+1}{\sqrt{(\|\bm{z}\|_{2}^{2}+1)(\|\bm{z}'\|_{2}^{2}+1)}}-\frac{\langle\bm{z},\bm{x}'\rangle+1}{\sqrt{(\|\bm{z}\|_{2}^{2}+1)(\|\bm{z}'\|_{2}^{2}+1)}}\right|\\
        &+\left|\frac{\langle\bm{z},\bm{x}'\rangle+1}{\sqrt{(\|\bm{z}\|_{2}^{2}+1)(\|\bm{z}'\|_{2}^{2}+1)}}-\frac{\langle\bm{z},\bm{x}'\rangle+1}{\sqrt{(\|\bm{z}\|_{2}^{2}+1)(\|\bm{x}'\|_{2}^{2}+1)}}\right|\\
        =&\frac{1}{\sqrt{(\|\bm{z}\|_{2}^{2}+1)(\|\bm{z}'\|_{2}^{2}+1)}}|\langle\bm{z},\bm{z}'-\bm{x}'\rangle|\\
            &+\frac{1}{\sqrt{(\|\bm{z}\|_{2}^{2}+1)(\|\bm{z}'\|_{2}^{2}+1)(\|\bm{x}'\|_{2}^{2}+1)}}|\langle\bm{z},\bm{x}'\rangle+1|\left|\sqrt{\|\bm{x}'\|_{2}^{2}+1}-\sqrt{\|\bm{z}'\|_{2}^{2}+1}\right|\\
        \leq&\sqrt{d}B\epsilon_{0}+(dB^{2}+1)\sqrt{2\sqrt{d}B}\sqrt{\epsilon_{0}}\\
        \end{aligned}
    \end{equation*}
    where the last line follows from the fact that $|\sqrt{x^{2}+1}-\sqrt{x'^{2}+1}|\leq\sqrt{|x+x'|}\sqrt{|x-x'|}$ for all $x,x'\in\mathbb{R}$.
    Then we have
    \begin{equation*}
        \begin{aligned}
        \abs{\psi(\bm{z},\bm{z}')-\psi(\bm{z},\bm{x}')}\leq&|\arccos 1-\arccos(1-\Delta)|=\int_{1-\Delta}^{1}\frac{1}{\sqrt{1-x^{2}}}\dx x \\
        \leq& \int_{1-\Delta}^{1}\frac{1}{\sqrt{1-x}}\dx x =2\sqrt{\Delta}\\
        \leq&2\sqrt{\sqrt{d}B\epsilon_{0}+(dB^{2}+1)\sqrt{2\sqrt{d}B}\sqrt{\epsilon_{0}}}.
        \end{aligned}
    \end{equation*}
\end{proof}

\subsubsection{Proof of Lemma \ref{lem: kernel with lazy params close to initial kernel}}\label{app:subsec:lazy:kernel:close:init}
It is hard to analyze $K_{\bm{\theta}(t)}^{m}$ directly, so we show $K_{\bm{\theta}}^{m}$ is close to $K_{\bm{\theta}(0)}^{m}$ if $\bm{\theta}$ is close to $\bm{\theta}(0)$ in Lemma \ref{lem: kernel close to init when params close to init} first and then prove $\bm{\theta}(t)$ is indeed near $\bm{\theta}(0)$ in Proposition \ref{prop: lazy regime}. 

\paragraph{Approximation for $K_{\bm{\theta}}^{m}$ to $K_{\bm{\theta}(0)}^{m}$}
Denote by
\begin{equation*}
    \bm{\Theta}(\bm{\theta}(0),R_{0})=\left\{\bm{\theta}~\middle|~|a_{r}-a_{r}(0)|,|\bm{w}_{r,(j)}-\bm{w}_{r,(j)}(0)|,|b_{r}-b_{r}(0)|\leq R_{0}, r\in [2m],j\in[d]\right\}
\end{equation*}
the neighborhood of $\bm{\theta}(0)$. 

\begin{lemma}\label{lem: activation pattern for params close to init}
    For all $\bm{x}\in\mathcal{X}$, for all $\bm{\theta}\in \bm{\Theta}(\bm{\theta}(0),R)$, conditioning on the event $\mathcal{R}$, then $|I(\bm{x})|\geq2(m-\lfloor m^{\gamma}\rfloor)$ where $I(\bm{x})=\{r\mid\bm{1}_{\bm{\theta},r}(\bm{x})=\bm{1}_{\bm{\theta}(0),r}(\bm{x})\}$ is the index set of neurons on which the activation pattern for $\bm{x}$ is the same at $\bm{\theta}$ and $\bm{\theta}(0)$.
\end{lemma}

\begin{proof}
    Since $\bm{\theta}\in \bm{\Theta}(\bm{\theta}(0),R)$, we have
    \begin{equation*}
        |h_{\bm{\theta},r}(\bm{x})-h_{\bm{\theta}(0),r}(\bm{x})|=|\langle\bm{w}_{r}-
        \bm{w}_{r}(0),\bm{x}\rangle+(b_{r}-b_{r}(0))|\leq(dB+1)R.
    \end{equation*}
    For $r\in I_{\mathrm{in}}(\bm{x})=\left\{r~\middle|~|h_{\bm{\theta}(0),r}(\bm{x})|>(dB+1)R~\right\}$, we have 
    \begin{equation*}
        |h_{\bm{\theta},r}(\bm{x})-h_{\bm{\theta}(0),r}(\bm{x})|<|h_{\bm{\theta}(0),r}(\bm{x})|,
    \end{equation*}
    which implies $\bm{1}_{\bm{\theta},r}(\bm{x})=\bm{1}_{\bm{\theta}(0),r}(\bm{x})$, thus $|I(\bm{x})|\geq|I_{\text{in}}(\bm{x})|\geq 2(m-\lfloor m^{\gamma}\rfloor)$.
\end{proof}

\begin{lemma}\label{lem: kernel close to init when params close to init}
    Conditioning on the event $\mathcal{B}\cap\mathcal{R}$, if $m$ is sufficiently large, then 
    \begin{equation*}
        \sup_{\bm{\theta}\in\bm{\Theta}(\bm{\theta}(0),R)}\sup_{\bm{x},\bm{x}'\in\mathcal{X}}|K_{\bm{\theta}}^{m}(\bm{x},\bm{x}')-K_{\bm{\theta}(0)}^{m}(\bm{x},\bm{x}')|\leq C_{7}m^{-(1-\gamma)}\log m,
    \end{equation*}
    where $C_{7}>0$ is a constant depending on $d,B$.
\end{lemma}

\begin{proof}
For all $\bm{x},\bm{x}'\in\mathcal{X}$, let $I=I(\bm{x})\cap I(\bm{x}')$, then $|I|=|I(\bm{x})|+|I(\bm{x}')|-|I(\bm{x})\cup I(\bm{x}')|\geq2m-4\lfloor m^{\gamma}\rfloor$ by Lemma \ref{lem: activation pattern for params close to init}. 
    Hence for all $\bm{\theta}\in\bm{\Theta}(\bm{\theta}(0),R)$,
    \begin{equation*}
    \begin{aligned}
      & \quad |H_{\bm{\theta}}^{m}(\bm{x},\bm{x}')-H_{\bm{\theta}(0)}^{m}(\bm{x},\bm{x}')|\\
      & \leq \frac{(dB^{2}+1)}{m}\sum_{r\in [2m]}|a_{r}^{2}\bm{1}_{\bm{\theta},r}(\bm{x})\bm{1}_{\bm{\theta},r}(\bm{x}')-a_{r}^{2}(0)\bm{1}_{\bm{\theta}(0),r}(\bm{x})\bm{1}_{\bm{\theta}(0),r}(\bm{x}')|\\
      & \leq \frac{(dB^{2}+1)}{m}\left(\sum_{r\in I}\abs{a_{r}^2-a_{r}^{2}(0)}+\sum_{r\in I\comp}\max\{a_{r}^2,a_{r}^{2}(0)\}\right)\\
      & \leq (dB^{2}+1)\left(\frac{|I|}{m}3RR_{B}+\frac{|I\comp|}{m}4R_{B}^{2}\right)\\
      & \leq (dB^{2}+1)\left(3RR_{B}+16R_{B}^{2}\frac{\lfloor m^{\gamma}\rfloor}{m}\right).
    \end{aligned}
\end{equation*}

Similarly, we have
\begin{equation*}\label{eq: Gm G0m bound}
      \begin{aligned}
        &\quad |G_{\bm{\theta}}^{m}(\bm{x},\bm{x}')-G_{\bm{\theta}(0)}^{m}(\bm{x},\bm{x}')|\\
        & \leq \frac{1}{m}\sum_{r\in[2m]}|\sigma(h_{\bm{\theta},r}(\bm{x}))\sigma(h_{\bm{\theta},r}(\bm{x}'))-\sigma(h_{\bm{\theta}(0),r}(\bm{x}))\sigma(h_{\bm{\theta}(0),r}(\bm{x}'))|\\
        & \leq \frac{1}{m}\bigg(\sum_{r\in I}\left(|h_{\bm{\theta},r}(\bm{x})||h_{\bm{\theta},r}(\bm{x}')-h_{\bm{\theta}(0),r}(\bm{x}')|+|h_{\bm{\theta},r}(\bm{x})-h_{\bm{\theta}(0),r}(\bm{x})||h_{\bm{\theta}(0),r}(\bm{x}')|\right)\\
        & \quad +\sum_{r\in I\comp}\max\{|h_{\bm{\theta},r}(\bm{x})||h_{\bm{\theta},r}(\bm{x}')|,|h_{\bm{\theta}(0),r}(\bm{x})||h_{\bm{\theta}(0),r}(\bm{x}')|\}\bigg)\\
        & \leq \frac{|I|}{m}4(dB+1)^{2}R_{B}R+\frac{|I\comp|}{m}8(dB+1)^{2}R_{B}^{2} \\
        & \leq 4(dB+1)^{2}R_{B}R+8(dB+1)^{2}R_{B}^{2}\frac{\lfloor m^{\gamma}\rfloor}{m}.
      \end{aligned}
\end{equation*}
    Simply by the triangle inequality, we have
    \begin{equation*}
        \begin{aligned}
            & \quad |K_{\bm{\theta}}^{m}(\bm{x},\bm{x}')-K_{\bm{\theta}(0)}^{m}(\bm{x},\bm{x}')|\\
            & \leq |H_{\bm{\theta}}^{m}(\bm{x},\bm{x}')-H_{\bm{\theta}(0)}^{m}(\bm{x},\bm{x}')|+|G_{\bm{\theta}}^{m}(\bm{x},\bm{x}')-G_{\bm{\theta}(0)}^{m}(\bm{x},\bm{x}')|\\
            & \leq (dB^{2}+1)\left(3R_{B}R+16R_{B}^{2}\frac{\lfloor m^{\gamma}\rfloor}{m}\right)+ 4(dB+1)^{2}R_{B}R+8(dB+1)^{2}R_{B}^{2}\frac{\lfloor m^{\gamma}\rfloor}{m}\\
            & \leq C_{7}m^{-(1-\gamma)}\log m,
        \end{aligned}
    \end{equation*}
    where the last inequality holds when $m$ is sufficiently large.
\end{proof}

\paragraph{Lazy regime}\label{app:lazy:regime}

\begin{proposition}\label{prop: lazy regime}
    Let $R'=\frac{4\sqrt{3}(dB+1)\|\bm{y}\|_{2}}{\lambda_{\min}(K_{d}(\bm{X},\bm{X}))\sqrt{n}}\sqrt{\frac{\log m}{m}}$. Denote the ``lazy regime'' event by
    \begin{equation*}
      \mathcal{A}=\mathcal{A}_{\lambda}\cap\mathcal{A}_{\bm{\theta}}\cap\mathcal{A}_{\bm{u}},
    \end{equation*}
    where
    \begin{gather*}
      \mathcal{A}_{\lambda}=\left\{\omega~\middle|~\lambda_{\min}(K_{\bm{\theta}(t)}^{m}(\bm{X},\bm{X}))\geq \frac{\lambda_{\min}(K_{d}(\bm{X},\bm{X}))}{2},~\forall t\geq 0\right\},\\
      \mathcal{A}_{\bm{\theta}}=\left\{\omega~\middle|~\bm{\theta}(t)\in\bm{\Theta}(\bm{\theta}(0),R'),~\forall t\geq 0\right\},\\
      \mathcal{A}_{\bm{u}}=\left\{\omega~\middle|~\|\bm{u}^{m}(t)\|_{2}^{2}\leq e^{-\frac{\lambda_{\min}(K_{d}(\bm{X},\bm{X}))}{n}t}\|\bm{u}(0)\|_{2}^{2},~\forall t\geq 0\right\}.
    \end{gather*}
    If we further set $\alpha<1/2$ and $\gamma > 1-\beta/4$, when $m$ is sufficiently large, such that $R'<R$, then we have 
    \begin{equation*}
        \mathcal{A}\supseteq \mathcal{B}\cap\mathcal{R}\cap\mathcal{C}.
    \end{equation*}
\end{proposition}

The proof of Proposition \ref{prop: lazy regime} is deferred to the end of Appendix \ref{app:lazy:regime}. To prove it, we need the following three lemmas. 

\begin{lemma}\label{lem: smallest eigenvalue leads to fast convergence}
    For some $t\geq0$, if there exists some $\lambda_{\min}>0$ such that for all $0\leq s\leq t$,
    \begin{equation*}
       \lambda_{\min}(K_{\bm{\theta}(s)}^{m}(\bm{X},\bm{X}))\geq \lambda_{\min}/2,    
    \end{equation*}
    then 
    \begin{equation*}
      \|\bm{u}^{m}(t)\|_{2}^{2}\leq e^{-\frac{\lambda_{\min}}{n}t}\| \bm{u}^{m}(0)\|_{2}^{2}.
    \end{equation*}
\end{lemma}

\begin{proof}
    Notice that
\begin{equation*}
\begin{aligned}
\frac{\partial\|\bm{u}^{m}(s)\|_{2}^{2}}{\partial s}=\frac{\partial\|\bm{u}^{m}(s)\|_{2}^{2}}{\partial\bm{u}^{m}(s)}\frac{\partial\bm{u}^{m}(s)}{\partial s}=-\frac{2}{n}\bm{u}^{m}(s)^{\top} K_{\bm{\theta}(s)}^{m}(\bm{X},\bm{X}) \bm{u}^{m}(s)\leq -\frac{\lambda_{\min}}{n}\|\bm{u}^{m}(s)\|_{2}^{2}
\end{aligned}
\end{equation*}
leads to 
\begin{equation*}
\frac{\partial e^{\frac{\lambda_{\min}}{n}s}\|\bm{u}^{m}(s)\|_{2}^{2}}{\partial s}=e^{\frac{\lambda_{\min}}{n}s}\left(\frac{\lambda_{\min}}{n}\|\bm{u}^{m}(s)\|_{2}^{2}+\frac{\partial\|\bm{u}^{m}(s)\|_{2}^{2}}{\partial s}\right)\leq 0.
\end{equation*}
Thus $e^{\frac{\lambda_{\min}}{n}s}\|\bm{u}^{m}(s)\|_{2}^{2}$ is non-increasing, which implies $e^{\frac{\lambda_{\min}}{n}s}\|\bm{u}^{m}(t)\|_{2}^{2}\leq \|\bm{u}^{m}(0)\|_{2}^{2}$. 
\end{proof}

\begin{lemma}\label{lem: bound smallest eigenvalue}
    Conditioning on $\mathcal{B}\cap\mathcal{R}\cap\mathcal{C}$, if we set $\gamma>1-\beta/4$ and $m$ is sufficiently large, then for all $\bm{\theta}\in\bm{\Theta}(\bm{\theta}(0),R)$,
    \begin{equation*}
        \lambda_{\min}(K_{\bm{\theta}}^{m}(\bm{X},\bm{X}))\geq\lambda_{\min}(K_{d}(\bm{X},\bm{X}))/2.
    \end{equation*}
\end{lemma}

\begin{proof}
Notice that
\begin{equation*}
    \begin{aligned}
    & \quad \|K_{\bm{\theta}}^{m}(\bm{X},\bm{X})-K_{d}(\bm{X},\bm{X})\|_{2} \\
    & \leq \lVert K_{\bm{\theta}}^{m}(\bm{X},\bm{X})-K_{\bm{\theta}(0)}^{m}(\bm{X},\bm{X})\rVert_{2}+\lVert K_{\bm{\theta}(0)}^{m}(\bm{X},\bm{X})-K_{d}(\bm{X},\bm{X})\rVert_{2}\\
    & \leq \|K_{\bm{\theta}}^{m}(\bm{X},\bm{X})-K_{\bm{\theta}(0)}^{m}(\bm{X},\bm{X})\|_{\mathrm{F}}+\| K_{\bm{\theta}(0)}^{m}(\bm{X},\bm{X})-K_{d}(\bm{X},\bm{X})\|_{\mathrm{F}}\\
    & \leq (C_{2}+C_{7})nm^{-(1-\gamma)}\log m,
    \end{aligned}
\end{equation*}
where the last inequality follows from Lemma \ref{lem: kernel close to init when params close to init} and Lemma \ref{lem: initial kernel close to fixed}.
If $m$ large enough such that $\|K_{\bm{\theta}}^{m}(\bm{X},\bm{X})-K_{d}(\bm{X},\bm{X})\|_{2}\leq \lambda_{\min}(K_{d}(\bm{X},\bm{X}))/2$, then
\begin{equation*}
\begin{aligned}
  \lambda_{\min}(K_{\bm{\theta}}^{m}(\bm{X},\bm{X}))&\geq\lambda_{\min}(K_{d}(\bm{X},\bm{X}))-\|K_{\bm{\theta}}^{m}(\bm{X},\bm{X})-K_{d}(\bm{X},\bm{X})\|_{2}\\
  &\geq \lambda_{\min}(K_{d}(\bm{X},\bm{X}))/2.
\end{aligned}
\end{equation*}
\end{proof}

\begin{lemma}\label{lem: smallest eigenvalue lowerbound and lazy training}
For some $t\geq0$, suppose that $\lambda_{\min}(K_{\bm{\theta}(s)}^{m}(\bm{X},\bm{X}))\geq \lambda_{\min}(K_{d}(\bm{X},\bm{X}))/2$ holds for all $s\in [0,t]$ and we set $\alpha<1/2$, so that $R'<R$ when $m$ is sufficiently large. Then conditioning on the event $\mathcal{B}\cap\mathcal{R}\cap\mathcal{C}$, we have $\bm{\bm{\theta}}(s)\in \bm{\Theta}(\bm{\theta}(0),R')$ for all $s\in[0,t]$ when $m$ is sufficiently large.
\end{lemma}

\begin{proof}
    We prove the following two statements instead.
\begin{enumerate}
    \item If $|\bm{w}_{r,(j)}(s)-\bm{w}_{r,(j)}(0)|\leq R$ and $|b_{r}(s)-b_{r}(0)|\leq R$ hold for all $r\in[2m],j\in[d]$ and all $s\in[0,t]$, then $|a_{r}(t)-a_{r}(0)|\leq R'$ holds for all $r\in [2m]$;
    \item If $|a_{r}(s)-a_{r}(0)|\leq R$ holds for all $r\in [2m]$ and for all $s\in[0,t]$, then $|\bm{w}_{r,(j)}(t)-\bm{w}_{r,(j)}(0)|\leq R',|b_{r}(t)-b_{r}(0)|\leq R'$ hold for all $r\in [2m],j\in[d]$.
\end{enumerate}
We can bound the distance from initializations by integrating the norm of gradient since $\|\bm{v}(t)-\bm{v}(0)\|_{2}\leq \int_{0}^{t}\|\dot{\bm{v}}(s)\|_{2}\dx s$ for any vector-valued function $\bm{v}(t)$.
The gradient flow of parameters is as follows: 
\begin{gather*}
  \dot{a}_{r}(s)=-\nabla_{a_{r}}\emprisk(f_{\bm{\theta}(s)}^{m})=-\frac{1}{n}m^{-1/2}\sum_{i\in[n]}\sigma(h_{\bm{\theta}(s),r}(\bm{x}_{i}))\bm{u}_{i}^{m}(s),\\ 
  \dot{\bm{w}}_{r,(j)}(s)=-\nabla_{\bm{w}_{r,(j)}}\emprisk(f_{\bm{\theta}(s)}^{m})=-\frac{1}{n}m^{-1/2}\sum_{i\in[n]} a_{r}(s)\bm{x}_{i,(j)}\bm{1}_{\bm{\theta}(s),r}(\bm{x}_{i})\bm{u}_{i}^{m}(s),\\ 
  \dot{b}_{r}(s)=-\nabla_{b_{r}}\emprisk(f_{\bm{\theta}(s)}^{m})=-\frac{1}{n}m^{-1/2}\sum_{i\in[n]} a_{r}(s)\bm{1}_{\bm{\theta}(s),r}(\bm{x}_{i})\bm{u}_{i}^{m}(s).
\end{gather*}
By the Cauchy-Schwartz inequality and Lemma \ref{lem: smallest eigenvalue leads to fast convergence}, 
\begin{equation*}
  \sum_{i\in[n]}|\bm{u}_{i}^{m}(s)|\leq\sqrt{n}\|\bm{u}^{m}(s)\|_{2}\leq\sqrt{n}e^{-\frac{\lambda_{\min}(K_{d}(\bm{X},\bm{X}))}{2n}s}\|\bm{y}\|_{2}.
\end{equation*} 
In the following, we suppose $m$ is sufficiently large such that $R\leq R_{B}$.

\noindent 1. Since $|\bm{w}_{r,(j)}(s)-\bm{w}_{r,(j)}(0)|\leq R$ and $|b_{r}(s)-b_{r}(0)|\leq R$ hold for all $r\in[2m],j\in[d]$, we have $\sigma(h_{\bm{\theta}(s),r}(\bm{x}_{i}))\leq 2(dB+1)R_{B}$. Thus, according to the gradient flow, we have
\begin{equation*}
  \begin{aligned}
  |\dot{a}_{r}(s)|\leq&\frac{1}{n}m^{-1/2}\max_{i\in[n]}\sigma(h_{\bm{\theta}(s),r}(\bm{x}_{i}))\sum_{i\in[n]}|\bm{u}_{i}^{m}(s)|\\ 
  \leq& \frac{1}{n}m^{-1/2}2(dB+1)R_{B}\sqrt{n}e^{-\frac{\lambda_{\min}(K_{d}(\bm{X},\bm{X}))}{2n}s}\|\bm{y}\|_{2}
  \end{aligned}
\end{equation*}
and
\begin{equation*}
  |a_{r}(t)-a_{r}(0)|\leq \int_{0}^{t}|\dot{a}_{r}(s)|\dx s\leq\frac{4\sqrt{3}(dB+1)\|\bm{y}\|_{2}}{\lambda_{\min}(K_{d}(\bm{X},\bm{X}))\sqrt{n}}\sqrt{\frac{\log m}{m}}\leq R'.
\end{equation*}
hold for every $r\in [2m]$.

\noindent 2. Since $|a_{r}(s)-a_{r}(0)|\leq R$ holds for all $r\in[2m]$, we have $|a_{r}(s)|\leq2R_{B}$. Thus, we have that
\begin{gather*}
  |\dot{\bm{w}}_{r,(j)}(s)|\leq \frac{1}{n}m^{-1/2}B2R_{B}\sqrt{n}e^{-\frac{\lambda_{\min}(K_{d}(\bm{X},\bm{X}))}{2n}s}\|\bm{y}\|_{2},\\
  |\dot{b}_{r}(s)|\leq \frac{1}{n}m^{-1/2}2R_{B}\sqrt{n}e^{-\frac{\lambda_{\min}(K_{d}(\bm{X},\bm{X}))}{2n}s}\|\bm{y}\|_{2},
\end{gather*}
and   
\begin{gather*}
  |\bm{w}_{r,(j)}(t)-\bm{w}_{r,(j)}(0)|\leq \int_{0}^{t}|\dot{\bm{w}}_{r,(j)}(s)|\dx s\leq \frac{4\sqrt{3}B\|\bm{y}\|_{2}}{\lambda_{\min}(K_{d}(\bm{X},\bm{X}))\sqrt{n}}\sqrt{\frac{\log m}{m}}\leq R',\\
  |b_{r}(t)-b_{r}(0)|\leq\int_{0}^{t}|\dot{b}_{r}(s)|\dx s \leq\frac{4\sqrt{3}\|\bm{y}\|_{2}}{\lambda_{\min}(K_{d}(\bm{X},\bm{X}))\sqrt{n}}\sqrt{\frac{\log m}{m}}\leq R',
\end{gather*}
hold for all $r\in[2m],j\in[d]$.  
\end{proof}

\begin{proof}[Proof of Proposition \ref{prop: lazy regime}]
    For every $\omega\in\mathcal{B}\cap\mathcal{R}\cap\mathcal{C}$, let 
    $\tau=\min\{\tau_{\lambda},\tau_{\bm{\theta}},\tau_{\bm{u}}\}$, where  
\begin{gather*}
    \tau_{\lambda}=\inf\left\{t~\middle|~\lambda_{\min}(K_{\bm{\theta}(t)}^{m}(\bm{X},\bm{X}))<\lambda_{\min}(K_{d}(\bm{X},\bm{X}))/2\right\},\\
    \tau_{\bm{\theta}}=\inf\left\{t~\middle|~\bm{\theta}(t)\notin\bm{\Theta}(\bm{\theta}(0),R')\right\},\\
    \tau_{\bm{u}}=\inf\left\{t~\middle|~\|\bm{u}^{m}(t)\|_{2}^{2}>e^{-\frac{\lambda_{\min}(K_{d}(\bm{X},\bm{X}))}{n}t}\|\bm{u}(0)\|_{2}^{2}\right\}.
\end{gather*}
We will show that $\tau=\infty$ by contradiction, so that $\omega\in\mathcal{A}$. 
Notice that $\tau_{\lambda}<\tau_{\bm{u}}$ and $\tau_{\lambda}<\tau_{\bm{\theta}}$ according to Lemma \ref{lem: smallest eigenvalue leads to fast convergence} and Lemma \ref{lem: smallest eigenvalue lowerbound and lazy training} respectively. However if $\tau=\tau_{\lambda}<\infty$, which suggests that $\bm{\theta}(\tau)\notin\bm{\Theta}(\bm{\theta}(0),R)$ by Lemma \ref{lem: bound smallest eigenvalue}. Then there must exists some $t_{0}$ such that $0\leq t_{0}<\tau$ and $\bm{\theta}(t_{0})\notin\bm{\Theta}(\bm{\theta}(0),R')$ since $R'<R$, which however violates the assumption that $\tau=\tau_{\lambda}$.
\end{proof}

\subsection{Proof of Proposition \ref{prop:funct:approx}}
Since we have $f_{\bm{\theta}(0)}^{m}(\bm{x})=f_{0}^{\NTK}(\bm{x})=0$ for every $\bm{x}\in\mathcal{X}$ under our initialization setting, we can bound the difference between $f_{\bm{\theta}(t)}^{m}(\bm{x})$ and $f_{t}^{\NTK}(\bm{x})$ by bounding the difference between their derivative, i.e., 
\begin{equation*}
    |f_{\bm{\theta}(t)}^{m}(\bm{x})-f_{t}^{\NTK}(\bm{x})|=\left|\int_{0}^{t}\dot{f}_{\bm{\theta}(s)}^{m}(\bm{x})-\dot{f}_{s}^{\NTK}(\bm{x})\dx s\right|\leq\int_{0}^{t}|\dot{f}_{\bm{\theta}(s)}^{m}(\bm{x})-\dot{f}_{s}^{\NTK}(\bm{x})|\dx s.
\end{equation*}

Recall that
\begin{gather*}
\dot{f}_{\bm{\theta}(s)}^{m}(\bm{x})=-\frac{1}{n}K_{\bm{\theta}(s)}^{m}(\bm{x},\bm{X})\bm{u}^{m}(s),\\
\dot{f}_{s}^{\NTK}(\bm{x})=-\frac{1}{n}K_{d}(\bm{x},\bm{X})\bm{u}^{\NTK}(s),
\end{gather*}
where $\bm{u}^{m}(s) = f_{\bm{\theta}(s)}^{m}(\bm{X}) - \bm{y}$ and $\bm{u}^{\NTK}(s) = f_{s}^{\NTK}(\bm{X}) - \bm{y}$. Let 
\begin{equation*}
    \Delta=\sup_{t\geq 0}\sup_{\bm{x},\bm{x}'\in\mathcal{X}}|K_{\bm{\theta}(t)}^{m}(\bm{x},\bm{x}')-K_{d}(\bm{x},\bm{x}')|.
\end{equation*}
Then
\begin{equation*}
    \begin{aligned}
    & \quad |f_{\bm{\theta}(t)}^{m}(\bm{x})-f_{t}^{\NTK}(\bm{x})|\\
    & \leq \frac{1}{n} \int_{0}^{t}|K_{\bm{\theta}(s)}^{m}(\bm{x},\bm{X})\bm{u}^{m}(s)-K_{d}(\bm{x},\bm{X})\bm{u}^{\NTK}(s)|\dx s \\
    & \leq \frac{1}{n}\int_{0}^{t}\left(\| K_{\bm{\theta}(s)}^{m}(\bm{x},\bm{X})^{\top}-K_{d}(\bm{x},\bm{X})^{\top}\|_{2} +\| K_{d}(\bm{x},\bm{X})^{\top}\|_{2}\right) \| \bm{u}^{m}(s)-\bm{u}^{\NTK}(s)\|_{2} \dx s\\
    & \quad + \frac{1}{n} \int_{0}^{t}\| K_{\bm{\theta}(s)}^{m}(\bm{x},\bm{X})^{\top}-K_{d}(\bm{x},\bm{X})^{\top}\|_{2}\| \bm{u}^{\NTK}(s)\|_{2}\dx s\\
    & \leq \frac{1}{n}\cdot(\sqrt{n}\Delta+\sqrt{n}C) \cdot \int_{0}^{t}\|\bm{y}\|_{2}\Delta s e^{-\frac{1}{n}(\lambda_{\min}(K_{d}(\bm{X},\bm{X}))-n\Delta)s} \dx s \\
    & \quad + \frac{1}{n}\cdot\sqrt{n}\Delta\cdot \int_{0}^{t} e^{-\frac{1}{n}\lambda_{\min}(K_{d}(\bm{X},\bm{X}))s}\|\bm{y}\|_{2} \dx s\\
    & \leq \frac{1}{n}\cdot(\sqrt{n}\Delta+\sqrt{n}C) \cdot\|\bm{y}\|_{2}\Delta(\lambda_{\min}(K_{d}(\bm{X},\bm{X}))/2)^{-2} \\
    & \quad + \frac{1}{n}\cdot\sqrt{n}\Delta\cdot\left(\frac{1}{n}\lambda_{\min}(K_{d}(\bm{X},\bm{X}))\right)^{-1}\|\bm{y}\|_{2} \\
    & \leq \frac{\epsilon}{2}+\frac{\epsilon}{2} = \epsilon,
    \end{aligned}
\end{equation*}
where $C>0$ is a constant depending only on $B$ and we apply Lemma \ref{lem: u_NTK} and Lemma \ref{lem: u-u_NTK} in the third inequality and the last line follows from Proposition \ref{prop:kernel:approx} that for sufficiently large $m$, we have
\begin{equation*}
    \Delta\leq\min\left\{C,\frac{\epsilon(\lambda_{\min}(K_{d}(\bm{X},\bm{X})))^{2}\sqrt{n}}{16C\|\bm{y}\|_{2}},\frac{\epsilon\lambda_{\min}(K_{d}(\bm{X},\bm{X}))}{2\sqrt{n}\|\bm{y}\|_{2}}\right\},
\end{equation*}
with probability at least $1-\delta$ over initialization.

\begin{lemma}\label{lem: u_NTK} 
    For all $t\geq0$, we have $\|\bm{u}^{\NTK}(t)\|_{2} \leq e^{-\frac{1}{n}\lambda_{\min}(K_{d}(\bm{X},\bm{X}))t}\|\bm{y}\|_{2}$.
\end{lemma}

\begin{proof}
    Recall that $\dot{\bm{u}}^{\NTK}(t)=-\frac{1}{n}K_{d}(\bm{X},\bm{X})\bm{u}^{\NTK}(t)$. Notice that $K_{d}(\bm{X},\bm{X})$ is fixed, so we can write the explicit form $\bm{u}^{\NTK}(t)=e^{-\frac{1}{n}K_{d}(\bm{X},\bm{X})t}\bm{u}^{\NTK}(0)$, where $\bm{u}^{\NTK}(0)=-\bm{y}$. Then $\| \bm{u}^{\NTK}(t)\|_{2}\leq e^{-\frac{1}{n}\lambda_{\min}(K_{d}(\bm{X},\bm{X}))t}\| \bm{y}\|_{2}$, since $\|e^{-\frac{1}{n}K_{d}(\bm{X},\bm{X})t}\|_{2}=e^{-\frac{1}{n}\lambda_{\min}(K_{d}(\bm{X},\bm{X}))t}$.
\end{proof}

\begin{lemma}\label{lem: u-u_NTK}
    For all $t\geq0$, we have 
    \begin{equation*}
    \|\bm{u}^{m}(t)-\bm{u}^{\NTK}
    (t)\|_{2} \leq \|\bm{y}\|_{2}\Delta t e^{-\frac{1}{n}(\lambda_{\min}(K_{d}(\bm{X},\bm{X}))-n\Delta)t}
    \end{equation*}
\end{lemma}

\begin{proof}
    Recall that we can express explicitly for all $s\geq 0$,
\begin{gather*}
    \dot{\bm{u}}^{m}(s)=-\frac{1}{n}K_{\bm{\theta}(s)}^{m}(\bm{X},\bm{X})\bm{u}^{m}(s)\\
    \text{~and~}\dot{\bm{u}}^{\NTK}(s)=-\frac{1}{n}K_{d}(\bm{X},\bm{X})\bm{u}^{\NTK}(s).
\end{gather*}

Notice that
\begin{equation*}
    \frac{\dx}{\dx s}e^{\frac{1}{n}K_{d}(\bm{X},\bm{X})s}\left(\bm{u}^{m}(s)-\bm{u}^{\NTK}(s)\right)=\frac{1}{n}e^{\frac{1}{n}K_{d}(\bm{X},\bm{X})s}(K_{d}(\bm{X},\bm{X})-K_{\bm{\theta}(s)}^{m}(\bm{X},\bm{X}))\bm{u}^{m}(s).
\end{equation*}
Since $\bm{u}^{m}(0)=\bm{u}^{\NTK}(0)$, integrating gives
\begin{equation*}
    e^{\frac{1}{n}K_{d}(\bm{X},\bm{X})t}\left(\bm{u}^{m}(t)-\bm{u}^{\NTK}(t)\right)=\frac{1}{n}\int_{0}^{t}e^{\frac{1}{n}K_{d}(\bm{X},\bm{X})s}(K_{d}(\bm{X},\bm{X})-K_{\bm{\theta}(s)}^{m}(\bm{X},\bm{X}))\bm{u}^{m}(s)\dx s,
\end{equation*}
then
\begin{equation*}
    \bm{u}^{m}(t)-\bm{u}^{\NTK}(t)=\frac{1}{n}\int_{0}^{t}e^{\frac{1}{n}K_{d}(\bm{X},\bm{X})(s-t)}(K_{d}(\bm{X},\bm{X})-K_{\bm{\theta}(s)}^{m}(\bm{X},\bm{X}))\bm{u}^{m}(s)\dx s.
\end{equation*}
Bounding the norm implies
\begin{equation*}
    \begin{aligned}
        &\|\bm{u}^{m}(t)-\bm{u}^{\NTK}(t)\|_{2}\\
        \leq&\frac{1}{n}\int_{0}^{t}\| e^{\frac{1}{n}K_{d}(\bm{X},\bm{X})(s-t)}(K_{d}(\bm{X},\bm{X})-K_{\bm{\theta}(s)}^{m}(\bm{X},\bm{X}))\bm{u}^{m}(s)\|_{2}\dx s\\
        \leq& \frac{1}{n}\int_{0}^{t}e^{\frac{1}{n}\lambda_{\min}(K_{d}(\bm{X},\bm{X}))(s-t)}\|K_{\bm{\theta}(s)}^{m}(\bm{X},\bm{X})-K_{d}(\bm{X},\bm{X})\|_{2}\| \bm{u}^{\NTK}(s)\|_{2}\dx s\\
        &+\frac{1}{n}\int_{0}^{t}e^{\frac{1}{n}\lambda_{\min}(K_{d}(\bm{X},\bm{X}))(s-t)}\| K_{\bm{\theta}(s)}^{m}(\bm{X},\bm{X})-K_{d}(\bm{X},\bm{X})\|_{2}\| \bm{u}^{m}(s)-\bm{u}^{\NTK}(s)\|_{2}\dx s.
    \end{aligned}
\end{equation*}
For all $s\geq 0$, let  
\begin{gather*}
    u(s)=e^{\frac{1}{n}\lambda_{\min}(K_{d}(\bm{X},\bm{X}))s}\| \bm{u}^{m}(s)-\bm{u}^{\NTK}(s)\|_{2},\\
    \alpha(s)=\frac{1}{n}\int_{0}^{s}e^{\frac{1}{n}\lambda_{\min}(K_{d}(\bm{X},\bm{X}))s'}\| K_{\bm{\theta}(s')}^{m}(\bm{X},\bm{X})-K_{d}(\bm{X},\bm{X})\|_{2}\| \bm{u}^{\NTK}(s')\|_{2}\dx s',\\
    \text{~and~}\beta(s)=\frac{1}{n}\|K_{\bm{\theta}(s)}^{m}(\bm{X},\bm{X})-K_{d}(\bm{X},\bm{X})\|_{2}.
\end{gather*}
Notice that $\beta(\cdot)$ is non-negative, $\alpha(\cdot)$ is non-decreasing and $u(\cdot)$ satisfies that $u(s)\leq \alpha(s)+\int_{0}^{s}\beta(s')u(s')\dx s'$ for all $s\geq 0$. Applying the Grönwall's inequality \cite{walter1970differential} yields $u(s) \leq \alpha(s)e^{\int_{0}^{s}\beta(s')\dx s'}$ for all $s\geq 0$. Then we have
\begin{equation*}
    \begin{aligned}
    & \quad \|\bm{u}^{m}(t)-\bm{u}^{\NTK}(t)\|_{2}\\
    & \leq \frac{1}{n}\int_{0}^{t}e^{\frac{1}{n}\lambda_{\min}(K_{d}(\bm{X},\bm{X}))(s-t)}\|K_{\bm{\theta}(s)}^{m}(\bm{X},\bm{X})-K_{d}(\bm{X},\bm{X})\|_{2}\| \bm{u}^{\NTK}(s)\|_{2}\dx s\\
    & \quad \cdot e^{\frac{1}{n}\int_{0}^{t} \| K_{\bm{\theta}(s)}^{m}(\bm{X},\bm{X})-K_{d}(\bm{X},\bm{X})\|_{2}\dx s} \\
    & \leq \frac{1}{n}\cdot n\Delta\cdot e^{-\frac{1}{n}\lambda_{\min}(K_{d}(\bm{X},\bm{X}))t}\|\bm{y}\|_{2}\cdot t\cdot e^{\frac{1}{n}\cdot n\Delta\cdot t} = \|\bm{y}\|_{2}\Delta t e^{-\frac{1}{n}(\lambda_{\min}(K_{d}(\bm{X},\bm{X}))-n\Delta)t},
    \end{aligned}
\end{equation*}
where Lemma \ref{lem: u_NTK} is applied in the second inequality.
\end{proof}

\section{Proof of Theorem \ref{thm:nn:early:stopping:d=1}}\label{app:proof:nn:es}

We recollect some essentials of spectral algorithms for the convenience of the readers here. For a thorough introduction on spectral algorithms, we refer the interested readers to \cite{lin2020optimal} and references therein. For simplicity, let $\mathcal{X} \subseteq \mathbb{R}^{d}$ be a compact set and $K:\mathcal{\mathcal{X}}\times\mathcal{X}$ a kernel function which is continuous and measurable. Denote the RKHS of $K$ by $\mathcal{H}$. Assume that the kernel $K$ satisfies that $\sup_{\bm{x} \in \mathcal{X}}K(\bm{x},\bm{x}) \leq \kappa^{2}$ for some $\kappa<\infty$. Notice $K_{d}$ satisfies all the assumptions above. Let $K_{\bm{x}}:\mathbb{R} \to \mathcal{H}$ be the mapping $x \mapsto xK(\cdot,\bm{x})$ and the adjoint operator $K_{\bm{x}}^{*}:\mathcal{H} \to \mathbb{R}$ such that $K_{\bm{x}}^{*}f \mapsto f(\bm{x})$. We further introduce $T_{\bm{x}} = K_{\bm{x}}K_{\bm{x}}^{*}: \mathcal{H} \to \mathcal{H}$ and $T_{\bm{X}}=\frac{1}{n}\sum_{i=1}^{n}T_{\bm{x}_{i}}$.

\begin{definition}[Filter functions]
    Let $\Lambda$ be a subset of $[0,\infty) \cup \{\infty\}$ and define $\infty^{-1}=0$. The functions $\{\mathcal{G}_{\lambda}:[0,\kappa^{2}] \to [0,\infty), \lambda \in \Lambda \}$ are the filter functions with qualification $\tau \geq 1$ if there exists absolute constant $E$ and constant $F_{\tau}$ depending on $\tau$ such that
    \begin{equation*}
        \sup_{\lambda \in \Lambda}\sup_{\alpha \in [0,1]}\sup_{u \in [0,\kappa^{2}]}u^{\alpha}\mathcal{G}_{\lambda}(u) \lambda^{1-\alpha} \leq E
    \end{equation*}
    and
    \begin{equation*}
        \sup_{\lambda \in \Lambda}\sup_{\alpha \in [0,\tau]}\sup_{u \in [0,\kappa^{2}]}u^{\alpha}|1-u\mathcal{G}_{\lambda}(u)| \lambda^{-\alpha} \leq F_{\tau}.
    \end{equation*}
\end{definition}

\begin{definition}[Spectral algorithms]
    Given the filter functions $\mathcal{G}_{\lambda}(u)$, define 
    \begin{equation*}
        \mathcal{G}_{\lambda}(T_{\bm{X}})=\sum_{j=1}^{\infty}\mathcal{G}_{\lambda}(\hat{\lambda}_{j})\langle \cdot , \hat{e}_{j}\rangle_{\mathcal{H}}\hat{e}_{j},
    \end{equation*}
    where $\{\hat{\lambda}_{j},j \geq 1\}$ and $\{\hat{e}_{j},j \geq 1\}$ are eigenvalues and eigenfunctions of $T_{\bm{X}}$. The estimator $\hat{f}_{\lambda}$ reads as follows,
    \begin{equation*}
        \hat{f}_{\lambda} = \mathcal{G}_{\lambda}(T_{\bm{X}})\frac{1}{n}\sum_{i=1}^{n}y_{i}K(\cdot,\bm{x}_{i}).
    \end{equation*}
\end{definition}

\begin{lemma}
    The filter function corresponding to gradient flow is 
    \begin{equation*}
        \mathcal{G}_{\lambda}(u) = (1-e^{-u/\lambda})/u,   
    \end{equation*}
    where $\Lambda=(0,\infty)\cup\{\infty\}$, $\tau$ could be any real number which is greater than or equal to $1$, $E=1$, $F_{\tau}=(\tau/e)^{\tau}$.
\end{lemma}

\begin{proof}
    By the definition of gradient flow, 
    \begin{equation*}
        \dot{f}=-\frac{\partial \emprisk}{\partial f}=-\frac{1}{n}\sum_{i=1}^{n}-K_{\bm{x}_{i}}(y_{i}-f(\bm{x}_{i}))=-(T_{\bm{X}}f-\frac{1}{n}\sum_{i=1}^{n}y_{i}K(\cdot,\bm{x}_{i})),
    \end{equation*}
    where
    $\emprisk(f)=\frac{1}{2n}\sum_{i=1}^{n}(y_{i}-f(\bm{x}_{i}))^{2}$ and $f_{0}=0 \in \mathcal{H}$. Hence 
    \begin{equation*}
        f_{t}=T_{\bm{X}}^{-1}(I-e^{-tT_{\bm{X}}})\frac{1}{n}\sum_{i=1}^{n}y_{i}K(\cdot,\bm{x}_{i}),
    \end{equation*}
    where $T_{\bm{X}}^{-1}$ is the Moore-Penrose inverse of $T_{\bm{X}}$. So $\mathcal{G}_{\lambda}(u)=(1-e^{-u/\lambda})/u$, where we parameterize $t=1/\lambda$. It could be verified that $\mathcal{G}_{\lambda}(u)=\sum_{j=1}^{\infty}\frac{(-1)^{j-1}(1/\lambda)^{j}}{j!}u^{j-1}$ is continuous for all $u \geq 0$ so $\mathcal{G}_{\lambda}(\cdot)$ can be applied to $T_{\bm{X}}$ by Theorem 5.1.11 from \cite{simon2015operator}. It could also be checked that $u^{\alpha}\mathcal{G}_{\lambda}(u)\lambda^{1-\alpha} \leq 1$ by the fact that $1-e^{-x} \leq x$ for all $x \geq 0$ and $u^{\alpha}|1-u\mathcal{G}_{\lambda}(u)|\lambda^{-\alpha} \leq (\alpha/e)^{\alpha}$.
\end{proof}

\begin{proposition}[Corollary 4.4 in \cite{lin2020optimal}]\label{prop:early:stopping}
    Suppose Assumption \ref{assump:f_star} holds and we observed $n$ i.i.d. samples $\{(\bm{x}_{i},y_{i}),i\in[n]\}$ from the model \eqref{equation:true_model}. For any given $\delta\in(0,1)$, if the training process is stopped at $t_{\star} \propto n^{2/3}$ for the NTK regression, then for sufficiently large $n$, there exists a constant $C$ independent of $\delta$ and $n$, such that
    \begin{equation*}
        \excrisk(f_{t_{\star}}^{\NTK})=Cn^{-\frac{2}{3}}\log^{2}\frac{6}{\delta}
    \end{equation*}
    holds with probability at least $1-\delta$ over the training data.
\end{proposition}

Setting $\epsilon=Cn^{-\frac{2}{3}}\log^{2}\frac{6}{\delta}$ in Theorem \ref{thm:risk:approx} yields
\begin{equation*}
    \begin{aligned}
    \excrisk(f_{\bm{\theta}(t_{\star})}^{m}) &\leq |\excrisk(f_{\bm{\theta}(t_{\star})}^{m})-\excrisk(f_{t_{\star}}^{\NTK})|+\excrisk(f_{t_{\star}}^{\NTK}) \\
    &\leq 2Cn^{-\frac{2}{3}}\log^{2}\frac{6}{\delta}.
    \end{aligned}
\end{equation*}

\section{Proof of Theorem \ref{thm:bad_gen}}

\begin{proof}[Proof of Theorem \ref{thm:bad_gen} $i)$]
Theorem \ref{thm:bad_gen} $i)$ is a direct corollary of the following lemmas and Proposition \ref{prop:funct:approx}.
\begin{lemma}[Overfitted NTK model can be approximated by the linear interpolation]\label{LI}
    Suppose that we have observed $n$ data $\{(x_{i},y_{i}), i\in [n]\}$ from the model \eqref{equation:true_model} and $x_i=\frac{i-1}{n-1}$, $i\in [n]$. With the probability at least $1-C_1/n$, the overfitted NTK model with zero initialization $f^{\NTK}_{\infty}(x) =K_1(x,\bm{X})K_1^{-1}(\bm{X},\bm{X})\bm{y}$ can be approximated by the linear interpolation, i.e.,
    \begin{align}
     \sup_{x\in [0,1]}|f_{\infty}^{\NTK}(x)-f_{\LI}(x)|\leq C_2\sqrt{\log n} /(n-1)^{2}
    \end{align}
     for some absolute constants $C_1$ $C_2$.
\end{lemma}

\begin{lemma}\label{lemma:NTK_t_vs_NTK_inf}
   Suppose that we have observed $n$ data $\{(x_{i},y_{i}), i\in [n]\}$ from the model \eqref{equation:true_model} and $x_i=\frac{i-1}{n-1}$, $i\in [n]$. If $t\geq C_1 n^2 \log  n$, we have 
\begin{equation}
  \sup_{x\in [0,1]}|f_{t}^{\NTK}(x)-f_{\infty}^{\NTK}(x)|\leq \frac{C_2}{(n-1)^3}
\end{equation}
for some absolute constants $C_1$, $C_2$.
\end{lemma}

\end{proof}

\begin{proof}[Proof of lemma \ref{LI}]
    Since 
    \begin{equation*}
         y_{i} = K(x_i,\bm{X})K^{-1}\bm{y}
    \end{equation*}
and 
\begin{equation*}
         y_{i+1} = K(x_{i+1},\bm{X})K^{-1}\bm{y},
    \end{equation*}
the Taylor Expansion and intermediate theorem imply that for $\forall x\in(x_{i},x_{i+1})$, there are $\xi_{i}$ and $\hat{\xi}_{i}\in (x_{i},x_{i+1})$ such that 
\begin{align}\label{equation:taylor_expansion}
 K(x,\bm{X})K^{-1}\bm{y} -y_{i}&= (x-x_i)K_{+}'(x_i,\bm{X})K^{-1}\bm{y} + \frac{(x-x_i)^2}{2}K''(\xi_i,\bm{X}) K^{-1}\bm{y},\\
     y_{i+1}-y_{i}
    &=(x_{i+1}-x_i)K_{+}'(x_i,\bm{X})K^{-1}\bm{y} + \frac{(x_{i+1}-x_i)^2}{2}K''(\hat{\xi}_i,\bm{X})K^{-1}\bm{y}
\end{align}
 where 
$K'_{+}(x_{i},\bm{X})=\lim_{x\rightarrow x_{i}^{+}} K'_{+}(x,\bm{X}) = \lim_{x\rightarrow x_{i}^{+}}\frac{\partial K(x,\bm{X})}{\partial x}$. Thus,
\begin{equation}
\begin{aligned}\label{eq:higher:order}
    K(x,\bm{X})K^{-1}\bm{y}& - y_i - \frac{(x-x_i)}{x_{i+1}-x_i} (y_{i+1}-y_{i}) \\
    &=-\frac{(x-x_i)}{x_{i+1}-x_i}\frac{(x_{i+1}-x_i)^2}{2}K''(\hat{\xi}_i,\bm{X})K^{-1}\bm{y} \\
    &+\frac{(x-x_i)^2}{2}K''(\xi_i,\bm{X}) K^{-1}\bm{y}.
\end{aligned}
\end{equation}

The second-order derivative can be bounded by the following lemma: 
 \begin{lemma}[Bounded second order derivative of overfitted NTK regression]\label{prop:bound_second_derivative}
   Suppose that we have observed $n$ data $\{(x_{i},y_{i}), i\in [n]\}$ from the model \eqref{equation:true_model} and $x_i=\frac{i-1}{n-1}$, $i\in [n]$. With the probability at least $1-\frac{2}{n}$, we have
    \begin{align}
     \sup_{x\in (x_i,x_{i+1})}|K_1''(x,\bm{X})K_1^{-1}(\bm{X},\bm{X})\bm{y}|\leq C\sqrt{\log n} 
    \end{align}
     for $\forall i\in[n]$ and for some absolute constant $C$.
    \end{lemma}

By Lemma \ref{prop:bound_second_derivative}, $|K''(\xi_i,\bm{X}) K^{-1}\bm{y}|<C\log n$ for some constant $C$ and the RHS of  \eqref{eq:higher:order}  is bounded by $ \frac{C\sqrt{\log n} }{(n-1)^{2}}$ for some constant $C$.
\end{proof}

\begin{proof}[Proof of Lemma \ref{lemma:NTK_t_vs_NTK_inf}]
  $K(x,x')$ is bounded by a constant By Lemma \ref{lem:bound_y}, $|y_i|$ are also bounded by $C\sqrt{\log n}$ with the probability at least $1-2/n$. By Theorem \ref{thm:spectral:d=1:L=1}, $\lambda_{\min}=C_1/n$ for some constant $C_1$. If $t\geq C_2 n^2 \log(n^6)=C_3n^2\log n$ for some constants $C_2$, $C_3$, we have
  \begin{equation*}
    \begin{split}
      |f_{t}^{\NTK}(x)-f_{\infty}^{\NTK}(x)|&= e^{-\frac{tK_1(\bm{X},\bm{X})}{n}}|K_1(x,\bm{X})K_1^{-1}(\bm{X},\bm{X})\bm{y}|\\
      &\leq e^{-\frac{t \lambda_{\min}}{n} }  \lambda_{\min}^{-1}n\sqrt{\log n} \\
       &\leq \frac{C_4}{(n-1)^3}
    \end{split}
  \end{equation*}
  for some constant $C_4$.
\end{proof}

\begin{proof}[Proof of Theorem \ref{thm:bad_gen} $ii)$]

Theorem \ref{thm:bad_gen} $ii)$ is a direct corollary of the following lemma and Theorem \ref{thm:bad_gen} $i)$.


\begin{lemma}[Linear Interpolation cannot Generalize Well]\label{LI_not_good}
    Suppose that we have observed $n$ data $\{(x_{i},y_{i}), i\in [n]\}$ from the model \eqref{equation:true_model} and $x_i=\frac{i-1}{n-1}$, $i\in [n]$. Let $f_{\LI}$ be a linear interpolation estimator. Then there exists a positive constant $C$ such that 
    \begin{equation}
        \excrisk( f_{\LI}(x))  \geq \frac{1}{3}\sigma^2.
    \end{equation}
holds with probability at least $1-\frac{C}{n}$.
 \end{lemma}

\end{proof}

\begin{proof}[Proof of Lemma \ref{LI_not_good}]
    For $x\in[x_i,x_{i+1}]$, the linear interpolation takes the form
    \[
    f_{\LI}(x)= \lambda_i(x)y_{i} + (1-\lambda_i(x))y_{i+1},
    \]
    where $\lambda_i(x)=\frac{x_{i+1}-x}{x_{i+1}-x_{i}}$.
    
    Denote $b^2(x)=(\Expc_{\bm{\varepsilon}}f_{\LI}(x)-f_{\star}(x))^{2}$ and $\sigma^{2}(x)=(f_{\LI}(x)-\Expc_{\bm{\varepsilon}}f_{\LI}(x))^2$ to be the bias and variance term respectively, where we denote by $\Expc_{\bm{\varepsilon}}$ taking expectations with respect to the noise $\varepsilon_{1},\dots,\varepsilon_{n}$. Thus the excess risk of the linear interpolation $\excrisk(f_{\LI})$ can be formulated as 
\begin{align}
    \excrisk(f_{\LI}) &=\int_{0}^{1} (b(x)+\sigma(x))^2 \dx x \\
    &= c_1 + \sum_{i=1}^{n-1} \frac{c_{2,i} }{n-1} (\varepsilon_i + \varepsilon_{i+1})  + \frac{1}{3(n-1)}\sum_{i=1}^{n-1} (\varepsilon_{i}^{2}  + \varepsilon_{i+1}^{2} + \varepsilon_{i}\varepsilon_{i+1}).
\end{align}
for some positive constant $c_1$ and a uniformly bounded sequence $\{c_{2,i}\}_{i=1}^{n-1}$. The last equation is the result of Lemma \ref{lem:risk_LI_fomular}. The expectation of $\excrisk(f_{\LI})$
\begin{align}
     \Expc_{\bm{\varepsilon}}\excrisk(f_{\LI}) =c_1 + \frac{2}{3}\sigma^2 \geq\frac{2}{3}\sigma^2
\end{align}
and the variance of $\excrisk(f_{\LI})$
\begin{align}
    \operatorname{Var}_{\bm{\varepsilon}}(\excrisk(f_{\LI})) \leq \frac{c_3}{n}
\end{align}
for some constant $c_3$. By Chebyshev's inequality, we have 
\begin{align}
    \Prob (|\excrisk(f_{\LI})-\Expc_{\bm{\varepsilon}}(\excrisk(f_{\LI}))|\geq \frac{1}{3}\sigma^2) \leq  \frac{\operatorname{Var}_{\bm{\varepsilon}}(\excrisk(f_{\LI}))}{\frac{1}{9}\sigma^4}.
\end{align}
Thus, we conclude that with probability at least $1-\frac{c_4}{n}$, 
\begin{align}
    \excrisk(f_{\LI}) \geq \frac{1}{3}\sigma^2.
\end{align}
for some constant $c_4$.
\end{proof}

\begin{lemma}\label{lem:risk_LI_fomular}
Denote $b^2(x)=(\Expc_{\bm{\varepsilon}}f_{\LI}(x)-f_{\star}(x))^{2}$ and $\sigma^{2}(x)=(f_{\LI}(x)-\Expc_{\bm{\varepsilon}}f_{\LI}(x))^2$. $ \excrisk(f_{\LI})= \int_{0}^{1} (b(x)+\sigma(x))^2 dx$ can be reformulated as 
\begin{align}
\excrisk(f_{\LI})= c_1 + \sum_{i=1}^{n-1} \frac{c_{2,i} }{n-1} (\varepsilon_i + \varepsilon_{i+1}) + \frac{1}{3(n-1)}\sum_{i=1}^{n-1} (\varepsilon_{i}^{2}  + \varepsilon_{i+1}^{2} + \varepsilon_{i}\varepsilon_{i+1})
\end{align}
for some positive constant $c_1$ and a uniformly bounded sequence $\{c_{2, i}\}_{i=1}^{n}$.
\end{lemma}

\begin{proof}[Proof of Lemma \ref{lem:risk_LI_fomular}]
Denote 
\begin{align}
    b_i(x) &= (1-\frac{x-x_i}{x_{i+1}-x_i})(f_{\star}(x_i) - f_{\star}(x)) + \frac{x-x_i}{x_{i+1}-x_i}(f_{\star}(x_{i+1}) - f_{\star}(x))\\
    \sigma_i(x) &= (1-\frac{x-x_i}{x_{i+1}-x_i}) \varepsilon_i + \frac{x-x_i}{x_{i+1}-x_i} \varepsilon_{i+1}.
\end{align}

\begin{align}
    \excrisk(f_{\LI})&=\int_{0}^{1} b^2(x)+2b(x)\sigma(x) +\sigma^2(x) \dx x\\
        &=\sum_{i=1}^{n-1}\int_{x_i}^{x_{i+1}}  b^{2}_i(x) + 2b_i(x)\sigma_i(x) +\sigma^2_i(x) \dx x
\end{align}
Since $f_{\star}$ is bounded, $b_i(x)\in (C_1, C_2)$ where $C_1$ and $C_2$ depend on $f_{\star}$. Thus, by the mean value theorems, there exists a positive constant $c_1$ and a uniformly bounded sequence $\{c_{2, i}\}_{i=1}^{n}$ such that
\begin{align}
\excrisk(f_{\LI})= c_1 + \sum_{i=1}^{n-1} \frac{c_{2,i} }{n-1} (\varepsilon_i + \varepsilon_{i+1}) + \frac{1}{3(n-1)}\sum_{i=1}^{n-1} (\varepsilon_{i}^{2}  + \varepsilon_{i+1}^{2} + \varepsilon_{i}\varepsilon_{i+1}).
\end{align}
\end{proof}

\subsection{Technical Lemmas}\label{subsection:tech_lemma} 
In the following content, to simplify the notations, denote $K=K(\bm{X},\bm{X})$, $G=G_1(\bm{X},\bm{X})$ and $\Pi=2\Pi_1(\bm{X},\bm{X})$.

\begin{proof}[Proof of lemma \ref{prop:bound_second_derivative}]

Let $\xi \in (x_i,x_{i+1}). $We only present the proof for $2\leq i \leq n-2$. When $i\in\{1,n-1\}$, one can prove the statement in a similar way.

Let $e_{k}$ be the $k$-th vector in the standard basis of $\mathbb{R}^{n}$. Denote $I_{k}=e_{k}e_{k}^{\top}$. Let us consider the rank one decomposition of $\Pi$:
\begin{equation}\label{eqn:decomposition:reoder}
    \begin{aligned}
    \Pi&=\sum_{k\not \in \{1,i,i+1,n\}}I_{k}\Pi+I_{i}\Pi+I_{i+1}\Pi+I_{1}\Pi+I_{n}\Pi\\
    &=\underbrace{I_{2}\Pi}_{\triangleq \Pi_{1}}+\cdots+\underbrace{I_{i-1}\Pi}_{\triangleq \Pi_{i-2}}+\underbrace{I_{i+2}\Pi}_{\triangleq \Pi_{i-1}}+\cdots+\underbrace{I_{n-1}\Pi}_{\triangleq\Pi_{n-4}}+\underbrace{I_{i}\Pi}_{\triangleq\Pi_{n-3}}+\underbrace{I_{i+1}\Pi}_{\triangleq\Pi_{n-2}}+\underbrace{I_{1}\Pi}_{\triangleq\Pi_{n-1}}+\underbrace{I_{n}\Pi}_{\triangleq\Pi_{n}}.
\end{aligned}
\end{equation}
We denote by $S$ (resp. $S^{-1}$ ) the transformation (resp. inverse transform) between the indices appeared in \eqref{eqn:decomposition:reoder}, i.e., $S(1)=2$, $S(2)=3$, \ldots, $S(i-2)=i-1$, $S(i-1)=i+2$,\ldots, $S(n-4)=n-1$, $S(n-3)=i$, $S(n-2)=i+1$, $S(n-1)=1$, and $S(n)=n$. It is clear that $\Pi_{k}=I_{S(k)}\Pi$.

Let $D_{k}=G+\Pi_{1}+...+\Pi_{k-1}$, $k=1,2,\ldots, n$.  It is clear that $D_{1}=G$ and $D_{n+1}=G+\Pi=K$. 
 To proceed with the proof, we need the following lemma:

\begin{lemma} \label{lemma:B_C_k}
Suppose that $n> 22$. Let $\Gamma=\operatorname{diag}\{1,-(n-1),\cdots,-(n-1),1\}$ be an $n\times n$ diagonal matrix. There exists a constant $C$ such that the following statements hold.
\begin{enumerate}
    \item For any $p \in [n]$, $D_{p}$ is an invertible matrix and $g_p=(1+\operatorname{tr}(\Pi_{k}D_{p}^{-1}))^{-1}\in(0,C]$.
    \item Let $H^{(p)}=\Pi D_{p}^{-1}$ and $C^{(p)}=\Gamma H^{(p)}$.  Then for any $a\in [n]$, we have $|C^{(p)}_{i,j}|\leq C$ for any $i,j \in [n]$.
\end{enumerate} 
\end{lemma}

 We remind that $\Pi$ is an invertible matrix (please see Lemma \ref{lem: Pi_positive_definite}). Thus, if $D_{k}$'s are invertible, then $H^{(k)}$'s are invertible. \\
 
\noindent $\bullet$ $\underline{p=1}$; \quad Since $K''(\xi_{i},\bm{X}) = \frac{4}{\pi((1+\xi_i^2)^2)}(|\xi_i-x_1|,\dots,|\xi_i-\bm{X}|)$, we can easily verify that there exists a constant $C$ such that
\begin{equation}
\left|\left(K''(\xi_{i},\bm{X})D_{1}^{-1}\right)_{j}\right| =\left|\left(K''(\xi,\bm{X})G^{-1}\right)_{j}\right| \leq 
\begin{cases}
\quad      C &   j\in\{1,i,i+1,n\},\\[6pt]
\quad      0  &  j\not \in \{ 1,i,i+1,n\}.
\end{cases}
\end{equation}
In other words, $K''G^{-1}=K''D_{1}^{-1}$ is a row vector with at most 4 non-zero entries located in $\{1,i,i+1,n\}$.

\vspace{ 3mm }
\noindent $\bullet$ $\underline{p\in \{2,\ldots,n-4\}}$; \quad Thanks to the Lemma \ref{lemma:B_C_k}, $D_{k}$'s are invertible matrices. Thus, the Sherman–Morrison formula gives us that for any $p=1,2,\ldots, n$, 
\begin{equation}\label{eqn:ess:recursive}
\begin{aligned}
    K''D_{p+1}^{-1}&=K''(D_{p}+\Pi_{p})^{-1}
    =K''D_{p}^{-1}-g_{p}K''D_{p}^{-1}\Pi_{p}D_{p}^{-1}.
\end{aligned}
\end{equation}

Because that for any $p\in \{1,2,\dots,n-4\}$, $S(p)\not \in \{1,i,i+1,n\}$, we know from  the definition of $\Pi_{p}$ that for $p\in\{1,2,\ldots,n-4\}$,  $K''D_{p}^{-1}\Pi_{p}=0$ and $K''D_{p+1}^{-1}=K''D_{p}^{-1}$.  In particular, we know that
\begin{align}
    K''D_{n-3}^{-1}=K''D_{n-4}^{-1}=\cdots=K''D_{1}^{-1}=K''G^{-1}.
\end{align}
In other words, $K''D_{p}^{-1},  p=2,3,\dots,n-4$ are row vectors with at most 4 non-zero entries located in $\{1,i,i+1,n\}$.

\vspace{3mm}
\noindent $\bullet$ \underline{ $p\in \{n-3,n-2,n-1,n\}$;}\quad  Since $g_{p}K''D_{p}^{-1}\Pi_{p}D_{p}^{-1}$ is no longer 0 for $p\geq n-3$, we do not have $K''D_{n-2}^{-1}=K''D^{-1}_{n-3}$ anymore. We need to treat them separately. Again, the Sherman–Morrison formula gives us that
\begin{equation}
\begin{aligned}
    K''D_{n-2}^{-1}&=K''(D_{n-3}+\Pi_{n-3})^{-1}
    =K''D_{n-3}^{-1}-g_{n-3}K''D_{n-3}^{-1}\Pi_{n-3}D_{n-3}^{-1}.
\end{aligned}
\end{equation}
Thus, there exists an absolute constant $C$, such that
\begin{align*}
    |(K''&(\xi_i,\bm{X})D^{-1}_{n-2})_{j}| \leq |(K''(\xi_i,\bm{X})D^{-1}_{n-3})_j| + g_{n-3} |(K''(\xi_i,\bm{X})D^{-1}_{n-3} \Pi_{n-3} D^{-1}_{n-3})_j|\\
    &=|(K''(\xi_i,\bm{X})D^{-1}_{n-3})_j| +g_{n-3}|(K''(\xi_i,\bm{X})D^{-1}_{1} \Pi_{n-3} D^{-1}_{n-3})_j| \leq 
\begin{cases}
     C, &   j\in\{1,i,i+1,n\}\\[6pt]
     C\frac{1}{n-1} &  j\neq \{1,i,i+1,n\}
\end{cases}
\end{align*}
where the last inequality follows from 
the Lemma \ref{lemma:B_C_k} and $\Pi_{n-1}D^{-1}_{n-3}=\frac{1}{n-1}I_{S(n-1)}C^{(n-3)}$. We can prove  the results for $p=n-2,n-1,n$ in a similar way. 

In other words, we have shown that there exists an absolute constant $C$ such that
\begin{align}
|(K''(\xi_i,\bm{X})K^{-1})_{j}| = |(K''(\xi_i,\bm{X})D^{-1}_{n+1})_{j}|\leq
\begin{cases}
     C, &   j=1,i,i+1,n\\[6pt]
     C\frac{1}{n-1} &  j\neq 1,i,i+1,n
\end{cases}.
\end{align}

Denote $y_{\max} =  \max_{i\in[n]} y_i $. Then $|K''(\xi_i,\bm{X})K^{-1}\bm{y}| \leq C |y_{\max}|$.

By Lemma \ref{lem:bound_y}, we have $|y_{\max}|\leq C\sqrt{\log n}$ with probability $1-\frac{2}{n}$. Thus,
\begin{equation}
    |K''(\xi_i,\bm{X})K^{-1}\bm{y}|\leq C \sqrt{\log n}
\end{equation}
with probability $1-\frac{2}{n}$ and for some constant $C$.
\end{proof}

\begin{proof}[Proof of Lemma \ref{lemma:B_C_k}]
    We prove the statements through induction on $k$.

\vspace{3mm}
\noindent $\bullet$ \underline{$ k=1$;}\quad   It is clear that $D_{1}=G$ is invertible.
The second statement follows the following lemma:
\begin{lemma}\label{lemma:G_inv_property}
 There exists an absolute constant $C$ such that
  \begin{align*}
  (\Pi G^{-1})_{i,j}\in \begin{cases}
    \quad (-0.54 -\frac{1}{n-1},  C] &\quad i\in [n], j=1 \\[4pt]
    \quad (0,  C] &\quad i\in [n],j=n \\[4pt]
    \quad [- 2\frac{|i-j|+1}{(n-1)^2}, 0], &\quad i\in [n], j\neq 1,n.
  \end{cases}
\end{align*}
Specifically,  $ (\Pi G^{-1})_{1,1} > 1.2-\frac{1}{n-1}$.
\end{lemma}
Moreover, the above lemma also shows that $g_{1}=(1+\operatorname{tr}(\Pi_{1}D_{1}^{-1}))^{-1}>0$ is bounded. Thus, we proved the Lemma \ref{lemma:B_C_k} for $k=1$.

\vspace{3mm}
\noindent $\bullet$ \underline{$k>1$;} \quad Suppose that the inductive hypotheses hold for any $1\leq k'\leq k-1$. 

\vspace{3mm}
Since $g_{k-1}=(1+\operatorname{tr}(\Pi_{k-1}D_{k-1}^{-1}))^{-1}\in (0,C]$, Sherman–Morrison formula implies that $D_{k}=D_{k-1}+\Pi_{k-1}$ is invertible.
Thus, we have 
\begin{align}
    C^{(k)}&=\Gamma\Pi D_{k}^{-1}=C^{(k-1)}-g_{k-1}C^{(k-1)}\Pi_{k-1}D_{k-1}^{-1}\\
    &=C^{(k-1)}-g_{k-1}C^{(k-1)}I_{S(k-1)}\Gamma^{-1}C^{(k-1)}.
\end{align}
Since both $\Pi$, $\Gamma$ and $D_{k}$ are invertible matrices, we know that $C^{(k)}$ is invertible and the Sherman-Morrison formula gives us
\begin{align}
    (C^{(k)})^{-1}=(C^{(k-1)})^{-1}-\frac{I_{S(k-1)}}{n-1}. 
\end{align}
 The desired bound about $C^{(k)}$ is provided by the following lemma: 
\begin{lemma}\label{lem:bound:C}
Assume that $k\leq n-1$ and $C^{(j)}, j=1,2,\dots, k$ are invertible matrices. There exists an absolute constant $C$ such that,
\begin{align}
  |C^{(k)}_{i,j}|\leq 21, \mbox{~if~} 2\leq i,j\leq n-1 \mbox{~and~} |C^{(k)}_{i,j}|\leq C, \mbox{ ~if~ } i \mbox { or } j \in \{1,n\} . 
\end{align}
\end{lemma}

First, Lemma \ref{lem:bound:C} implies that the second statement in Lemma \ref{lemma:B_C_k} hold for $k$. Second, 
since $n\geq 23$, Lemma \ref{lem:bound:C} implies that $C^{(k)}_{S(k),S(k)} \leq  21<n-1$. Thus, for any  constant $C>22$, we have   $g_k=\left(1-\frac{C^{(k)}_{S(k),S(k)}}{n-1}\right)^{-1} \in (0,C]$.

\end{proof}

\vspace{3mm}
\noindent{\bf Lemma \ref{lemma:G_inv_property}} {
There exists an absolute constant $C$ such that
  \begin{align*}
  (\Pi G^{-1})_{i,j}\in \begin{cases}
    \quad (-0.54 -\frac{1}{n-1},  C] &\quad i\in [n], j=1 \\[4pt]
    \quad (0,  C] &\quad i\in [n],j=n \\[4pt]
    \quad [- 2\frac{|i-j|+1}{(n-1)^2}, 0], &\quad i\in [n], j\neq 1,n.
  \end{cases}
\end{align*}
Moreover, we can prove that  $ (\Pi G^{-1})_{1,1} > 1.2-\frac{1}{n-1}$.
}

\begin{proof} 
Since $\Pi(x,y)$ are continuous differentiable of 2nd order, the Taylor expansion gives us that for any $i$, there exist $\xi_{i}$ and $\xi_{i}'\in [x_{i},x_{i+1}]$ such that
\begin{align}
&\Pi(x.x_{i+1})-\Pi(x,x_{i})=\Pi'(x,\xi_{i})(x_{i+1}-x_{i}), \\
 &   \Pi(x.x_{i+1})-\Pi(x,x_{i})=\Pi'(x,x_{i})(x_{i+1}-x_{i})+\frac{1}{2}\Pi''(x,\xi_{i}')(x_{i+1}-x_{i})^{2}.
\end{align}

\noindent $\bullet$ \underline{ $j=1$;} For any $i\in [n]$, 
\begin{equation}\label{equation:Pi_G_i_1}
  \begin{split}
    (\Pi G^{-1})_{i,1} 
    &= \frac{\pi}{2}(n-1)(\Pi(x_i,x_1)-\Pi(x_i,x_2)) + \frac{\pi}{2}\frac{\Pi(x_i,x_1)+\Pi(x_i,x_n)}{2\pi-1} \\
    &= \frac{\pi}{2}\Pi^{'}(x_i,\xi_1) + \frac{\pi}{2}\frac{\Pi(x_i,x_1)+\Pi(x_i,x_n)}{2\pi-1}\\
    &= - \xi_1\frac{|x_i-\xi_1|}{1+\xi_1^2} - x_i(\pi -\psi(\xi_1, x_i)) + \frac{\pi}{2}\frac{\Pi(x_i,x_1)+\Pi(x_i,x_n)}{2\pi-1}
  \end{split}
\end{equation}
Since $x_i\in[0,1]$, It is clear that there exists a constant $C$ such that  $(\Pi G^{-1})_{i,1}\leq C$.
On the other hand, 
\begin{align*}
        (\Pi G^{-1})_{i,1}&\geq -x_2 -  x_i(\pi -\psi(x_1, x_i)) + \frac{\pi}{2}\frac{\Pi(x_i,x_1)+\Pi(x_i,x_n)}{2\pi-1}\\
    &= -\frac{1}{n-1} - \frac{i-1}{n-1}(\pi-\psi(x_1,x_i)) + \frac{\pi-\psi(x_1,x_i) + (1+x_i)(\pi-\psi_(x_n,x_i))+1}{2\pi-1}\\
    &\geq -0.54-\frac{1}{n-1}
\end{align*}
Finally, we have $ (\Pi G^{-1})_{1,1} \geq -\frac{1}{n-1} + \frac{2\pi-\frac{\pi}{4}+1}{2\pi-1} >1.2-\frac{1}{n-1}$.

\vspace{3mm}
\noindent $\bullet$ \underline{ $j=n$;}
\begin{equation}\label{equation:Pi_G_i_n}
  \begin{split}
    (\Pi G^{-1})_{i,n} 
    &= \frac{\pi}{2}(n-1)(\Pi(x_i,x_n)-\Pi(x_i,x_{n-1})) + \frac{\pi}{2}\frac{\Pi(x_i,x_1)+\Pi(x_i,x_n)}{2\pi-1} \\
    &= \frac{\pi}{2}\Pi^{'}(x_i,\xi_{n-1}) + \frac{\pi}{2}\frac{\Pi(x_i,x_1)+\Pi(x_i,x_n)}{2\pi-1}.
  \end{split}
\end{equation}
Since $x_i\in[0,1]$, it is clear there exists a constant $C$ such that $(\Pi G^{-1})_{i,n}\leq C$. On the other hand, we have 
\begin{align}
        (\Pi G^{-1})_{i,n}&= \xi_1\frac{|x_i-\xi_{n-1}|}{1+\xi_i^2} + x_i(\pi -\psi(\xi_{n-1}, x_i)) + \frac{\pi}{2}\frac{\Pi(x_i,x_1)+\Pi(x_i,x_n)}{2\pi-1}>0.
\end{align}

\vspace{3mm}
\noindent $\bullet$ \underline{$2\leq j\leq n-1$;}
\begin{equation}\label{equation:Pi_G_i_j}
  \begin{split}
    (\Pi G^{-1})_{i,j} &=\frac{\pi}{2} (n-1)  \left(2\Pi(x_i,x_j)-\Pi(x_i,x_{j+1})-\Pi(x_i,x_{j-1})\right)\\
    &=- \frac{\pi}{2} (n-1) (\Pi^{''}(x,\xi_{j})\frac{(x_{j}-x_{j+1})^2}{2} + \Pi^{''}(x,\xi_{j-1})\frac{(x_{j}-x_{j-1})^2}{2})\\
    &= -\frac{1}{(n-1)} \left(\frac{|x_i-\xi_{j}|}{(1+\xi_{j}^2)^2} + \frac{|x_i-\xi_{j-1}|}{(1+\xi_{j-1}^2)^2} \right)
  \end{split}
\end{equation}
It is clear that $(\Pi G^{-1})_{i,j}\leq 0$.
 One the other hand, 
 since $\xi_j'\in(x_j,x_{j+1})$ and $\xi_{j-1}'\in(x_{j-1},x_{j})$, we have
\begin{align}
        (\Pi G^{-1})_{i,j}&\geq -\frac{1}{(n-1)}\frac{|x_i-\xi_{j}|+|x_i-\xi_{j-1}|}{(1+x_{j-1}^2)^2} 
    \geq -2\frac{|x_i-x_j|+\frac{1}{n-1}}{(n-1)} 
    = - 2\frac{|i-j|+1}{(n-1)^2} .
\end{align}

\end{proof}

\vspace{5mm}
\noindent{\bf Lemma \ref{lem:bound:C} }{
 Assume that $k\leq n-1$ and $C^{(j)}, j=1,2,\dots, k$ are invertible matrices. There exists an absolute constant $C$ such that,
\begin{align}
  |C^{(k)}_{i,j}|\leq 21, \mbox{~if~} 2\leq i,j\leq n-1 \mbox{~and~} |C^{(k)}_{i,j}|\leq C, \mbox{ ~if~ } i \mbox { or } j \in \{1,n\} . 
\end{align}
}

\begin{proof} We prove this lemma by induction on $k$. 

\vspace{3mm}
\noindent $\bullet$ \underline{$ k=1$;}\quad  Recall that Lemma \ref{lemma:G_inv_property} implies that
\begin{align}
 C^{(1)}_{i,j}=
 \begin{cases}
      \quad -(n-1)(\Pi G^{-1})_{i,j} \leq 2\frac{|i-j|+1}{(n-1)}\leq 2 & i\in[n], j\neq 1,n\\[8pt]
      \quad \quad \quad \quad \quad (\Pi G^{-1})_{i,j}\leq C & i\in[n], j=1,n\\
 \end{cases}.
 \end{align}
 Thus the statements hold for $k=1$. 
 
\vspace{3mm}
\noindent $\bullet$ \underline{$ k>1$;}\quad Suppose that the inductive hypotheses hold for $k$. Then
\begin{align}
    (C^{(k+1)})^{-1}=(C^{(k)})^{-1}-\frac{1}{n-1}I_{S_{k}}=(C^{(1)})^{-1}-\frac{1}{n-1}\sum_{j=1}^{k}I_{S(j)}.
\end{align}
Denote  $\frac{1}{n-1}\sum_{j=1}^{k}I_{S(j)}$ by $T_{k}$. Then, we have
\begin{align*}
    C^{(k+1)}= C^{(1)} + C^{(1)}T_{k}C^{(1)} + C^{(1)}\left(T_{k}C^{(1)}\right)^{2}+\cdots=\mathcal{Q} +\mathcal{Q} \left(T_{k}C^{(1)}\right)^{3}+\mathcal{Q}\left(T_{k}C^{(1)}\right)^{6}+\cdots
\end{align*}
 where  $\mathcal{Q}=C^{(1)}\left(1+T_{k}C^{(1)}+\left(T_{k}C^{(1)}\right)^{2}\right)$.
 Simple calculations show that (please see Lemma \ref{lem:simple:computations} ), for any $2\leq i,j\leq n-1$ and $q\in \mathbb{N}$,  we have
\begin{equation}\label{eqn:simple:computations}
\begin{aligned}
   | C^{(1)}_{i,j}|\leq 2,~ 
    \left| \left(C^{(1)} T_{k}C^{(1)}\right)_{i,j}\right| \leq \frac{4}{3}, ~
 \left| \left(C^{(1)} \left(T_{k}C^{(1)}\right)^{2}\right)_{i,j}\right| \leq \frac{4}{5},~ 
  \left|\left( \left(T_{k}C^{(1)}\right)^{3q}\right)_{i,j}\right| \leq \frac{\left(\frac{4}{5}\right)^{q}}{n-1}.
\end{aligned}    
\end{equation}
Note that the first row and last row of $T_{k}$ are zero vectors. Thus, for any $2\leq i,j\leq n-1$, we get 
\begin{align}
 | \mathcal{Q}_{i,j}|= \left|\left(C^{(1)}\left(1+T_{k}C^{(1)}+\left(T_{k}C^{(1)}\right)^{2}\right)\right)_{i,j}\right| \leq 2+\frac{4}{3} + \frac{4}{5}= \frac{62}{15}
\end{align}
and 
\begin{align}
    |C^{(k+1)}_{i,j}|&\leq \sum_{q\geq 0} \left|\left(\mathcal{Q} \left(T_{k}C^{(1)}\right)^{3q}\right)_{i,j}\right|\leq \sum_{q\geq 0}\sum_{p=1}^{n}|\mathcal{Q}|_{i,p}\left|\left(T_{k}C^{(1)}\right)^{3q}_{p,j}\right|\leq \sum_{q\geq 0}\sum_{p=2}^{n-1} \frac{62}{15} \left(\frac{4}{5}\right)^{q}\frac{1}{(n-1)}  < 21
\end{align}

If $i$ or  $j \in \{1,n\}$, we can prove $|C^{(k+1)}_{i,j}| \leq C$ similarly. 

Finally, we can  
show that $\left(C^{(1)}\left( T_{n-2}C^{(1)}\right)^{p}\right)_{1,1}\geq 0$ for $p=0,1,...$ and 
\begin{equation}
C^{(n-1)}_{1,1} = C^{(1)}_{1,1} + \left(C^{(1)} T_{n-2}C^{(1)}\right)_{1,1} +\cdots \geq C^{(1)}_{1,1}\geq 1.2 -\frac{1}{n-1}>0.
\end{equation}

\end{proof}

Under the bounded input $x_i$, we can get the bounded $f^*(x_i)$. Let $y_{max} = \max_{i\in[n]} y_i$. we can have the upper bound of $y_{max}$ through the following lemma:

\begin{lemma}\label{lem:bound_y}
With the definition of Equation \eqref{equation:true_model}, with the probability at least $1-\frac{2}{n}$, we have 
\begin{equation}
    |y_{\max}| \leq C\sqrt{\log n}
\end{equation}
 and for some constant $C$.
\end{lemma}

\begin{proof}[Proof of Lemma \ref{lem:bound_y}]
    Denote $\epsilon_{\max} = \max_{i\in[n]}\epsilon_i$.  For $\forall t>0$, we have
    \begin{equation}
    \exp(t\Expc(\epsilon_{\max}))\leq \Expc(\exp(t \epsilon_{\max} )) \leq \sum_{i=1}^n \Expc ( \exp(t \epsilon_i )) = n \exp(t^2\sigma^2/2).
\end{equation}

The first inequality is Jensen’s inequality, the second is the union bound, and the final equality follows from the definition of the moment-generating function. Taking the logarithm of both sides of this inequality, we have
\begin{equation}
    \Expc(\epsilon_{\max}) \leq \frac{\log n}{t} + \frac{\sigma^2 t}{2}.
\end{equation}
Let $t=\frac{\sqrt{2\log n}}{\sigma}$, we have 
\begin{equation}
    \Expc(\epsilon_{\max}) \leq \sigma \sqrt{2\log n}.
\end{equation}
By Borell-ITS inequality, since $\sigma_{max}^2 = \max_{i\in[n]} E(\epsilon_i^2) = \sigma^2$, for $t\geq 0$, we have 
\begin{equation}
    P(|\epsilon_{\max}-\Expc(\epsilon_{\max})|\geq t)<2exp(-\frac{t^2}{2\sigma_{max}^2}).
\end{equation}

Let $t=\sigma \sqrt{2\log n}$, we have 
\begin{equation}
    P(|\epsilon_{\max}|\leq 2\sigma \sqrt{2\log n} )>1- \frac{2}{n}.
\end{equation}
Since $f_{\star}(x)$ is bounded, we have $\max_{x\in[0,1]}|f_{\star}(x)| \leq C \sqrt{\log n}$ for some constant C and $n>2$. Thus, with the probability at least $1- \frac{2}{n}$, we have 
\begin{equation}
    |y_{\max}|\leq C\sqrt{\log n}
\end{equation}
for some constant $C$.
\end{proof}

\begin{lemma}\label{lemma:K_derivative}
 For any $x,x'\in[0,1]$, the function $\Pi(x^{\prime},x) =2\left(\frac{\pi - \psi(x',x)}{\pi}(1+x'x) + \frac{\lvert x^{\prime}-x\rvert}{\pi}\right)$  is a twice continuously differentiable function where $\psi(x',x)=\arccos\frac{1+x'x}{\sqrt{(1+x^{'2})(1+x^{2})}}$. Moreover,
\begin{equation}
  \frac{\partial \Pi(x',x)}{\partial x} =\frac{2x}{\pi}\frac{|x-x'|}{1+x^2} +2x' - \frac{2}{\pi}x'\psi \mbox{~~ and~~ }  \frac{\partial^2 \Pi(x',x)}{\partial x^2} = \frac{4}{\pi}\frac{|x-x'|}{(1+x^2)^2}.
\end{equation}
\end{lemma}

\begin{proof}[Proof of Lemma \ref{lemma:K_derivative}] Simple calculations give us that
\begin{equation}
\begin{split}
  \frac{\partial \psi(x,x')}{\partial x}&=\frac{\operatorname{sgn}(x-x')}{(1+x^2)} \mbox{~ and ~}\frac{\partial^2 \psi(x,x')}{\partial x^2}=-\frac{2x \operatorname{sgn}(x-x')}{(1+x^2)^2} 
\end{split}
\end{equation}
and the desired results.
\end{proof}

\begin{lemma}\label{lem:simple:computations}
Detailed calculations in equation \eqref{eqn:simple:computations}.
\end{lemma}

\begin{proof}
For $k\leq n-2$, we have
\begin{equation}
    C^{(1)} = \begin{pmatrix}
    C^{(1)}_{1,1} & u & C^{(1)}_{1,n}\\
    l & M & r \\
    C^{(1)}_{n,1}& d & C^{(1)}_{n,n}
    \end{pmatrix}, \quad 
    T_{k}C^{(1)} = \frac{1}{n-1}\begin{pmatrix}
    0 & 0 &0\\
    l & M & r \\
    0& 0 & 0
    \end{pmatrix}
\end{equation}
where $u$ and $d$ are $1\times (n-2)$ vectors, $l$ and $r$ are $(n-2)\times 1$ vectors, and $M$ is a $(n-2)\times (n-2)$ matrix. Simple calculations imply that
\begin{equation}
    \left(T_{k}C^{(1)}\right)^p =   \left(\frac{1}{n-1}\right)^{p} \begin{pmatrix}
    0 & 0 &0\\
    M^{p-1}l & M^p & M^{p-1}r \\
    0& 0 & 0
    \end{pmatrix}.
\end{equation}
For any $2\leq i,j\leq n-1$,  we have
\begin{equation}
\begin{aligned}
    \left(C^{(1)}\right)_{i,j}&=M_{i,j}=-(n-1)(\Pi G^{-1})_{i,j} \leq 2\frac{|i-j|+1}{n-1},\\
    \left(C^{(1)}T_{k}C^{(1)}\right)_{i,j}&=  \left(\frac{M^2}{n-1}\right)_{i-1,j-1}
    \leq \frac{1}{n-1}\sum_{k=1}^{n-2}\frac{2(|i-1-k|+1)}{n-1}\frac{2(|k-j+1|+1)}{n-1} \\
    &\leq \frac{4}{3}\\
  \left(C^{(1)} \left( T_{k}C^{(1)}\right)^{2}\right)_{i,j}&=\left(\frac{M^3}{(n-1)^2}\right)_{i-1,j-1}\\
  &\leq\frac{1}{(n-1)^5}\sum_{l=1}^{n-2} \sum_{k=1}^{n-2}8(|i-1-k|+1)(|k-l|+1)(|l-j+1|+1)\\
 &\leq \frac{1}{(n-1)^5}\sum_{l=1}^{n-2} \sum_{k=1}^{n-2}8(n-1-k)(|k-l|+1)l\\
 &=\frac{2}{15}\frac{(n-2)n(2n-1)(3n-7)}{(n-1)^5}\leq \frac{4}{5}.
\end{aligned}
\end{equation}
Finally, we have
\begin{equation}
\left( \left(T_{k}C^{(1)}\right)^{3q}\right)_{i,j}\leq  \left(\frac{4}{5}\right)^{q}\frac{1}{(n-1)}.
\end{equation}
\end{proof}

